\documentclass{article}
\usepackage[utf8]{inputenc}
\usepackage{geometry}
\newgeometry{vmargin={1in}, hmargin={1in,1in}} 
\usepackage[utf8]{inputenc} %
\usepackage[T1]{fontenc}    %
\usepackage{hyperref}       %
\usepackage{url}            %
\usepackage{booktabs}       %
\usepackage{amsfonts}       %
\usepackage{nicefrac}       %
\usepackage{microtype}      %
\usepackage{xcolor}         %
\usepackage{amsmath,amssymb, amsthm}
\usepackage{mathabx}
\usepackage{soul}
\usepackage{pgfplots}
\pgfplotsset{width=10cm,compat=1.9}
\usepgfplotslibrary{external}
\hypersetup{
	colorlinks=true,
    linkcolor=red,
    citecolor=blue,
    filecolor=black,
    urlcolor=black,
}

\usepackage{amsopn}
\DeclareMathOperator*{\argmin}{argmin}

\newtheorem{theorem}{Theorem} 
\newtheorem{lemma}[theorem]{Lemma}
\newtheorem{definition}[theorem]{Definition}
\newtheorem{proposition}[theorem]{Proposition}
\newtheorem{assumption}
{Assumption} 
\newtheorem{remark}[theorem]{Remark}
\newtheorem{example}
{Example}
\newcommand{\ZZ}{\mathcal{Z}}
\newcommand{\wt}{\widetilde}

\usepackage{cleveref}
\crefname{algorithm}{Algorithm}{Algorithms}
\crefname{assumption}{Assumption}{Assumptions}
\crefname{equation}{}{}
\crefname{figure}{Fig.}{Figs.}
\crefname{table}{Table}{Tables}
\crefname{section}{Section}{Sections}
\crefname{subsection}{Section}{Sections}
\crefname{theorem}{Theorem}{Theorems}
\crefname{lemma}{Lemma}{Lemmmas}
\crefname{proposition}{Proposition}{Propositions}
\crefname{definition}{Definition}{Definitions}
\crefname{corollary}{Corollary}{Corollaries}
\crefname{remark}{Remark}{Remarks}
\crefname{example}{Example}{Examples}
\crefname{appendix}{Appendix}{Appendices}

\usepackage{hyperref}       %
\usepackage{amsfonts}
\usepackage{bbm}
\usepackage{color}
\usepackage{enumitem}
\usepackage{wrapfig}
\usepackage{svg}
\usepackage{algorithm}
\usepackage{algorithmic}
\newcommand{\XX}{\mathcal{X}}

\newcommand{\WW}{\mathcal{W}}

\newcommand{\ws}{w^{*}}

\newcommand{\prox}{\texttt{\textup{prox}}}

\newcommand{\hf}{\widehat{F}}
\newcommand{\DD}{\mathcal{D}}
\newcommand{\expec}{\mathbb{E}}

\newcommand{\EPL}{\mathbb{E}F(\widehat{w}_T) - F^*}

\usepackage{mathtools}
\usepackage{cleveref}

\crefname{algorithm}{Algorithm}{Algorithms}
\crefname{assumption}{Assumption}{Assumptions}
\crefname{equation}{Eq.}{Eqs.}
\crefname{figure}{Fig.}{Figs.}
\crefname{table}{Table}{Tables}
\crefname{section}{Section}{Sections}
\crefname{theorem}{Theorem}{Theorems}
\crefname{lemma}{Lemma}{Lemmas}
\crefname{proposition}{Proposition}{Propositions}
\crefname{definition}{Definition}{Definitions}
\crefname{corollary}{Corollary}{Corollaries}
\crefname{remark}{Remark}{Remarks}
\crefname{example}{Example}{Examples}
\crefname{appendix}{Appendix}{Appendices}

\newcommand{\bx}{\mathbf{x}}

\newcommand{\Al}{\mathcal{A}}
\newcommand{\EPLL}{\expec F(w_T) - F^*}
\newcommand{\tilt}{\widetilde{\nabla} F_t}

\newcommand{\tilr}{\widetilde{\nabla} F^0_r}
\newcommand{\hilt}{\widehat{\nabla} F_t}

\newcommand{\Renyi}{R\'enyi }

\newcommand{\var}{\text{Var}}

\newcommand{\eplac}{\expec F(w_T^{ag}) - F^*}

\newcommand{\hw}{\hat{w}}
\newcommand{\till}{\widetilde{\nabla} F_\lambda}

\title{Private Stochastic Optimization with Large Worst-Case Lipschitz Parameter}

\author{Andrew Lowy \hspace{0.5cm} 
Meisam Razaviyayn}
\date{
\texttt{\{lowya, razaviya\}@usc.edu}\\
\vspace{0.1cm}
University of Southern California}

\begin{document}

\maketitle

\begin{abstract}%
We study differentially private (DP) stochastic optimization (SO) with loss functions whose worst-case Lipschitz parameter over all data points may be extremely large or infinite. To date, the vast majority of work on DP SO assumes that the loss is uniformly Lipschitz continuous over data (i.e. stochastic gradients are uniformly bounded  over all data points). While this assumption is convenient, it often leads to pessimistic excess risk bounds. In many practical problems, the worst-case (uniform) Lipschitz parameter of the loss over all data points may be extremely large due to outliers and/or heavy-tailed data. In such cases, the error bounds for DP SO, which scale with the worst-case Lipschitz parameter of the loss, are vacuous. To address these limitations, this work provides improved excess risk bounds that do not depend on the uniform Lipschitz parameter of the loss. Building on a recent line of work (Wang et al., 2020; Kamath et al., 2022), we assume that stochastic gradients have bounded $k$-th order \textit{moments} for some $k \geq 2$. Compared with works on uniformly Lipschitz DP SO, our excess risk scales with the $k$-th moment bound instead of the uniform Lipschitz parameter of the loss, allowing for significantly faster rates in the presence of outliers and/or heavy-tailed data. 

For \textit{smooth} convex and strongly convex loss functions, we provide \textit{linear-time} algorithms with state-of-the-art excess risk. We complement our excess risk upper bounds with novel lower bounds. 
In certain parameter regimes, our linear-time excess risk bounds are \textit{minimax optimal}. Second, we provide the first algorithm to handle \textit{non-smooth} convex loss functions. 
To do so, we develop novel algorithmic and stability-based proof techniques, which we believe will be useful for future work in obtaining optimal excess risk. 
Finally, our work is the first to address \textit{non-convex} non-uniformly Lipschitz loss functions satisfying the \textit{Proximal-PL inequality}; this covers some practical machine learning models. Our Proximal-PL algorithm has \textit{near-optimal} excess risk. 
\end{abstract}

\section{Introduction}
As the use of machine learning (ML) models in industry and society has grown dramatically in recent years, so too have concerns about the privacy of personal data that is used in training such models. It is well-documented that ML models may leak training data, e.g., via model inversion attacks and membership-inference
attacks~\cite{inversionfred, shokri2017membership, korolova2018facebook, nasr2019comprehensive, carlini2021extracting}.
\textit{Differential privacy} (DP)~\cite{dwork2006calibrating} is a rigorous notion of data privacy, and a plethora of work has been devoted to differentially private machine learning and optimization~\cite{chaudhuri2008privacy, duchi13, bst14, ullman2015private, wang2017ermrevisited, bft19, fkt20, lr21fl, cheu2021shuffle, asiL1geo}. Of particular importance is the fundamental problem of DP \textit{stochastic (convex) optimization} (S(C)O): given $n$ i.i.d. samples $X = (x_1, \ldots, x_n) \in \XX^n$ from an unknown distribution $\mathcal{D}$, we aim to privately solve
\begin{equation}
\label{eq:SO}
    \min_{w \in \WW}\big\{F(w) := \expec_{x \sim \DD} [f(w,x)] \big\},
\end{equation}
where  
$f: \WW \times \XX \to \mathbb{R}$ 
is the loss function and  $\WW \subset \mathbb{R}^d$ is the parameter domain. Since finding the exact solution to \eqref{eq:SO} is not generally possible, we measure the quality of the obtained solution via  \textit{excess risk} (a.k.a. excess population loss): The excess risk of a (randomized) algorithm $\Al$ for solving~\cref{eq:SO} is defined as $\expec F(\Al(X)) - \min_{w \in \WW} F(w)$, where the expectation is taken over both the random draw of the data $X$ and the algorithm~$\Al$. 

\vspace{0.2cm}

A large body of literature is devoted to characterizing the optimal achievable differentially private excess risk of~\cref{eq:SO} when the function $f(\cdot, x)$ is uniformly $L_f$-\textit{Lipschitz} for all $x \in \XX$---see e.g.,~\cite{bft19, fkt20, asiL1geo, bassily2021non, lr21fl}. In these works, the gradient of $f$ is assumed to be uniformly bounded  with $\sup_{w \in \WW, x \in \XX} \|\nabla_w f(w,x) \| \leq L_f$, and excess risk bounds scale with $L_f$. While this assumption is convenient for bounding the \textit{sensitivity}~\cite{dwork2006calibrating} of the steps of the algorithm, it is often unrealistic in practice or leads to pessimistic excess risk bounds. 
In many practical applications, data contains outliers, is unbounded or heavy-tailed (see e.g.~\cite{crovella1998heavy, markovich2008nonparametric, woolson2011statistical} and references therein for such applications). 
Consequently, $L_f$ may be prohibitively large. 
For example, even  the linear regression loss  $f(w, x) = \frac{1}{2}(\langle w, x^{(1)} \rangle - x^{(2)})^2$ with compact $\WW$ and data from $\mathcal{X} = \XX^{(1)} \times \XX^{(2)}$, leads to $L_f \geq \text{diameter}(\mathcal{X}^{(1)})^2$, which could be huge or even infinite. 
Similar observations can be made for other useful ML models such as deep neural nets~\cite{lei2021sharper}, and the situation becomes even grimmer in the presence of heavy-tailed data. \ul{In these cases, existing excess risk bounds, which scale with $L_f$, becomes vacuous.}

\vspace{0.2cm}

While $L_f$ can be very large in practice (due to outliers), the $k$-th \textit{moment} of the stochastic gradients is often reasonably small for some $k \geq 2$ (see, e.g.,~\cref{example: gauss lin reg}). This is because the $k$-th moment $\wt{r}_{k}:= \expec \left[\sup_{w\in \WW} \| \nabla_w f(w, x) \|_2^k\right]^{1/k}$ depends on the \textit{average} behavior of the stochastic gradients, while $L_f$ depends on the \textit{worst-case} behavior over all data points. 
Motivated by this observation and building on the prior results~\cite{wx20, klz21}, this work 
makes progress towards answering the following fundamental open questions: 
\begin{itemize}
 \item Question I: What are the minimax optimal rates for (strongly) convex DP SO?
 \item Question II: What utility guarantees are achievable for non-convex DP SO?
\end{itemize}

Prior works have sought to address the first question above:\footnote{\cite{wx20,klz21} consider a slightly different problem class than the class $\wt{r}_{k}$, which we consider: see~\cref{app: lemma scaling factors}. 
However, our results imply improved rates for the problem class considered in \cite{wx20,klz21}, by ~\cref{app: lemma scaling factors} and the precise versions of our linear-time upper bounds given in~\cref{app: linear time}. } The work of~\cite{wx20} provided the first excess risk upper bounds for \textit{smooth}
DP (strongly) convex SO. \cite{klz21} gave improved, yet suboptimal, upper bounds for \textit{smooth} (strongly) convex $f(\cdot, x)$, and lower bounds for (strongly) convex SO. 
In this work, we provide both new lower bounds and new algorithms with improved excess risk for (strongly) convex losses. In certain practical parameter regimes, our excess risk bounds for smooth $F$ are \textit{minimax optimal}, giving a partial answer to Question I. 
We also provide a novel algorithm for \textit{non-smooth} convex $F$. 
Regarding Question II, we give the \textit{first algorithm for DP SO with non-convex loss} functions satisfying the Proximal Polyak-Łojasiewicz (PL) condition~\cite{polyak, karimi2016linear}. We provide a summary of our results for the case $k=2$ in Figure~\ref{table: sum}, and a thorough discussion of related work in~\cref{app: related work}. 

\subsection{Preliminaries}
\label{subsec:Preliminaries}
Let $\| \cdot \|$ be the $\ell_2$ norm. Let $\WW$ be a convex, compact set of $\ell_2$ diameter~$D$. Function $g: \WW \to \mathbb{R}$ is \textit{$\mu$-strongly convex} if $g(\alpha w + (1- \alpha) w') \leq \alpha g(w) + (1 - \alpha) g(w') - \frac{\alpha (1-\alpha) \mu}{2}\|w - w'\|^2$ for all $\alpha \in [0,1]$ and all $w, w' \in \WW$. If $\mu = 0,$ we say $g$ is \textit{convex.} For convex $f(\cdot, x)$, denote any \textit{subgradient} of $f(w,x)$ w.r.t. $w$ by $\nabla f(w,x) \in \partial_w f(w,x)$: i.e. $f(w', x) \geq f(w, x) + \langle \nabla f(w,x), w' - w \rangle$ for all $w' \in \WW$. 
Function 
$g$
is \textit{$\beta$-smooth} if it is differentiable and its derivative 
$\nabla g$
is $\beta$-Lipschitz. For $\beta$-smooth, $\mu$-strongly convex $g$, denote its \textit{condition number} by $\kappa = \beta/\mu$. %
For functions $a$ and $b$ of input parameters, write $a \lesssim b$ if there is an absolute constant $A$ such that $a \leq Ab$ for all feasible values of input parameters. 
Write $a = \widetilde{\mathcal{O}}(b)$ if $a \lesssim \ell b$ for a logarithmic function $\ell$ of input parameters.
We assume that the stochastic gradient distributions have bounded $k$-th moment for some $k \geq 2$: 
\begin{assumption}
\label{ass:tilde}
There exists $k \geq 2$ and $\wt{r}^{(k)} > 0$ such that $\expec \left[\sup_{w \in \WW} \| \nabla f(w, x) \|_2^k\right]\leq \wt{r}^{(k)}$ for all
  $\nabla f(w, x_i) \in \partial_w f(w, x_i)$. Denote $\wt{r}_k := (\wt{r}^{(k)})^{1/k}$.
 \end{assumption}
 
\noindent Clearly, $\wt{r}_k \leq L_f = \sup_{\{\nabla f(w,x) \in \partial_w f(w,x)\}} \sup_{w, x} ~\|\nabla f(w,x)\|$, but this inequality is often very loose: 
\begin{example}
\label{example: gauss lin reg}
For linear regression on a unit ball $\WW$ with $1$-dimensional data $x^{(1)}, x^{(2)}  \in [-10^{10}, 10^{10}]$ having Truncated Normal distributions and $\var(x^{(1)}) = \var(x^{(2)}) \leq 1$, we have $L_f \geq 10^{20}$. On the other hand, $\wt{r}_k$ is much smaller than $L_f$ for $k < \infty$: e.g., $\wt{r}_2 \leq 5, ~\wt{r}_4 \leq 8$, and $\wt{r}_8 \leq 14$.
\end{example}
\noindent \noindent \textbf{Differential Privacy:} 
\textit{Differential privacy}~\cite{dwork2006calibrating} ensures that no adversary---even one with enormous resources---can infer much more 
about any person who contributes training data than if that person's data were absent. If two data sets $X$ and $X'$ differ in a single entry (i.e. $d_{\text{hamming}}(X, X') = 1$), then we say that $X$ and $X'$ are \textit{adjacent}.

\begin{definition}[Differential Privacy]
\label{def: DP}
Let $\varepsilon \geq 0, ~\delta \in [0, 1).$ A randomized algorithm $\Al: \XX^n \to \mathcal{W}$ is \textit{$(\varepsilon, \delta)$-differentially private} (DP) if for all pairs of adjacent data sets $X, X' \in \XX^n$
and all measurable subsets $S \subseteq \WW$, we have
$\mathbb{P}(\Al(X) \in S) \leq e^\varepsilon \mathbb{P}(\Al(X') \in S) + \delta$. 
\end{definition}
\normalsize \noindent In this work, we focus on \textit{zero-concentrated differential privacy}~\cite{bun16}:
\begin{definition}[Zero-Concentrated Differential Privacy (zCDP)]
A randomized algorithm $\mathcal{A}: \XX^n \to \mathcal{W}$ satisfies $\rho$-zero-concentrated differential privacy ($\rho$-zCDP) if for all 
pairs of adjacent data sets $X, X' \in \XX^n$
and all $\alpha \in (1, \infty)$, we have $D_\alpha(\Al(X) || \Al(X')) \leq \rho \alpha$,
where $D_\alpha(\Al(X) || \Al(X'))$ is the $\alpha$-\Renyi divergence\footnote{For distributions $P$ and $Q$ with probability density/mass functions $p$ and $q$, $D_\alpha(P || Q) := \frac{1}{\alpha - 1}\ln \left(\int p(x)^{\alpha} q(x)^{1 - \alpha}dx\right)$~\cite[Eq. 3.3]{renyi}.} between the distributions of $\Al(X)$ and $\Al(X')$.
\end{definition}
zCDP is weaker than $(\varepsilon, 0)$-DP, but stronger than $(\varepsilon, \delta)$-DP ($\delta > 0$) in the following sense: 
\begin{proposition}\cite[Proposition 1.3]{bun16}
\label{prop:bun1.3}
If $\Al$ is $\rho$-zCDP, then $\Al$ is  $(\rho + 2\sqrt{\rho \log(1/\delta)}, \delta)$ for any $\delta > 0$. 
\end{proposition}
\noindent Thus, if $\varepsilon \leq \sqrt{\log(1/\delta)}$, then any $\frac{\varepsilon^2}{2}$-zCDP algorithm is $(2\varepsilon\sqrt{\log(1/\delta)}, \delta)$-DP. \cref{app: privacy prelims} contains more background on differential privacy.
\subsection{Contributions and Related Work}
\label{sec: contributions}
We discuss our contributions in the context of related work. 
See Figure~\ref{table: sum} for a summary of our results when $k=2$, and~\cref{app: related work} for a more thorough discussion of related work. 

\vspace{.2cm}
\noindent \textbf{\ul{Improved Risk for Smooth Convex Losses in Linear Time} (\cref{sec: linear time}):}
For convex, $\beta$-smooth $F$, we provide an accelerated DP algorithm (\cref{alg: ACSA}), building on the work of~\cite{ghadimilan1}.\footnote{In contrast to~\cite{wx20, klz21}, \textit{we do not require $f(\cdot, x)$ to be $\beta_f$-smooth} for all $x$.} 
Our algorithm is \textit{linear time} and attains excess risk that improves over the previous state-of-the-art (\textit{not linear time}) algorithm~\cite[Theorem 5.4]{klz21} in practical parameter regimes (e.g. $d \gtrsim n^{1/6}$). The excess risk of our algorithm is 
\textit{minimax optimal} in certain cases: e.g., $d \gtrsim (\varepsilon n)^{2/3}$ or ``sufficiently smooth'' $F$ (see Remark~\ref{rem: affine optimal}). 

\vspace{.15cm}
For $\mu$-\textit{strongly convex}, $\beta$-smooth losses, acceleration results in excessive bias accumulation, so we propose a simple noisy clipped SGD. Our algorithm builds on~\cite{klz21}, but uses a lower-bias clipping mechanism from~\cite{bd14} and a new, tighter analysis.
We attain excess risk
that is near-optimal up to a \small $\widetilde{\mathcal{O}}((\beta/\mu)^{(k-1)/k})$\normalsize ~factor: see~\cref{thm: strongly convex smooth upper bound}. Our bound strictly improves over the best previous bound of~\cite{klz21}. 

\vspace{.15cm}
We complement our excess risk upper bounds with information-theoretic \textit{lower bounds}. Our lower bounds refine (to describe the dependence on $\wt{r}_k, D, \mu$), extend (to $k \gg 1$), and tighten (for $\mu = 0$) the lower bounds of \cite{klz21}. 

\vspace{.15cm}
Special cases of our main results for smooth (strongly) convex losses are stated below: 
\begin{theorem}[Informal/special cases, see~\cref{thm: convex ACSA one pass,thm: strongly convex smooth upper bound,thm: convex lower bound,thm: strongly convex lower bound,rem: affine optimal}]
\label{thm: main result informal - smooth}
Let $F$ be convex and $\beta$-smooth and assume the constraint set $\WW$ has $\ell_2$-diameter $D$. Grant~\cref{ass:tilde}. If $\beta \lesssim \frac{\wt{r}_k}{D}\left(\frac{\sqrt{d}}{\varepsilon n}\right)^{\frac{k-1}{k}} + \frac{\wt{r}_2}{D}\frac{1}{\sqrt{n}}$, then there is a linear-time $\frac{\varepsilon^2}{2}$-zCDP algorithm $\Al$ such that \small $\expec F(\Al(X)) - F^* = \widetilde{\mathcal{O}}\left(D
\left(\frac{\wt{r}_2}{\sqrt{n}} + \wt{r}_{k}\left(\frac{\sqrt{d}}{\varepsilon n} \right)^{(k-1)/k} \right) \right)$. \normalsize If $F$ is $\mu$-strongly convex and $\beta/\mu = \mathcal{O}(1)$, then there is a linear-time $\frac{\varepsilon^2}{2}$-zCDP algorithm $\Al'$ such that \small $\expec F(\Al'(X)) - F^* = \widetilde{\mathcal{O}}\left(\frac{1}{\mu}\left(\frac{\wt{r}_2^2}{n} + \wt{r}_{k}^2\left(\frac{\sqrt{d}}{\varepsilon n} \right)^{(2k-2)/k} \right) \right)$. \normalsize Moreover, the above excess risk bounds are \textit{minimax optimal} up to logarithmic factors (in the stated parameter regimes).
\end{theorem}

As $k \to \infty$, $\wt{r}_{k} \to L_f$ and~\cref{thm: main result informal - smooth} recovers the known rates for uniformly $L_f$-Lipschitz DP SCO~\cite{bft19, fkt20}. However, when $k < \infty$ and $\wt{r}_{2k} \ll L_f$, the excess risk bounds in~\cref{thm: main result informal - smooth} may be much smaller than the uniformly Lipschitz excess risk bounds, which increase with $L_f$.

\vspace{.1cm}
The works~\cite{wx20,klz21} make a slightly different assumption than~\cref{ass:tilde}: they instead assume that the $k$-th order central moment of each coordinate $\nabla_j f(w,x)$ is bounded by $\gamma_k^{1/k}$ for all $j \in [d], w \in \WW$. We also provide excess risk bounds for the class of problems satisfying the coordinate-wise moment assumption of~\cite{wx20,klz21} and having \textit{subexponential} stochastic subgradients: see~\cref{app: asymptotic}.  

\vspace{.1cm}
The previous state-of-the-art convex upper bound was $\Omega\left(\wt{r}_k D \sqrt{\frac{d}{n}}\right)$
~\cite[Theorems 5.2 and 5.4]{klz21}.\footnote{We write the bound in~\cite[Theorem 5.4]{klz21} in terms of~\cref{ass:tilde}, replacing their $\gamma_k^{1/k} d$ by $\wt{r}_k \sqrt{d}$.}
Their result also required $f(\cdot, x)$ to be $\beta_f$-smooth for all $x \in \XX$ with $\beta_f \leq 10$, which can be restrictive with outlier data: e.g. this implies that $f(\cdot, x)$ is uniformly $L_f$-Lipschitz with $L_f \leq \beta_f D \leq 10 D$ if $\nabla f(\ws(x), x) = 0$ for some $\ws(x) \in \WW$. By comparison, our bounds hold even for $f$ with $L_f, \beta_f \gg 1$ or $\beta_f = L_f = \infty$. 

\vspace{.1cm}
Technically, a key bottleneck in the prior works~\cite{klz21,wx20} is the reliance on \textit{uniform convergence} to bound the gradient estimation error simultaneously for all $w \in \WW$. This approach leads to poor dependence on the dimension $d$. By contrast, our linear-time algorithm avoids uniform convergence arguments and instead uses disjoint batches of data in each iteration to maintain the independence between iterates $w_{t}$ and $w_{t+1}$. This results in improved dependence on $d$ in our bounds compared to the prior works. Additionally, we incorporate \textit{acceleration}~\cite{ghadimilan1} to speed up convergence. To prove our upper bound, we give 
the first analysis of accelerated SGD with biased stochastic gradients.
\vspace{.1cm}
Our linear-time $\mu$-strongly convex bound also improves over the best previous upper bound of~\cite[Theorem 5.6]{klz21}, which required uniform $\beta_f$-smoothness of $f(\cdot, x)$. In fact, ~\cite[Theorem 5.6]{klz21} was incorrect, 
as we explain in~\cref{app: wrong proofs}.\footnote{In short, the mistake is that Jensen's inequality is used in the wrong direction to claim that the $T$-th iterate of their algorithm $w_T$ satisfies $\expec[\|w_T - \ws\|^2] \leq (\expec \|w_T - \ws\|)^2$, which is false.} After communicating with the authors of~\cite{klz21}, they updated and corrected the result and proof in the arXiv version of their paper. The arXiv version of~\cite[Theorem 5.6]{klz21}---which we derive in~\cref{app: wrong proofs} for completeness---is suboptimal by a factor of $\wt{\Omega}((\beta_f/\mu)^{3})$. In practice, the worst-case condition number $\beta_f/\mu$ can be very large, especially in the presence of outliers or heavy-tailed data. Our excess risk bound removes this dependence on $\beta_f/\mu$ and is minimax optimal for constant $\beta/\mu$.

\begin{figure}[t]
 \centering
 \includegraphics[width=\textwidth]
{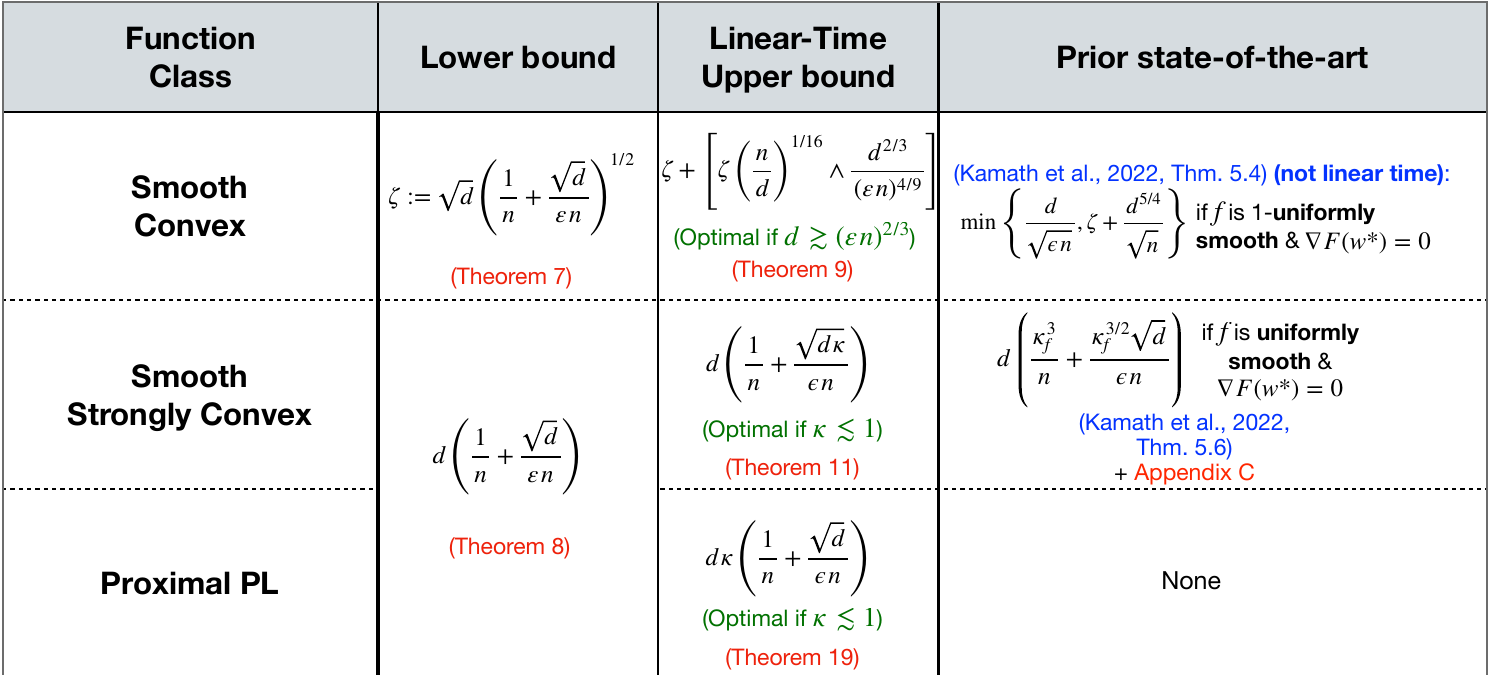}
\vspace{-.25in}
\caption{\footnotesize 
Smooth excess risk for $k=2$, $\wt{r}_2 = \sqrt{d}$; we omit logarithms. $\kappa = \beta/\mu$ is the condition number of $F$; $\kappa_f = \beta_f/\mu$ is the worst-case condition number of $f(\cdot, x)$. See~\cref{thm: localization convex,thm: localization strongly convex} for non-smooth upper bounds. 
}\label{table: sum}
\vspace{-.2in}
\end{figure}

\vspace{.2cm}
\noindent \textbf{\underline{First Algorithm for Non-Smooth (Strongly) Convex Losses} (\cref{sec: optimal rates}):} %
Prior works~\cite{wx20,klz21} required a strict uniform smoothness assumption in order to privately estimate the gradient of the loss function using~\cite{hol19}. Additionally, their uniform convergence argument required a smoothness assumption. 
In this work, to handle non-smooth loss functions, we develop a novel algorithm and analysis that circumvents the bottlenecks of the previous works: To privately estimate the gradient of the loss function, we borrow the method of~\cite{bd14}, which does not require uniform smoothness. Additionally, to prove generalization of our algorithm, we prove and use novel \textit{stability} results for non-smooth loss functions instead of the uniform convergence arguments of~\cite{wx20,klz21}. We summarize our main results for non-smooth loss functions below: 

\begin{theorem}[Informal, see~\cref{thm: localization convex}, \cref{thm: localization strongly convex}]
\label{thm: main result informal}
Let $f(\cdot, x)$ be convex. 
Grant~\cref{ass:tilde}. 
Then, there is a polynomial-time $\frac{\varepsilon^2}{2}$-zCDP algorithm $\Al$ such that \small $\small \expec F(\Al(X)) - F^* = \widetilde{\mathcal{O}}\left(
\wt{r}_{2k} D
\left(\frac{1}{\sqrt{n}} + \left(\frac{\sqrt{d}}{\varepsilon n} \right)^{(k-1)/k} \right) \right)$. \normalsize
If $f(\cdot, x)$ is $\mu$-strongly convex, then \small $\expec F(\Al(X)) - F^* = \widetilde{\mathcal{O}}\left(
\frac{\wt{r}_{2k}^2}{\mu}
\left(\frac{1}{n} + \left(\frac{\sqrt{d}}{\varepsilon n} \right)^{(2k-2)/k} \right) \right)$. \normalsize
\end{theorem}
The above upper bounds nearly resemble the lower bounds in~\cref{thm: convex lower bound,thm: strongly convex lower bound} except that $\wt{r}_{2k}$ appears in place of $\wt{r}_k$. We leave it as an open problem to obtain tight bounds with $\wt{r}_{2k} \mapsto \wt{r}_k$. 
We believe that this can be accomplished by building on our~\cref{alg: localization} and our novel proof techniques.

\vspace{.15cm}
Our \cref{alg: localization} combines the iterative localization technique of~\cite{fkt20, asiL1geo} with a noisy \textit{clipped} subgradient method. With clipped (hence \textit{biased}) stochastic subgradients and non-Lipschitz/non-smooth $f(\cdot, x)$, the excess risk analysis of our algorithm is harder than in the uniformly Lipschitz setting. Instead of the uniform convergence analysis used in~\cite{wx20, klz21}, we derive new results about the \textit{stability}~\cite{kearns1999algorithmic, bousquet2002stability} and generalization error of learning with loss functions that are not uniformly Lipschitz or differentiable; prior results (e.g. \cite{shalev2009stochastic, lei2020fine}) were limited to $\beta_f$-smooth and/or $L_f$-Lipschitz $f(\cdot, x)$. Specifically, we show the following \textit{for non-Lipschitz/non-smooth $f(\cdot, x)$}: 
a) \textit{On-average model stability}~\cite{lei2020fine} implies generalization (Proposition~\ref{prop: stability implies generalization}); and b) regularized empirical risk minimization is on-average model stable, hence it generalizes (Proposition~\ref{cor: reg ERM excess risk}).
We combine these 
results with an empirical error bound for 
\textit{biased}, noisy subgradient method to bound the excess risk of our algorithm (\cref{thm: localization convex}). We obtain our strongly convex bound (\cref{thm: localization strongly convex}) by a reduction to the convex case, ala~\cite{hk14, fkt20}.

\vspace{.2cm}
\noindent \textbf{\underline{First Algorithm for Non-Convex (Proximal-PL) Losses} (\cref{sec: PL}):}
We consider losses satisfying the \textit{Proximal Polyak-\L ojasiewicz (PPL) inequality} \cite{polyak, karimi2016linear} (Definition~\ref{def: Prox PL}), an extension of the classical PL inequality to the proximal setting. This covers important models like (some) neural nets, linear/logistic regression, and LASSO~\cite{karimi2016linear, lei2021sharper}. We propose a DP proximal clipped SGD to attain near-optimal excess risk that almost matches the \textit{strongly convex} rate: see~\cref{thm: PL upper bound}. 

\vspace{.2cm}
We also provide (in~\cref{app: SDP mean estimators}) the \textbf{first \textit{shuffle differentially private (SDP)}}~\cite{prochlo, cheu2019distributed} \textbf{algorithms for SO with large worst-case Lipschitz parameter}.
Our SDP algorithms achieve the same risk bounds as their zCDP counterparts \textit{without requiring a trusted curator.}

\section{Private Heavy-Tailed Mean Estimation Building Blocks}
In each iteration of our SO algorithms, we need 
a way 
to privately estimate the mean $\nabla F(w_t) = \expec_{x \sim \DD}[\nabla f(w_t, x)]$. If $f(\cdot, x)$ is uniformly Lipschitz, then one can simply draw a random sample $x^t$ from $X$ and add noise to the stochastic gradient $\nabla f(w_t, x^t)$ to obtain a DP estimator of $\nabla F(w_t)$: the $\ell_2$-sensitivity of the stochastic gradients is bounded by $\sup_{x, x' \in \XX} \| \nabla f(w_t, x) - \nabla f(w_t, x')\| \leq 2L_f$, so the Gaussian mechanism guarantees DP (by Proposition~\ref{prop: gauss}). However, in the setting that we consider, $L_f$ (and hence the sensitivity) may be huge, leading the privacy noise variance to also be huge. Thus, we \textit{clip} the stochastic gradients (to force the sensitivity to be bounded) before adding noise. Specifically, we invoke~\cref{alg: MeanOracle2} on a batch of $s$ stochastic gradients at each iteration of our algorithms.
In~\cref{alg: MeanOracle2}, $\Pi_{C}(z) := \argmin_{y \in B_2(0, C)}\|y - z\|^2$ denotes the projection onto the centered $\ell_2$ ball of radius $C$ in $\mathbb{R}^d$. Lemma~\ref{lem: bias and variance of bd14} bounds the bias and variance of~\cref{alg: MeanOracle2}.
\begin{algorithm}[ht]
\caption{$\texttt{MeanOracle1}(\{z_i\}_{i=1}^s; s; C; \frac{\varepsilon^2}{2})$ \cite{bd14}}
\label{alg: MeanOracle2}
\begin{algorithmic}[1]
\STATE {\bfseries Input:} 
$Z = \{z_i\}_{i=1}^s$, $C>0$, $\varepsilon > 0$. Set $\sigma^2 = \frac{4 C^2}{s^2 \varepsilon^2}$ for $\frac{\varepsilon^2}{2}$-zCDP. 
\STATE  Draw $N \sim \mathcal{N}(0, \sigma^2 \mathbf{I}_d)$ and compute $\widetilde{\nu} := \frac{1}{s}\sum_{i=1}^s \Pi_{C}(z_i) + N$.
\STATE {\bfseries Output:} $\widetilde{\nu}$. 
\end{algorithmic}
\end{algorithm}
\begin{lemma}[\cite{bd14}]
\label{lem: bias and variance of bd14}
Let $\{z_i\}_{i=1}^s \sim \DD^s$ be $\mathbb{R}^d$-valued random vectors with $\expec z_i = \nu$ and $\expec\|z_i\|^k \leq r^{(k)}$ for some $k \geq 2$. Denote the noiseless average of clipped samples by $\widehat{\nu} := \frac{1}{s}\sum_{i=1}^s \Pi_C(z_i)$ and $\wt{\nu}:= \widehat{\nu} + N$. Then, $\|\expec \widetilde{\nu} - \nu \| = \|\expec \widehat{\nu} - \nu \| \leq \expec \|\widehat{\nu} - \nu \| \leq \frac{r^{(k)}}{(k-1)C^{k-1}}$, and $\expec\|\widetilde{\nu} - \expec \widetilde{\nu}\|^2 = \expec \| \widetilde{\nu} - \expec \widehat{\nu}\|^2 \leq d\sigma^2 + \frac{r^{(2)}}{s}$. 
\end{lemma}

\section{Excess Risk Lower Bounds for Smooth (Strongly) Convex Losses}
\label{sec: lower bounds}
In this section, we provide excess risk lower bounds for the class of DP SCO problems satisfying~\cref{ass:tilde}. These lower bounds will be useful for addressing \textit{Question I} of the introduction and for calibrating the performance of our algorithms in later sections.

\vspace{.1cm}
The work of~\cite{klz21} proved lower bounds for $D = \mu = 1$, $\wt{r}_k = \sqrt{d}$, and $k = \mathcal{O}(1)$.\footnote{The lower bounds asserted in~\cite{klz21} only hold for constant $k = \mathcal{O}(1)$ since the moments of the Gaussian distribution that they construct grow exponentially/factorially with $k$.}
Our contribution in this subsection is: refining these lower bounds to display the correct dependence on $\wt{r}_k, D, \mu$; tightening the convex lower bound~\cite[Theorem 6.4]{klz21} in the regime $d > n$; and extending~\cite[Theorems 6.1 and 6.4]{klz21} to $k \gg 1$. 

Our first lower bound holds even for affine functions:
\begin{theorem}[Smooth Convex, Informal]
\label{thm: convex lower bound}
Let $\rho \leq d$. 
For any $\rho$-zCDP algorithm $\mathcal{A}$, there exist closed convex sets $\WW, \XX \subset \mathbb{R}^d$ such that $\|w - w'\| \leq 2D$ for all $w, w' \in \WW$, a 
$\beta_f$-smooth, linear, convex (in $w$) loss $f: \WW \times \XX \to \mathbb{R}$, and distribution $\mathcal{D}$ on $\XX$ such that \cref{ass:tilde} holds
and if $X \sim \mathcal{D}^n$, then 
\small
\[
\expec F(\Al(X)) - F^* = \Omega\left(D \left(\frac{\wt{r}_2}{\sqrt{n}} + \wt{r}_k \min\left\{1, \left(\frac{\sqrt{d}}{\sqrt{\rho} n}\right)^{\frac{k-1}{k}}\right\}\right)\right).\] \normalsize 
\end{theorem}
\noindent Remark~\ref{rem: lower bounds vs. klz} (in \cref{app: lower bounds}) discusses parameter regimes in which~\cref{thm: convex lower bound} is strictly tighter than~\cite[Theorem 6.4]{klz21}, as well as differences in our proof vs. theirs. 

Next, we provide lower bounds for smooth, strongly convex loss functions: 
\begin{theorem}[Smooth Strongly Convex, Informal]
\label{thm: strongly convex lower bound}
Let 
$\rho \leq d$. For any $\rho$-zCDP algorithm $\mathcal{A}$, there exist compact convex sets $\WW, \XX \subset \mathbb{R}^d$, 
a $\mu$-smooth, $\mu$-strongly convex (in $w$) loss $f: \WW \times \XX \to \mathbb{R}$, and distribution $\mathcal{D}$ on $\XX$ such that:
\cref{ass:tilde} holds, 
and if $X \sim \mathcal{D}^n$, then \[ \expec F(\Al(X)) - F^* = \Omega\left(\frac{1}{\mu}\left(\frac{\wt{r}_2^2}{n} + \wt{r}_k^2\min\left\{1, \left(\frac{\sqrt{d}}{\sqrt{\rho} n}\right)^{\frac{2k-2}{k}}\right\}\right)\right).\]
\end{theorem}

In the next section, we provide algorithms and excess risk upper bounds for DP SCO under~\cref{ass:tilde}. 
\section{Linear-Time Algorithms for Smooth (Strongly) Convex Losses}
\label{sec: linear time}
This section gives linear-time algorithms for smooth loss functions. We begin by providing our results for convex loss functions. 
\subsection{Noisy Clipped Accelerated SGD for Smooth Convex Losses}
\label{sec: convex}
\cref{alg: ACSA} is a one-pass accelerated algorithm, which builds on (non-private) AC-SA of~\cite{ghadimilan1}; 
its privacy and excess risk guarantees are given in~\cref{thm: convex ACSA one pass}.
\begin{algorithm}[ht]
\caption{Noisy Clipped Accelerated SGD (AC-SA) for Heavy-Tailed DP SCO}
\label{alg: ACSA}
\begin{algorithmic}[1]
\STATE {\bfseries Input:} 
Data $X \in \XX^n$, iteration number $T \leq n$, %
stepsize parameters $\{\eta_t \}_{t \in [T]}, \{\alpha_t \}_{t \in [T]}$ with $\alpha_1 = 1, \alpha_t \in (0,1)$~$\forall t \geq 2$. 
\STATE Initialize $w_0^{ag} = w_0 \in \WW$ and $t = 1$. 
 \FOR{$t \in [T]$} 
 \STATE $w_t^{md} := 
 (1- \alpha_t)w_{t-1}^{ag} + \alpha_t w_{t-1}$. 
 \STATE Draw new batch $\mathcal{B}_t$ (without replacement) of $n/T$ samples from $X$.
 \STATE $\tilt(w_t^{md}) := \texttt{MeanOracle1}(\{\nabla f(w_t^{md}, x)\}_{x \in \mathcal{B}_t}; \frac{n}{T}; \frac{\varepsilon^2}{2})$ 
 \STATE $w_{t} := 
 \argmin_{w \in \mathcal{W}}\left\{\alpha_t\langle \tilt(w_t^{md}), w\rangle + \frac{\eta_t}{2}\|w_{t-1} - w\|^2\right\}.
 $
 \STATE $w_{t}^{ag} := \alpha_t w_t + (1-\alpha_t)w_{t-1}^{ag}.$
\ENDFOR \\
\STATE {\bfseries Output:} $w_T^{ag}$.
\end{algorithmic}
\end{algorithm}
\begin{theorem}[Informal]
\label{thm: convex ACSA one pass}
Grant~\cref{ass:tilde}. Let $F$ be convex and $\beta$-smooth. Then, there are parameters such that \cref{alg: ACSA} 
is $\frac{\varepsilon^2}{2}$-zCDP and:
\begin{equation}
\label{eq: convex part 1}
\eplac \lesssim
\frac{\wt{r}_2 D}{\sqrt{n}} + \wt{r}_k D\left[
\left(\frac{\sqrt{d}}{\varepsilon n}\right)^{\frac{k-1}{k}} + \min\left\{
\left(\left(\frac{\beta D}{\wt{r}_k}\right)^{1/4} \frac{\sqrt{d}}{\varepsilon n} \right)^{\frac{4(k-1)}{5k-1}},
\left(\frac{\beta D}{\wt{r}_k}\right)^{\frac{k-1}{4k}}\left(\frac{\sqrt{d}}{\varepsilon n}\right)^{\frac{k-1}{k}}n^{\frac{k-1}{8k}}
\right\}
\right].
\end{equation}
Moreover, \cref{alg: ACSA} uses at most $n$ gradient evaluations.
\end{theorem}
\noindent 
The key ingredient used to prove~\cref{eq: convex part 1} is a novel convergence guarantee for AC-SA with \textit{biased}, noisy stochastic gradients: see Proposition~\ref{prop: ACSA generic} in~\cref{app: convex accel}. Combining Proposition~\ref{prop: ACSA generic} with Lemma~\ref{lem: bias and variance of bd14} and a careful choice of stepsizes, clip threshold, and $T$ yields~\cref{thm: convex ACSA one pass}.

\vspace{.1cm}
Comparing~\cref{thm: convex ACSA one pass} to the lower bound in~\cref{thm: convex lower bound}, one sees that  
our excess risk bounds are \textit{minimax optimal} in certain parameter reigmes, giving a partial answer to Question I.

\begin{remark}[Optimal rate for ``sufficiently smooth'' functions]
\label{rem: affine optimal}
Note that the upper bound~\cref{eq: convex part 1} scales with the smoothness parameter $\beta$. Thus, for sufficiently small $\beta$---namely $\small \beta \lesssim \frac{\wt{r}_k}{D}\left(\frac{\sqrt{d}}{\varepsilon n}\right)^{\frac{k-1}{k}} + \frac{\wt{r}_2}{D}\frac{1}{\sqrt{n}}$ \normalsize
---the optimal rates are attained (c.f. the lower bound in~\cref{thm: convex lower bound}). 

For example, fix $k=2$ and consider linear regression with $x \sim \mathcal{N}_d(0, \mathbf{I}_d)$ and $\varepsilon = D = 1$, and suppose the true labels $y = \|x\| + \zeta$ for some random $\zeta$ independent of $x$. Then, $\beta = 1$ and $\wt{r}_2 \geq d$, so that~\cref{eq: convex part 1} is optimal whenever $d \geq n^{2/5}$. 

As another example, for \textit{affine functions} (which were not addressed in~\cite{wx20, klz21} since these works assume $\nabla F(\ws) = 0$), $\beta = 0$ and \cref{alg: ACSA} is optimal.
\end{remark}

Having discussed the dependence on $\beta$, let us focus on understanding how the bound in \cref{thm: convex ACSA one pass} scales with $n, d$ and $\varepsilon$. Thus, let us fix \small $ \small \beta = D = 1 \normalsize$ and $\small \wt{r}_k = \sqrt{d} \normalsize$ for simplicity. If $k=2$, then the bound in \cref{eq: convex part 1} is \small $\small \mathcal{O}\left(\sqrt{\frac{d}{n}} + \frac{d^{3/4}}{\sqrt{\varepsilon n}} + \min\left\{\frac{d^{2/3}}{(\varepsilon n)^{4/9}}, \frac{d^{3/4}}{\sqrt{\varepsilon n}} \left(\frac{n}{d}\right)^{1/16}\right\}\right),$ \normalsize whereas the lower bound in~\cref{thm: convex lower bound} (part 2) is \small $\small \Omega\left(\sqrt{\frac{d}{n}} + \frac{d^{3/4}}{\sqrt{\varepsilon n}}\right) \normalsize$. \normalsize Therefore, the bound in \cref{eq: convex part 1} is \textit{optimal} if $d \gtrsim n$ or $d^{3/2} \gtrsim \varepsilon n$. For general $n, d, \varepsilon$,~\cref{eq: convex part 1} is suboptimal by a multiplicative factor of \small $\small \min\left\{\left(\frac{n}{d}\right)^{1/16},\left(\frac{\varepsilon n}{d^{3/2}}\right)^{1/18} \normalsize\right\}$. \normalsize By comparison, the previous state-of-the-art (\textit{not linear-time}) bound for $\varepsilon \approx 1$ was \small $\small \mathcal{O}\left(\frac{d}{\sqrt{n}}\right) \normalsize$\normalsize~\cite[Theorem 5.4]{klz21}.
 Our bound~\cref{eq: convex part 1}
improves over~\cite[Theorem 5.4]{klz21} if $\small d \gtrsim n^{1/6} \normalsize$, \normalsize which is typical in practical ML applications. 
As $k \to \infty$,
~\cref{eq: convex part 1} becomes \small $\small \mathcal{O}\left(\sqrt{\frac{d}{n}} + \left(\frac{d}{n}\right)^{4/5}\right) \normalsize$\normalsize ~for $\varepsilon \approx 1$\normalsize, which is strictly better than the bound in~\cite[Theorem 5.4]{klz21}. 
\subsection{Noisy Clipped SGD for Strongly Convex Losses}
\label{sec: strongly convex}
Our algorithm for strongly convex losses (\cref{alg: vanilla SGD} in~\cref{app: strongly convex}) is a simple one-pass noisy clipped SGD. 
Compared to the algorithm of~\cite{klz21}, our approach differs in the choice of \texttt{MeanOracle}, step size, and iterate averaging weights, and in our \textit{analysis}.

\begin{theorem}[Informal]
\label{thm: strongly convex smooth upper bound}
Grant~\cref{ass:tilde}. Let $F$ be $\mu$-strongly convex, $\beta$-smooth with $\frac{\beta}{\mu} \leq n/\ln(n)$. Then, there are algorithmic parameters such that \cref{alg: vanilla SGD} is $\frac{\varepsilon^2}{2}$-zCDP and: 
\begin{equation}
\label{eq: smooth sc upper l2 clip}
\EPL \lesssim \frac{1}{\mu}\left(\frac{\wt{r}_2^2}{n} + \wt{r}_k^2\left(\frac{\sqrt{d \kappa \ln(n)}}{\varepsilon n}\right)^{\frac{2k-2}{k}} \right).
\end{equation}
Moreover, \cref{alg: vanilla SGD} uses at most $n$ gradient evaluations.
\end{theorem}
\noindent 
The bound \cref{eq: smooth sc upper l2 clip} is optimal up to a $\widetilde{\mathcal{O}}((\beta/\mu)^{(k-1)/k})$ factor and
improves over the best previous bound in~\cite[Theorem 5.6]{klz21} by removing the dependence on $\beta_f$ (which can be much larger than $\beta$ in the presence of outliers): see~\cref{app: wrong proofs} for details. In particular, for well-conditioned losses with $\beta/\mu = \mathcal{O}(1)$, \cref{thm: strongly convex smooth upper bound} is optimal up to logarithms. 

\vspace{.1cm}
The proof of \cref{thm: strongly convex smooth upper bound} (in~\cref{app: strongly convex}) relies on a novel convergence guarantee for projected SGD with biased noisy stochastic gradients: Proposition~\ref{prop: strongly convex biased sgd}. Compared to results in \cite{asi2021private} for convex ERM and \cite{as21} for PL SO, Proposition~\ref{prop: strongly convex biased sgd} is tighter, which is needed to obtain near-optimal excess risk: we leverage smoothness and strong convexity. Our new analysis also avoids the issue in the proofs of (the ICML versions of)~\cite{wx20, klz21}.

\section{Algorithm for Non-Smooth (Strongly) Convex Losses}
\label{sec: optimal rates}
In this section, we develop a novel algorithm for non-smooth convex loss functions. 
We present our result for convex losses in~\cref{sec: localization}, and our result for strongly convex losses in~\cref{sec: localization strongly}.
In~\cref{sec: lower bounds}, we provide lower bounds, which show that our upper bounds are tight (up to logarithms). 
\subsection{Localized Noisy Clipped Subgradient Method for Convex Losses}
\label{sec: localization}
\noindent Our algorithm (\cref{alg: localization}) uses iterative localization~\cite{fkt20, asiL1geo} with clipping (in \cref{alg: clipped GD}) to handle stochastic subgradients with large norm.\footnote{We assume WLOG that $n = 2^l$ for some $l \in \mathbb{N}$. If this is not the case, then throw out samples until it is; since the number of remaining samples is at least $n/2$, our bounds still hold up to a constant factor.}
\begin{algorithm}[ht]
\caption{Noisy Clipped Regularized Subgradient Method for DP ERM
}
\label{alg: clipped GD}
\begin{algorithmic}[1]
\STATE {\bfseries Input:} 
Data $X \in \XX^n$, 
$T \in \mathbb{N}$, stepsize $\eta$, clip thresh. $C$, regularization $\lambda \geq 0$, constraint set $\WW$ and initialization $w_0 \in \WW$.
 \FOR{$t \in \{0, 1, \cdots, T-1\}$} 
 \STATE $\tilt(w_t) := \texttt{MeanOracle1}(\{\nabla f(w_t, x_i)\}_{i=1}^n; n; C; \frac{\varepsilon^2}{2T})$ for subgradients $\nabla f(w_t, x_i)$. 
 \STATE $w_{t+1} = 
 \Pi_{\WW}\left[w_t - \eta \left(\tilt(w_t) + \lambda(w_t - w_0) \right)\right]
 $
\ENDFOR \\
\STATE {\bfseries Output:} $w_T$. 
\end{algorithmic}
\end{algorithm}
\vspace{-.1in}
\begin{algorithm}[ht]
\caption{Localized Noisy Clipped Subgradient Method for DP SCO}
\label{alg: localization}
\begin{algorithmic}[1]
\STATE {\bfseries Input:} 
Data $X \in \XX^n$, stepsize $\eta$, clip thresh.
$\{C_i\}_{i=1}^{\log_2(n)}$, iteration num. $\{T_i\}_{i=1}^{\log_2(n)}$.
\STATE Initialize $w_0 \in \WW$. Let $l := \log_2(n)$ and $p = 1 + 1/k$. 
 \FOR{$i \in [l]$} 
 \STATE Set $n_i = 2^{-i} n, \eta_i = 4^{-i} \eta$, $\lambda_i = \frac{1}{\eta_i n_i^p}$, $T_i = \widetilde{\Theta}\left(\frac{1}{\lambda_i \eta_i}\right)$, and $D_i = \frac{4 \wt{r}_k (\sqrt{n} 2^i)^{1/k}}{\lambda_i}$.
 \STATE Draw new batch $\mathcal{B}_i$ of $n_i = |\mathcal{B}_i|$ samples from $X$ without replacement. 
 \STATE Let $\hf_i(w) := \frac{1}{n_i} \sum_{j \in \mathcal{B}_i} f(w, x_j) + \frac{\lambda_i}{2}\|w - w_{i-1}\|^2$. 
 \STATE Use~\cref{alg: clipped GD} with initialization $w_{i-1}$ to minimize $\hf_i$ over $\mathcal{W}_i := \{w \in \WW : \|w - w_{i-1}\| \leq D_{i}\}$, for $T_i$ iterations with clip threshold $C_i$ and noise $\sigma_i^2 = \frac{4C_i^2 T_i}{n_i^2 \varepsilon^2}$. Let $w_i$ be the output of~\cref{alg: clipped GD}.  
\ENDFOR \\
\STATE {\bfseries Output:} $w_l$. 
\end{algorithmic}
\end{algorithm}
\newpage
 
\quad 
The main ideas of~\cref{alg: localization} are:

\begin{enumerate}
\item \textit{Clipping only the non-regularized component of the subgradient to control sensitivity and bias:} Notice that when we call~\cref{alg: clipped GD} in phase $i$ of~\cref{alg: localization}, we only clip the subgradients of $f(w, x_j)$ and not the regularized loss $f(w,x_j) + \frac{\lambda}{2}\|w - w_{i-1}\|^2$. Compared to clipping the full gradient of the regularized loss, our selective clipping approach significantly reduces the bias of our subgradient estimator. This is essential for obtaining our excess risk bound. Further, this reduction in bias comes at no cost to the variance of our subgradient estimator: the $\ell_2$-sensitivity of our estimator is unaffected by the regularization term. 
   \item\textit{Solve regularized ERM subproblem with a \textit{stable} DP algorithm}:  We run a \textit{multi-pass} zCDP solver on a \textit{regularized} empirical loss: Multiple passes let us reduce the noise variance in phase $i$ by a factor of $T_i$ (via strong composition for zCDP) and get a more accurate solution to the ERM subproblem. Regularization makes the empirical loss strongly convex, which improves \textit{on-average model stability} and hence generalization of the obtained solution (see Proposition~\ref{prop: stability implies generalization} and~\ref{prop: strongly convex ERM generalization error}). 
\item \textit{Localization}~\cite{fkt20, asi2021private} (i.e. iteratively ``zooming in'' on a solution): In early phases (small $i$), when we are far away from the optimum $\ws$, we use more samples (larger $n_i$) and large learning rate $\eta_i$ to make progress quickly. As $i$ increases, $w_i$ is closer to $\ws$, so fewer samples and slower learning rate suffice.
Since step size $\eta_i$ shrinks (geometrically) faster than $n_i$, the effective variance of the privacy noise $\eta_i^2 \sigma_i^2$ decreases as $i$ increases. This prevents $w_{i+1}$ from moving too far away from $w_i$ (and hence from $\ws$). We further enforce this localization behavior by increasing the regularization parameter $\lambda_i$ and shrinking $D_i$ over time. We choose $D_i$ as small as possible subject to the constraint that $\argmin_{w \in \WW} \widehat{F}_i(w) \in \WW_i$ with high probability. This constraint ensures that~\cref{alg: clipped GD} can find $w_i$ with small expected excess risk.  
\end{enumerate}

\vspace{.1cm}
Next, we provide privacy and excess risk guarantees for~\cref{alg: localization}:
\begin{theorem}
\label{thm: localization convex}
Grant~\cref{ass:tilde}. 
Let $\varepsilon \leq \sqrt{d}$ and let $f(\cdot, x)$ be convex. Then, there are algorithmic parameters such that~\cref{alg: localization} is $\frac{\varepsilon^2}{2}$-zCDP, and 
has excess risk 
\begin{equation*}
    \expec F(w_l) - F^* \lesssim \wt{r}_{2k} D\left(\frac{1}{\sqrt{n}} + \left(\frac{\sqrt{d \ln(n)}}{\varepsilon n}\right)^{\frac{k-1}{k}}\right).
\end{equation*}
Moreover, this excess risk is attained in $\wt{\mathcal{O}}(n^{2 + 1/k})$ subgradient evaluations. 
\end{theorem}
\begin{remark}[Improved Computational Complexity for Approximate or Shuffle DP]
\label{rem: computation}
If one desires $(\varepsilon, \delta)$-DP or $(\varepsilon, \delta)$-SDP instead of zCDP, then the subgradient complexity of~\cref{alg: localization} can be improved to $\wt{\mathcal{O}}\left(n^{\frac{3}{2} + 1/k}\right)$: see~\cref{app: localization}.
\end{remark} 

The excess risk bound in~\cref{thm: localization convex} nearly resembles the lower bound in~\cref{thm: convex lower bound}, except that the upper bound scales with $\wt{r}_{2k}$ instead of $\wt{r}_k$.\footnote{In fact, the term $\wt{r}_{2k}$ in \cref{thm: localization convex} can be replaced by a smaller term, which is $\mathcal{O}(r_{2k})$ as $n \to \infty$ under mild assumptions. See~\cref{app: localization} and~\cref{app: asymptotic}.} We conjecture that an appropriate modification of our algorithm and techniques can be used to get the optimal bound depending on $\wt{r}_k$. We leave this as an interesting direction for future work.

\vspace{.2cm}
A key feature of our bound is that it does not depend on $L_f$ and holds even for heavy-tailed problems with $L_f = \infty$. 
By contrast, prior works~\cite{wx20, klz21} required uniform $\beta_f$-smoothness of $f(\cdot, x)$, which implies the restriction $L_f \leq \beta_f D$ for loss functions that have a vanishing gradient at some point.\footnote{Additionally, \cite{klz21} assumes $\beta_f \leq 10$.} 

\vspace{.2cm}
We now sketch the proof of the privacy claim in~\cref{thm: localization convex}:
\begin{proof}[Sketch of the Proof of~\cref{thm: localization convex}: Privacy]
\sloppy Since the batches $\{\mathcal{B}_i\}_{i=1}^l$ are disjoint, it suffices 
to show that $w_i$ (produced by $T_i$ iterations of~\cref{alg: clipped GD} in line 7 of~\cref{alg: localization}) is $\frac{\varepsilon^2}{2}$-zCDP~~$\forall i \in [l]$. The $\ell_2$ sensitivity of the clipped subgradient update is $\Delta = \sup_{w, X \sim X'} \|\frac{1}{n_i} \sum_{j=1}^{n_i} \Pi_{C_i}(\nabla f(w, x_j)) - \Pi_{C_i}(\nabla f(w, x'_j))\| \leq 2C_i/n_i$. (Note that the regularization term does not contribute to sensitivity.) Thus, the privacy guarantees of the Gaussian mechanism (Proposition~\ref{prop: gauss}) and the composition theorem for zCDP (Lemma~\ref{lem: composition}) imply that \cref{alg: localization} is $\frac{\varepsilon^2}{2}$-zCDP. 
\end{proof}

\vspace{.2cm}
The proof of the excess risk bound in~\cref{thm: localization convex} consists of three main steps: i) We bound the empirical error of the noisy clipped subgradient subroutine (Lemma~\ref{lem: subgrad ERM bound}). ii) We prove that if an algorithm is \textit{on-average model stable} (Definition~\ref{def: stability}), then it generalizes (Proposition~\ref{prop: stability implies generalization}).
iii) We bound the on-average model stability of regularized ERM with non-smooth/non-Lipschitz $f(\cdot, x)$ (Proposition~\ref{prop: strongly convex ERM generalization error}), leading to an excess population loss bound for~\cref{alg: clipped GD} run on the regularized empirical objective (c.f. line 7 of~\cref{alg: localization}).
By using these results with the proof technique of~\cite{fkt20}, we can obtain~\cref{thm: localization convex}. 

\vspace{.15cm}
First, we bound the empirical error of the step in line 7 of~\cref{alg: localization}, by extending the analysis of noisy subgradient method to \textit{biased} subgradient oracles: 
\begin{lemma}
\label{lem: subgrad ERM bound}
Fix $X \in \XX^n$ and let $\widehat{F}_\lambda(w) = \frac{1}{n} \sum_{i=1}^n f(w, x_i) + \frac{\lambda}{2}\|w - w_0\|^2$ for $w_0 \in \WW$, where $\WW$ is a closed convex domain with diameter $D$. Assume $f(\cdot, x)$ is convex and $\small \widehat{r}_n(X)^{(k)} \geq \sup_{w \in \WW}\left\{\frac{1}{n} \sum_{i=1}^n \|\nabla f(w, x_i)\|^k\right\}$ for all $\nabla f(w,x_i) \in \partial_w f(w, x_i)$. \normalsize Denote $\widehat{r}_n(X) = \left[\widehat{r}_n(X)^{(k)}\right]^{1/k}$ and $\hat{w} = \argmin_{w \in \WW} \widehat{F}_\lambda(w)$. Let $\eta \leq \frac{2}{\lambda}$. Then, 
the output of~\cref{alg: clipped GD} satisfies 
\[
\expec\|w_T - \hat{w}\|^2 \leq \exp\left(-\frac{\lambda \eta T}{2}\right)\|w_0 - \hw\|^2 + \frac{8 \eta}{\lambda}\left(\widehat{r}_n(X)^2 + \lambda^2 D^2 + d \sigma^2 \right) + \frac{20}{\lambda^2}\left(\frac{\widehat{r}_n(X)^{(k)}}{(k-1) C^{k-1}}\right)^2,
\]
\normalsize
where $\sigma^2 = \frac{4 C^2 T}{n^2 \varepsilon^2}$. 
\end{lemma}
\noindent Detailed proofs for this subsection are deferred to~\cref{app: localization}. 

\vspace{.15cm}
Our next goal is to bound the generalization error of regularized ERM with convex loss functions that are not differentiable or uniformly Lipschitz. We will use a stability argument to obtain such a bound. Recall the notion of \textit{on-average model stability}~\cite{lei2020fine}: 
\begin{definition}
\label{def: stability}
Let $X = (x_1, \cdots, x_n)$ and $X' = (x'_1, \cdots, x'_n)$ be drawn independently from $\mathcal{D}$. For $i \in [n]$, let $X^{i} := (x_1, \cdots, x_{i-1}, x'_i, x_{i+1}, \cdots, x_n)$. We say 
randomized algorithm $\Al$ has on-average model stability $\alpha$ (i.e. $\Al$ is $\alpha$-on-average model stable) if
\small
$
\expec\left[\frac{1}{n} \sum_{i=1}^n \|\Al(X) - \Al(X^i)\|^2\right] \leq \alpha^2.
$
\normalsize
The expectation is over the randomness of $\Al$ and the draws of $X$ and $X'$.
\end{definition}
\noindent On-average model stability is weaker than the notion of \textit{uniform stability}~\cite{bousquet2002stability}, which has been used in DP Lipschitz SCO (e.g. by \cite{bft19}); this is necessary for obtaining our learnability guarantees without uniform Lipschitz continuity. 

\vspace{.15cm}
The main result in \cite{lei2020fine} showed that on-average model stable algorithms generalize well if $f(\cdot, x)$ is $\beta_f$-\textit{smooth} for all $x$, which leads to a restriction on $L_f$.
We show that neither smoothness nor Lipschitz continuity of $f$ is needed to ensure generalizability: 
\begin{proposition} 
\label{prop: stability implies generalization}
Let $f(\cdot, x)$ be convex for all $x$.
Suppose $\Al: \XX^n \to \WW$ is $\alpha$-on-average model stable. Let $\hf_X(w) := \frac{1}{n}\sum_{i=1}^n f(w, x_i)$ be an empirical loss. Then for any $\zeta > 0$, 
\[
\expec[F(\Al(X)) - \hf_X(\Al(X))] \leq \frac{\wt{r}^{(2)}}{2 \zeta} + \frac{\zeta}{2} \alpha^2.
\]
\normalsize
\end{proposition}

Proposition~\ref{prop: strongly convex ERM generalization error} (and its proof) in \cref{app: localization} shows that regularized ERM is $\alpha$-on average model stable for $\alpha \leq \frac{2\wt{r}_2}{\lambda n}$. The proof of Proposition~\ref{prop: strongly convex ERM generalization error} combines techniques from~\cite{lei2020fine} with Proposition~\ref{prop: stability implies generalization}. By using Proposition~\ref{prop: stability implies generalization} and Proposition~\ref{prop: strongly convex ERM generalization error}, we can bound the generalization error 
and excess (population) risk of regularized ERM:
\begin{proposition}
\label{cor: reg ERM excess risk}
Let $f(\cdot, x)$ be convex, $w_{i-1}, y \in \WW$, and $\hw_i := \argmin_{w \in \WW} \hf_i(w)$, where $\hf_i(w) := \frac{1}{n_i} \sum_{j \in \mathcal{B}_i} f(w, x_j) + \frac{\lambda_i}{2}\|w - w_{i-1}\|^2$ (c.f. line 6 of~\cref{alg: localization}). Then, 
\[
\expec[F(\hw_i)] - F(y) \leq \frac{2\wt{r}^{(2)}}{\lambda_i n_i} + \frac{\lambda_i}{2}\|y - w_{i-1}\|^2,
\]
\normalsize
where the expectation is over both the random draws of $X$ from $\DD$ and $\mathcal{B}_i$ from $X$. 
\end{proposition}

With the pieces developed above, we can now sketch the excess risk proof of~\cref{thm: localization convex}: 
\begin{proof}[Sketch of the Proof of~\cref{thm: localization convex}: Excess risk]
First, our choice of $D_i$ ensures that $\hat{w}_i \in \WW_i$ for all $i \in [l]$ with high probability $\geq 1 - 1/\sqrt{n}$, by Chebyshev's inequality. We will assume that this event occurs in the rest of the proof sketch: see the detailed proof in the Appendix for when this event breaks. Denote $p = 1 + 1/k$. 
Combining Lemma~\ref{lem: subgrad ERM bound} with~Lemma~\ref{lem: bias and variance of bd14} and proper choices of $\eta$ and $T_i$, we get:
\begin{equation}
\label{eq: 0ing}
    \expec\|w_i - \hat{w}_i\|^2 
    \lesssim \eta_i^2 n_i^p (\wt{r}_k^2 n^{1/k} 2^{2i/k} + d\sigma_i^2) + \frac{\eta_i^2 n_i^{2p} \widetilde{r}^{(2k)}}{C_i^{2k-2}}  
    \lesssim \left(\frac{\eta^2 n^p}{32^i}\left((\wt{r}_k^2 n^{1/k} 2^{2i/k} + \frac{d C_i^2 T_i}{\varepsilon^2 n_i^2} +  \frac{n^p\widetilde{r}^{(2k)}}{C_i^{2k-2} 2^{pi}}\right) \right).
\end{equation}
\normalsize
Now, following the strategy used in the proof of~\cite[Theorem 4.4]{fkt20}, we write 
\small
$\expec F(w_l) - F(\ws) = \expec[F(w_l) - F(\hat{w}_l)] + \sum_{i=1}^l \expec[F(\hat{w}_i) - F(\hat{w}_{i-1})]$,
\normalsize
where $\hat{w}_0 := \ws$. Using~\cref{eq: 0ing} and $\wt{r}_k$-Lipschitz continuity of $F$ (which is implied by~\cref{ass:tilde}), we can bound $\expec[F(w_l) - F(\hat{w}_l)]$ for the right $\eta$ and $C_l$. To bound the sum (second term), we use Proposition~\ref{cor: reg ERM excess risk} to obtain
\begin{align*}
\sum_{i=1}^l \expec[F(\hw_i) - F(\hw_{i-1})] 
&\lesssim r^2 \eta n^{p-1} + \frac{D^2}{\eta n^p} +  \eta \wt{r}_k^2 n^{1/k} + \eta 
\sum_{i=2}^l 4^{-i} n_i^p \wt{r}_{2k}^{2} \left(\frac{d \ln(n)}{\varepsilon^2 n_i^2}\right)^{\frac{k-1}{k}} \\
&\lesssim \eta\left[r^2 n^{p-1} +  \wt{r}_k^2 n^{1/k} + \wt{r}_{2k}^2 n^p
\left(\frac{d \ln(n)}{\varepsilon^2 n^2}\right)^{\frac{k-1}{k}}\right] +  \frac{D^2}{\eta n^p},
\end{align*}
\normalsize
for the right choice of $C_i$. Then properly choosing $\eta$ completes the excess risk proof. 

\vspace{.2cm}
\noindent \textbf{Computational complexity:} The choice $T_i = \wt{\mathcal{O}}(n_i^p) = \wt{\mathcal{O}}(n_i^{1 + 1/k})$ implies that the number of subgradient evaluations is bounded by $\sum_{i=1}^l n_i T_i = \wt{\mathcal{O}}(n^{p+1}) = \wt{\mathcal{O}}(n^{2 + 1/k})$. 
\end{proof}
\subsection{The Strongly Convex Case}
\label{sec: localization strongly}
Following~\cite{fkt20}, we use a folklore reduction to the convex case (detailed in~\cref{app: localization strongly}) in order to obtain the following upper bound via~\cref{thm: localization convex}: 
\begin{theorem}
\label{thm: localization strongly convex}
Grant~\cref{ass:tilde}. 
Let $\varepsilon \leq \sqrt{d}$ and let $f(\cdot, x)$ be $\mu$-strongly convex. 
Then, there is a polynomial-time $\frac{\varepsilon^2}{2}$-zCDP algorithm $\Al$ based on~\cref{alg: localization} with excess risk 
\begin{equation*}
    \expec F(\Al(X)) - F^* \lesssim  \frac{ \wt{r}_{2k}^2}{\mu}\left(\frac{1}{n} + \left(\frac{\sqrt{d \ln(n)}}{\varepsilon n}\right)^{\frac{2k-2}{k}}\right).  
\end{equation*}
\normalsize
\end{theorem}
\section{Algorithm for Non-Convex Proximal-PL Loss Functions}
\label{sec: PL}
Assume: $f(w,x) = f^0(w,x) + f^1(w)$; $f^0(\cdot, x)$ is differentiable (maybe non-convex); $f^1$ is proper, closed, and convex (maybe non-differentiable) for all $x \in \XX$; and $F(w) = F^0(w) + f^1(w) = \expec_{x \sim \DD}[f^0(w,x)] + f^1(w)$ satisfies the \textit{Proximal-PL} condition~\cite{karimi2016linear}:
\begin{definition}[$\mu$-PPL]
\label{def: Prox PL}
Let $F(w) = F^0(w) + f^1(w)$ be bounded below; $F^0$ is $\beta$-smooth and $f^1$ is convex. $F$ satisfies \textit{Proximal Polyak-\L ojasiewicz} inequality with parameter $\mu > 0$ if 
\begin{align*}
\mu[F(w) - \inf_{w'} F(w')] &\leq - \beta \min_{y}\left[\langle \nabla F^0(w), y - w \rangle 
+ \frac{\beta}{2}\|y - w\|^2 + f^1(y) - f^1(w)\right], ~\forall~w \in \mathbb{R}^d.
\end{align*}
\normalsize
\end{definition}
\noindent Definition~\ref{def: Prox PL} generalizes the classical PL condition ($f^1 = 0$), allowing for constrained optimization or non-smooth regularizer depending on $f^1$~\cite{polyak, karimi2016linear}. 

\vspace{.15cm}
Recall that the \textit{proximal operator} of a convex function $g$ is defined as $\prox_{\eta g}(z) := \argmin_{y \in \mathbb{R}^d}\left(\eta g(y) + \frac{1}{2}\|y - z\|^2 \right)$ for $\eta > 0$\normalsize.  We propose \textit{Noisy Clipped Proximal SGD} (\cref{alg: zCSDP SGD} in~\cref{app: PL}) for PPL losses. The algorithm runs as follows. For $t \in [T]$: first draw a new batch $\mathcal{B}_t$ (without replacement) of $n/T$ samples from $X$; let $\tilt^0(w_t) := \texttt{MeanOracle1}(\{\nabla f^0(w_t, x)\}_{x \in \mathcal{B}_t}; \frac{n}{T}; \frac{\varepsilon^2}{2})$; then update $w_{t+1} = \prox_{\eta_t f^1}\left(w_t - \eta_t \tilt^0(w_t)\right)$. Finally, return the last iterate, $w_T$. Thus, the algorithm is linear time. Furthermore:
\begin{theorem}[Informal]
\label{thm: PL upper bound}
Grant~\cref{ass:tilde}. 
Let 
$F$
be $\mu$-PPL for $\beta$-smooth $F^0$, 
with $\frac{\beta}{\mu} \leq n/\ln(n)$. 
Then, there are parameters such that \cref{alg: zCSDP SGD} is $\frac{\varepsilon^2}{2}$-zCDP, and:
\[
\EPLL \lesssim \frac{1}{\mu}\left(\wt{r}_k^2 \left(\frac{\sqrt{d}}{\varepsilon n} (\beta/\mu) \ln(n) \right)^{\frac{2k-2}{k}} + \frac{\wt{r}_2^2 (\beta/\mu) \ln(n)}{n}\right).
\]
\normalsize
Moreover, \cref{alg: zCSDP SGD} uses at most $n$ gradient evaluations.
\end{theorem}
\noindent The bound in \cref{thm: PL upper bound} nearly matches the smooth \textit{strongly convex} (hence PPL) lower bound in \cref{thm: strongly convex lower bound} up to  $\widetilde{\mathcal{O}}((\beta/\mu)^{(2k-2)/2})$, and is attained without convexity. 

\vspace{.15cm}
To prove~\cref{thm: PL upper bound}, we derive a convergence guarantee for proximal SGD with generic biased, noisy stochastic gradients in terms of the bias and variance of the oracle (see Proposition~\ref{lemma:extendsAS21Thm6}). We then apply this guarantee for \texttt{MeanOracle1} (\cref{alg: MeanOracle2}) with carefully chosen stepsizes, clip threshold, and $T$, using~Lemma~\ref{lem: bias and variance of bd14}. Proposition~\ref{lemma:extendsAS21Thm6} generalizes \cite[Theorem 6]{as21}--which covered the unconstrained classical PL problem--to the 
proximal setting. However, the proof of Proposition~\ref{lemma:extendsAS21Thm6} is very different from the proof of \cite[Theorem 6]{as21}, since \texttt{prox} makes it hard to bound excess risk without convexity when the stochastic gradients are biased/noisy. Instead, our proof builds on the proof of \cite[Theorem 3.1]{lowy2022NCFL}, using techniques from the analysis of \textit{objective perturbation}~\cite{chaud, kifer2012private}. See~\cref{app: PL} for details.

\section{Concluding Remarks and Open Questions}
This paper was motivated by practical problems in which data contains outliers and potentially heavy tails, causing the worst-case Lipschitz parameter of the loss over all data points to be prohibitively large. In such cases, existing bounds for DP SO that scale with the worst-case Lipschitz parameter become vacuous. Thus, we operated under the more relaxed assumption of stochastic gradient distributions having bounded $k$-th moments. The $k$-th moment bound can be much smaller than the worst-case Lipschitz parameter in practice. For (strongly) convex loss functions, we established the 
optimal rates (up to logarithms) in certain parameter regimes, but not in every parameter regime. 
Thus, a natural problem for future work is to provide optimal risk bounds in all regimes. We believe that a promising approach will be to build on our algorithm and techniques in~\cref{sec: localization} in order to replace the $\wt{r}_{2k}$ term by $\wt{r}_k$ in our upper bounds.
We also initiated
the study of non-convex DP SO without uniform Lipschitz continuity, showing that the optimal strongly convex
rates can nearly be attained without convexity, via the proximal-PL condition. We leave the treatment of general non-convex losses for future work.

\section*{Acknowledgements}
We would like to thank John Duchi, Larry Goldstein, and Stas Minsker for very helpful conversations and pointers related to our lower bounds and the proof of Lemma~\ref{lem:4.1}. We also thank the authors of~\cite{klz21} for clarifying some steps in the proof of their Theorem 4.1 and providing useful feedback on the first draft of this manuscript. 

\clearpage
\bibliography{references}
\bibliographystyle{alpha}
\clearpage
\appendix
\section*{Appendix}

\section{Additional Discussion of Related Work}
\label{app: related work}
\textbf{DP SCO Without Uniform Lipschitz Continuity:} The study of DP SCO without uniformly Lipschitz continuous loss functions was initiated by~\cite{wx20}, who provided upper bounds for smooth convex/strongly convex loss. The work of \cite{klz21} provided lower bounds and improved, yet still suboptimal, upper bounds for the convex case. Both of the works~\cite{wx20, klz21} require $f$ to be $\beta_f$-smooth. It is also worth mentioning that~\cite{wx20, klz21} restricted attention to losses satisfying $\nabla F(\ws) = 0$ for $\ws \in \WW$, i.e. $\WW$ is a compact set containing the \textit{unconstrained} optimum $\ws = \argmin_{w \in \mathbb{R}^d} F(w) \in \WW$. By comparison, we consider the more general \textit{constrained} optimization problem $\min_{w \in \WW} F(w)$, where $\WW$ need not contain the global unconstrained optimum.  

Here we provide a brief discussion of the techniques used in~\cite{wx20,klz21}. The work of~\cite{wx20} used a full batch (clipped, noisy) gradient descent based algorithm, building on the heavy-tailed mean estimator of~\cite{hol19}. They bounded the excess risk of their algorithm by using a uniform convergence~\cite{vapnik1999overview} argument, resulting in a suboptimal dependence on the dimension $d$. 
The work of~\cite{klz21} used essentially the same approach as~\cite{wx20}, but obtained an improved rate with a more careful analysis.\footnote{Additionally, ~\cite[Theorem 5.2]{klz21} provided a bound via noisy gradient descent with the clipping mechanism of~\cite{ksu20}, but this bound is inferior (in the practical privacy regime $\varepsilon \approx 1$) to their bound in~\cite[Theorem 5.4]{klz21} that used the estimator of~\cite{hol19}.} However, as discussed, the bound in~\cite{klz21}
is $\mathcal{O}\left(r D \sqrt{\frac{d}{n}}\right)$ when $\varepsilon \approx 1$, which 
is still suboptimal.\footnote{The bound in~\cite[Theorem 5.4]{klz21} for $k=2$ is stated in the notation of~\cref{ass:coordinatewise} and thus has an extra factor of $\sqrt{d}$, compared to the bound written here. We write their bound in terms of our~\cref{ass:tilde}, replacing their $\gamma_k d$ term by $r \sqrt{d}$.} 

More recently, DP optimization with outliers was studied in special cases of sparse learning~\cite{hu2021high}, multi-arm bandits~\cite{tao2021optimal}, and 
$\ell_1$-norm linear regression~\cite{wang2022differentially}. 

\vspace{.2cm}
\noindent \textbf{DP ERM and DP GLMs without Uniform Lipschitz continuity:} 
The work of~\cite{asi2021private} provides bounds for constrained DP \textit{ERM} with arbitrary convex loss functions using a Noisy Clipped SGD algorithm that is similar to our~\cref{alg: vanilla SGD}, except that their algorithm is multi-pass and ours is one-pass. In a concurrent work, \cite{das2022beyond} considered DP \textit{ERM} in the \textit{unconstrained} setting with convex and non-convex loss functions. Their algorithm, noisy clipped SGD, is also similar to~\cref{alg: vanilla SGD} and the algorithm of~\cite{asi2021private}.
The results in \cite{das2022beyond} are not directly comparable to~\cite{asi2021private} since \cite{das2022beyond} consider the unconstrained setting while \cite{asi2021private} consider the constrained setting, but the rates in \cite{asi2021private} are faster. \cite{das2022beyond} also analyzes the convergence of noisy clipped SGD with smooth non-convex loss functions. 

The works of~\cite{song2021evading, arora2022differentially} consider \textit{generalized linear models (GLMs)}, a particular subclass of convex loss functions and provide empirical and population risk bounds for the \textit{unconstrained} DP optimization problem. The unconstrained setting is not comparable to the constrained setting that we consider here: in the unconstrained case, a dimension-independent upper bound is achievable, whereas our lower bounds (which apply to GLMs) imply that a dependence on the dimension $d$ is necessary in the constrained case. 

\vspace{.2cm}
\noindent \textbf{Other works on gradient clipping:}
The gradient clipping technique (and adaptive variants of it) has been studied empirically in works such as~\cite{abadi16, chen2020understanding, andrew2021differentially}, to name a few. The work of~\cite{chen2020understanding} shows that gradient clipping can prevent SGD from converging, and describes the clipping bias with a disparity measure between the gradient distribution and a geometrically symmetric distribution. 

\vspace{.2cm}
\noindent \textbf{Optimization with biased gradient oracles:} The works \cite{as21, asi2021private} analyze SGD with biased gradient oracles. Our work provides a tighter bound for smooth, strongly convex functions and analyzes accelerated SGD and proximal SGD with biased gradient oracles. 

\vspace{.2cm}
\noindent \textbf{DP SO with Uniformly Lipschitz loss functions:}
In the absence of outlier data, there are a multitude of works studying uniformly Lipschitz DP SO, mostly in the convex/strongly convex case. We do not attempt to provide a comprehensive list of these here, but will name the most notable ones, which provide optimal or state-of-the-art utility guarantees. The first  suboptimal bounds for DP SCO were provided in~\cite{bst14}. The work of~\cite{bft19} established the optimal rate for non-strongly convex DP SCO, by bounding the uniform stability of Noisy DP SGD (without clipping). The strongly convex case was addressed by~\cite{fkt20}, who also provided optimal rates in linear times for sufficiently smooth, convex losses. Since then, other works have provided faster and simpler (optimal) algorithms for the non-smooth DP SCO problem~\cite{bassily2020nonsmooth, asiL1geo, kulkarni2021private, bg21} and considered DP SCO with different geometries~\cite{asiL1geo, bassily2021non}. State-of-the-art rates for DP SO with the proximal PL condition are due to~\cite{lowy2022NCFL}. 

\section{Other Bounded Moment Conditions Besides~\cref{ass:tilde}}
\label{app: lemma scaling factors}
In this section, we give the alternate bounded moment assumption made in~\cite{wx20,klz21} and a third bounded moment condition, and discuss the relationships between these assumptions. The notation presented here will be necessary in order to state the sharper versions of our linear-time excess risk bounds and the asymptotically optimal excess risk bounds under the coordinate-wise assumption of~\cite{wx20,klz21} (which our~\cref{alg: localization} also attains). First, we introduce a relaxation of~\cref{ass:tilde}: 

\begin{assumption}
\label{ass:boundednoncentral}
There exists $k \geq 2$ and $r^{(k)} > 0$ such that $\sup_{w \in \WW} \expec \left[\| \nabla f(w, x) \|_2^k\right]\leq r^{(k)}$,
  for all subgradients $\nabla f(w, x_i) \in \partial_w f(w, x_i)$. Denote $r_k := (r^{(k)})^{1/k}$.
\end{assumption}

\noindent 
\cref{ass:tilde} implies~\cref{ass:boundednoncentral} for $r \leq \wt{r}$. Next, we precisely state the coordinate-wise moment bound assumption that is used in~\cite{wx20, klz21} for differentiable $f$: 

\begin{assumption}{\bf(Used by~\cite{wx20, klz21}\footnote{The work of~\cite{klz21} assumes that $L \lesssim \gamma_k^{1/k} = 1$. On the other hand, \cite{wx20} assumes that $F$ is $\beta$-smooth and $\nabla F(\ws) = 0$ for some $\ws \in \WW$, 
which
implies $L \leq 2 \beta D$.}
, but not  in this work)
}
\label{ass:coordinatewise}
There exists $k \geq 2$ and $\gamma_k > 0$ such that 
$\sup_{w \in \WW} \expec | \langle \nabla f(w,x) - \nabla F(w), e_j \rangle |^k \leq \gamma_k$,  for all 
$j \in [d]$, where $e_j$ denotes the $j$-th standard basis vector in $\mathbb{R}^d$. Also, $L \triangleq \sup_{w \in \WW}\| \nabla F(w)\| \leq \sqrt{d} \gamma_k^{1/k}$. 
\end{assumption}

\vspace{.1cm}
Lemma~\ref{lem: comparing assumptions} 
allows us  compare our results in~\cref{sec: linear time} obtained under~\cref{ass:boundednoncentral} to the results in~\cite{wx20, klz21}, which require \cref{ass:coordinatewise}. 
\begin{lemma}
\label{lem: comparing assumptions}
Suppose~\cref{ass:coordinatewise} holds. 
Then,~\cref{ass:boundednoncentral} holds for $r_k \leq 4 \sqrt{d} \gamma_k^{1/k}$.
\end{lemma}
\begin{proof}
We use the following inequality, which can easily be %
verified inductively,
using Cauchy-Schwartz and Young's inequalities: for any vectors $u, v \in \mathbb{R}^d$, we have
\begin{equation}
\label{eq: binomial}
\|u\|^k \leq 2^{k-1}\left(\|u - v\|^k + \|v\|^k \right).
\end{equation}
Therefore, \begin{align*}
r^{(k)} &= \sup_{w \in \WW} \expec\|\nabla f(w,x)\|^k \\
&\leq 2^{k-1}\left(\sup_{w \in \WW} \expec\|\nabla f(w,x) - \nabla F(w)\|^k + L^k \right) \\
&= 2^{k-1}\left(\sup_{w \in \WW} \expec\left[\left\{\sum_{j=1}^d |\langle \nabla f(w,x) - \nabla F(w), e_j \rangle |^2 \right\}^{k/2}\right] + L^k \right) \\
&\leq (2L)^k + 2^k d^{k/2} \sup_{w \in \WW}\expec\left[\frac{1}{d} \sum_{j=1}^d |\langle \nabla f(w,x) - \nabla F(w), e_j \rangle|^k \right],
\end{align*}
where we used convexity of the function $\phi(y) = y^{k/2}$ for all $y \geq 0, k \geq 2$ and Jensen's inequality in the last inequality. Now using linearity of expectation and~\cref{ass:coordinatewise} gives us \begin{align*}
r^{(k)} \leq 2^k\left(L^k + d^{k/2} \gamma_k\right) \leq 2^{k+1} d^{k/2} \gamma_k,
\end{align*} 
since $L^k \leq d^{k/2} \gamma_k$ by hypothesis.  
\end{proof}

\begin{remark}
Since~\cref{ass:boundednoncentral} is implied by \cref{ass:coordinatewise}, the upper bounds that we obtain under \cref{ass:boundednoncentral} also hold (up to constants) if we grant \cref{ass:coordinatewise} instead, with $r_k \leftrightarrow \sqrt{d} \gamma_k^{1/k}$. In particular, the upper bounds stated in~\cref{app: linear time} also imply improved upper bounds under~\cref{ass:coordinatewise}. 
Also, in~\cref{app: asymptotic}, we will use Lemma~\ref{lem: comparing assumptions} to show that our non-smooth excess risk bounds under~\cref{ass:tilde} imply similar excess risk bounds under~\cref{ass:coordinatewise}. 
\end{remark}

\section{Correcting the Errors in the Strongly Convex Upper Bounds Claimed in \cite{klz21, wx20}}
\label{app: wrong proofs}
While the ICML 2022 paper~\cite[Theorem 5.6]{klz21} claimed an upper bound for smooth strongly convex losses that is tight up to a factor of $\widetilde{\mathcal{O}}(\kappa_f^2)$---where $\kappa_f = \beta_f/\mu$ is the uniform condition number of $f(\cdot, x)$ over all $x \in \XX$--we identify an issue with their proof that invalidates their result. A similar issue appears in the proof of \cite[Theorems 5 and 7]{wx20}, which \cite{klz21} built upon. We then show how to salvage a correct upper bound within the framework of~\cite{klz21}, albeit at the cost of an additional factor of $\kappa_f$.\footnote{The corrected result of \cite{klz21}, derived here, is also included in the latest arXiv version of their paper. We communicated with the authors of~\cite{klz21} to obtain this correct version.} 

\vspace{.2cm}
The proof of \cite[Theorem 5.6]{klz21} relies on  \cite[Theorem 3.2]{klz21}. The proof of \cite[Theorem 3.2]{klz21}, in turn, bounds $\expec \|w_T - \ws\| \leq \frac{(\lambda + L)(M+1) G}{\lambda L}$ 
in the notation of \cite{klz21}, where $L$ is the smoothness parameter, $\lambda$ is the strong convexity parameter (so $L \geq \lambda$), and $M$ is the diameter of $\WW$. Then, it is \textit{incorrectly} deduced that $\expec[\|w_T - \ws\|^2] \leq \left(\frac{(\lambda + L)(M+1) G}{\lambda L}\right)^2$ (final line of the proof). Notice that $\expec[\|w_T - \ws\|^2] $ can be much larger than $(\expec \|w_T - \ws\|)^2$ in general: for example, if $\|w_T - \ws\|$ has the Pareto distribution with shape parameter $\alpha \in (1, 2]$ and scale parameter $1$, then $(\expec \|w_T - \ws\|)^2 = \left(\frac{\alpha}{\alpha - 1}\right)^2 \ll \expec(\|w_T - \ws \|^2) = \infty$. 

As a first attempt to correct this issue, one could use Young's inequality to instead bound \begin{align*}
\expec[\|w_T - \ws\|^2] &\leq 2\left(1 - \frac{2\lambda L}{(\lambda + L)^2} \right)\expec[\|w_{T-1} - \ws\|^2] + \frac{2G^2}{(\lambda + L)^2} \\
&\leq \left[2\left(1 - \frac{2\lambda L}{(\lambda + L)^2} \right)\right]^T \|w_0 - \ws\|^2 + \frac{2G^2}{(\lambda + L)^2}\sum_{t=0}^{T-1}\left[2\left(1 - \frac{2\lambda L}{(\lambda + L)^2} \right)\right]^t,
\end{align*}
but the geometric series above diverges to $+ \infty$ as $T \to \infty$, since $2\left(1 - \frac{2\lambda L}{(\lambda + L)^2} \right) \geq 1 \iff (\lambda - L)^2 \geq 0$. 
\vspace{.2cm}

Next, we show how to modify the proof of~\cite[Theorem 5.6]{klz21} in order to obtain a correct excess risk upper bound of \begin{equation}
\label{eq:iuj}
\wt{\mathcal{O}}\left(\frac{\gamma_k^{2/k}}{\mu}d\left[\frac{(\beta_f/\mu)^3}{n} + \left(\frac{\sqrt{d (\beta_f/\mu)^3}}{\varepsilon n}\right)^{(2k-2)/k}\right]\right)
\end{equation}
(in our notation). This correction was derived in collaboration with the authors of~\cite{klz21}, who have also updated the arXiv version of their paper accordingly. By waiting until the very of the proof of~\cite[Theorem 3.2]{klz21} to take expectation, we can derive \begin{align}
    \|w_t - \ws\| &\leq \left(1 - \frac{\lambda L}{(\lambda + L)^2} \right)\|w_{t-1} - \ws\| + \frac{\|\wt{\nabla} F(w_{t-1}) - \nabla F(w_{t-1})\|}{\lambda + L}
    \label{eq:iiing}
\end{align}
for all $t$, where we use their $L = \beta_f$ and $\lambda = \mu$ notation but our notation $F$ and $\wt{\nabla} F$ for the population loss and its biased noisy gradient estimate (instead of their $L_{\mathcal{D}}$ notation). By iterating~\cref{eq:iiing}, we can get 
\begin{align*}
\|w_T - \ws\| &\leq \left(1 - \frac{2\lambda L}{(\lambda + L)^2} \right)^T \|w_0 - \ws\| + \sum_{t=0}^{T-1}\left(1 - \frac{2\lambda L}{(\lambda + L)^2} \right)^t\left[\frac{\|\wt{\nabla} F(w_{T-t}) - \nabla F(w_{T-t})\|}{\lambda + L} \right].
\end{align*}
Squaring both sides and using Cauchy-Schwartz, we get \begin{align*}
\|w_T - \ws\|^2 &\leq 2\left(1 - \frac{2\lambda L}{(\lambda + L)^2} \right)^{2T} \|w_0 - \ws\|^2 + T \sum_{t=0}^{T-1}\left(1 - \frac{2\lambda L}{(\lambda + L)^2} \right)^{2t}\left[\frac{\|\wt{\nabla} F(w_{T-t}) - \nabla F(w_{T-t})\|}{\lambda + L} \right]^2.
\end{align*}

Using $L$-smoothness of $F$ and the assumption made in~\cite{klz21} that $\nabla F(\ws) = 0$, and \textit{then} taking expectation yields \begin{equation}
\label{eq:weng}
    \expec F(w_T) - F^* \lesssim L \|w_0 - \ws\|^2 \left(1 - \frac{2 \lambda L}{(\lambda + L)^2} \right)^{2T} + T G^2 \frac{L}{\lambda},
\end{equation}
where $G^2 \geq \expec\left[\|\wt{\nabla} F(w_{T-t}) - \nabla F(w_{T-t})\|^2\right]$ for all $t$. It is necessary and sufficient to choose $T = \wt{\Omega}(L/\lambda)$ to make the first term in the right-hand side of~\cref{eq:weng} less than the second term (up to logarithms). With this choice of $T$, we get \begin{equation}
\label{eq:nfd}
    \expec F(w_T) - F^* = \wt{\mathcal{O}}\left(G^2 \kappa_f^2 \right),
\end{equation}
where $\kappa_f = L/\lambda$. Next, we apply the bound on $G^2$ for the MeanOracle that is used in~\cite{klz21}; this bound is stated in the version of \cite[Lemma B.5]{klz21} that appears in the updated (November 1, 2022) arXiv version of their paper. The bound (for general $\gamma_k$) is $G^2 = \wt{\mathcal{O}} \left(\gamma_k^{2/k}\left[\frac{Td}{n} + d\left(\frac{\sqrt{d}T^{3/2}}{\varepsilon n} \right)^{(2k-2)/k} \right] \right)$. Plugging this bound on $G^2$ into \cref{eq:nfd} yields~\cref{eq:iuj}.

\section{More Differential Privacy Preliminaries}
\label{app: privacy prelims}
We collect some basic facts about DP algorithms that will be useful in the proofs of our results. Our algorithms use the \textit{Gaussian mechanism} to achieve zCDP: 
\begin{proposition}{\cite[Proposition 1.6]{bun16}}
\label{prop: gauss}
Let $q: \XX^n \to \mathbb{R}$ be a query with $\ell_2$-sensitivity $\Delta := \sup_{X \sim X'}\|q(X) - q(X')\|$. Then the Gaussian mechanism, defined by $\mathcal{M}: \XX^n \to \mathbb{R}$, $M(X) := q(X) + u$ for $u \sim \mathcal{N}(0, \sigma^2)$, is $\rho$-zCDP if $\sigma^2 \geq \frac{\Delta^2}{2\rho}$.
\end{proposition}

The (adaptive) composition of zCDP algorithms is zCDP, with privacy parameters adding:
\begin{lemma}{\cite[Lemma 2.3]{bun16}}
\label{lem: composition}
Suppose $\mathcal{A}: \XX^n \to \mathcal{Y}$ satisfies $\rho$-zCDP and $\mathcal{A}': \XX^n \times \mathcal{Y} \to \ZZ$ satisfies $\rho'$-zCDP (as a function of its first argument). Define the composition of $\mathcal{A}$ and $\mathcal{A}'$, $\mathcal{A}'': \XX^n \to \ZZ$ by
$\mathcal{A}''(X) = \mathcal{A}'(X, \mathcal{A}(X))$. Then $\mathcal{A}''$ satisfies $(\rho + \rho')$-zCDP. In particular, the composition of $T$ $\rho$-zCDP mechanisms is a $T\rho$-zCDP mechanism. 
\end{lemma}

The definitions of DP and zCDP given above do not dictate \textit{how} the algorithm $\Al$ operates. In particular, they allow $\Al$ to send sensitive data to a third party curator/analyst, who can then add noise to the data. However, in certain practical applications (e.g. federated learning~\cite{kairouz2019advances}), there is no third party that can be trusted to handle sensitive user data. On the other hand, it is often more realistic to have a secure \textit{shuffler} (a.k.a. mixnet):
in each iteration of the algorithm, the shuffler receives encrypted noisy reports (e.g. noisy stochastic gradients) from each user and applies a uniformly random permutation to the $n$ reports, thereby anonymizing them (and amplifying privacy)~\cite{prochlo, cheu2019distributed, esarevisited, fmt20}. An algorithm is \textit{shuffle private} if all of these ``shuffled'' reports are DP: 
\begin{definition}{(Shuffle Differential Privacy~\cite{prochlo, cheu2019distributed})}
\label{def: SDP}
A randomized algorithm is \textit{$(\varepsilon, \delta)$-shuffle DP} \textit{(SDP)} if the collection of reports output by the shuffler satisfies Definition~\ref{def: DP}. 
\end{definition}

\section{Details and Proofs for~\cref{sec: lower bounds}: Lower Bounds}
\label{app: lower bounds}
In this section, we prove the lower bounds stated in~\cref{sec: lower bounds}, and also provide tight lower bounds under~\cref{ass:boundednoncentral,ass:coordinatewise}. 
 \begin{theorem}[Precise Statement of \cref{thm: convex lower bound}]
Let $k \geq 2$, $D, \gamma_k, r^{(k)}, \wt{r}^{(k)} > 0$, $\beta_f \geq 0$, 
$d \geq 40, 
n > 7202$, and $\rho \leq d$. Then, for any $\rho$-zCDP algorithm $\mathcal{A}$, there exist $\WW, \XX \subset \mathbb{R}^d$ such that $\|w - w'\| \leq 2D$ for all $w, w' \in \WW$, a 
$\beta_f$-smooth, linear, convex (in $w$) loss $f: \WW \times \XX \to \mathbb{R}$, and distributions $\mathcal{D}$ and $\mathcal{D'}$ on $\XX$ such that:\\
1. \cref{ass:tilde} holds 
and if $X' \sim \mathcal{D'}^n$, then 
\begin{equation}
\expec F(\Al(X')) - F^* = \Omega\left(D \left(\frac{\wt{r}_2}{\sqrt{n}} + \wt{r}_k \min\left\{1, \left(\frac{\sqrt{d}}{\sqrt{\rho} n}\right)^{\frac{k-1}{k}}\right\}\right)\right). 
\end{equation}
\noindent 2. \cref{ass:boundednoncentral} holds 
and if $X' \sim \mathcal{D'}^n$, then 
\begin{equation}
\expec F(\Al(X')) - F^* = \Omega\left(D \left(\frac{r_2}{\sqrt{n}} + r_k\min\left\{1, \left(\frac{\sqrt{d}}{\sqrt{\rho} n}\right)^{\frac{k-1}{k}}\right\}\right)\right). 
\end{equation}
\noindent 3. 
\cref{ass:coordinatewise} holds
and if $X \sim \mathcal{D}^n$, then
\[
\expec F(\Al(X)) - F^* = \Omega\left(D \left(\gamma_2^{1/2}\sqrt{\frac{d}{n}} + \gamma_k^{1/k}\sqrt{d}\min\left\{1, \left(\frac{\sqrt{d}}{\sqrt{\rho} n}\right)^{\frac{k-1}{k}}\right\}\right)\right). 
\]
\end{theorem}
\begin{proof}
We will prove part 3 first. \\
3. We begin by proving the result for $\gamma_k = D = 1$. In this case, it is proved in \cite{klz21} that \[
\expec F(\Al(X)) - F^* = \Omega\left( \sqrt{d}\min\left\{1, \left(\frac{\sqrt{d}}{\sqrt{\rho} n}\right)^{\frac{k-1}{k}}\right\}\right)
\]
for $f(w,x) = - \langle w, x \rangle$ with $\WW = B_2^d(0,1)$ and $\XX = \{\pm 1\}^d$, and a distribution satisfying \cref{ass:coordinatewise} with $\gamma_k = 1$.
Then $f(\cdot, x)$ is linear, convex, and $\beta_f$-smooth for all $\beta_f \geq 0$. We prove the first (non-private) term in the lower bound.
By the Gilbert-Varshamov bound (see e.g. \cite[Lemma 6]{asz21}) and the assumption $d \geq 40$, there exists a set $\mathcal{V} \subseteq \{\pm 1\}^d$ with $|\mathcal{V}| \geq 2^{d/20}$, $d_{\text{Ham}}(\nu, \nu') \geq \frac{d}{8}$ for all $\nu, \nu' \in \mathcal{V}, ~\nu \neq \nu'$. For $\nu \in \mathcal{V}$, define the product distribution $Q_\nu = (Q_{\nu_1}, \cdots Q_{\nu_d})$, where for all $j \in [d]$, 
\[
Q_{\nu_j} = \begin{cases}
1 &\mbox{with probability $\frac{1 + \delta_{\nu_j}}{2}$}\\
-1 &\mbox{with probability $\frac{1 - \delta_{\nu_j}}{2}$}
\end{cases}
\]
for $\delta_{\nu_j} \in (0,1)$ to be chosen later. Then $\expec Q_{\nu_j} := \mu_{\nu_j} = \delta_{\nu_j}$ and for any $w \in \WW$, $x \sim Q_{\nu}$, we have \begin{align}
\expec |\langle \nabla f(w,x) - \nabla F(w), e_j \rangle|^2 &= \expec|\langle -x + \expec x, e_j \rangle|^k \\
&= \expec|x_j - \mu_{\nu_j}|^k \\
& \leq \expec |x_j|^k \\
&\leq 1
\end{align}
for $\delta_{\nu_j} \in (0,1)$. 
Thus, our unscaled hard instance/distribution for the non-private lower bound term (which will be chosen from among $Q_{\nu}, \nu \in \mathcal{V}$) has $k$-th moment bounded by $1$ for any $k \in [2, \infty]$. 
Now, let $p:= 
\sqrt{d/n}
$
and $\delta_{\nu_j} := \frac{p \nu_j}{\sqrt{d}} 
$. Note that $\expec Q_\nu := \mu_\nu = \frac{p \nu}{\sqrt{d}}$ and $w_\nu := \frac{\mu_\nu}{\|\mu_\nu\|} = \frac{\nu}{\|\nu\|}$. Also, $\| \mu_\nu \| = 
p
:= \| \mu\|$ for all $\nu \in \mathcal{V}$. Now, denoting $F_{Q_{\nu}}(w) := \expec_{x \sim Q_{\nu}} f(w,x)$, we have for any $w \in \WW$ (possibly depending on $X \sim Q_{\nu}^n$) that
\begin{align}
    \max_{\nu \in \mathcal{V}} \expec\left[F_{Q_{\nu}}(w) - \min_{w' \in \WW} F_{Q_{\nu}}(w')\right] &= \max_{\nu \in \mathcal{V}} \expec\left[\left \langle \frac{\mu_\nu}{\|\mu\|}, \mu_\nu \right \rangle - \left \langle w, \mu_\nu \right\rangle\right] \\
    &= \max_{\nu \in \mathcal{V}} \expec\left[ \|\mu\| - \langle w, \mu_{\nu} \rangle\right] \\
    &= \max_{\nu \in \mathcal{V}}\expec\left(\|\mu\|[1 - \langle w, w_\nu \rangle]\right) \\
    &\geq \max_{\nu \in \mathcal{V}} \expec\left[\frac{1}{2} \|\mu\| \| w - w_\nu \|^2 \right],
\end{align}
since $\|w\|, \|w_\nu\| \leq 1$. Further, denoting $\hat{w} := \argmin_{\nu \in \mathcal{V}}\|w_\nu - w\|$, we have $\|\hat{w} - w_\nu \|^2 \leq 4\|w_\nu - w\|^2$ for all $\nu \in \mathcal{V}$ (via Young's inequality). Hence \begin{equation}
\label{eq:tingyy}
\max_{\nu \in \mathcal{V}} \expec\left[F_{Q_{\nu}}(w) - \min_{w' \in \WW} F_{Q_{\nu}}(w')\right] \geq \frac{\| \mu \|}{8} \max_{\nu \in \mathcal{V}} \expec \|\hat{w} - w_\nu\|^2.
\end{equation}

Now we apply Fano's method (see e.g. \cite[Lemma 3]{yu97}) to lower bound $\max_{\nu \in \mathcal{V}} \expec \|\hat{w} - w_\nu\|^2$. For all $\nu \neq \nu'$, we have $\|w_\nu - w_{\nu'}\|^2 \geq \frac{\|\nu - \nu'\|^2}{\|\nu\|^2} \geq 1$ since $d_{\text{Ham}}(\nu, \nu') \geq \frac{d}{2}$ and $\nu \in \{\pm 1\}^d$ implies $\|\nu - \nu'\|^2 \geq \frac{d}{2}$ and $\|\nu\|^2 = d$. Also, a straightforward computation shows that for any $j \in [d]$ and $\nu, \nu' \in \mathcal{V}$, 
\begin{align}
D_{KL}(Q_{\nu_j} || Q_{\nu'_j}) &\leq \frac{1 + \frac{p}{\sqrt{d}}}{2}\left[\log\left(\frac{\sqrt{d} + p}{\sqrt{d}}\right) + \log\left(\frac{\sqrt{d}}{\sqrt{d} - p}\right)\right]\\
& \leq \log\left(\frac{1 + \frac{p}{\sqrt{d}}}{1 - \frac{p}{\sqrt{d}}}\right) \\
&\leq \frac{3p}{\sqrt{d}},
\end{align}
for our choice of $p$, provided $\frac{p}{\sqrt{d}} = \frac{1}{\sqrt{n}} \in (0, \frac{1}{2})$, which holds if 
$n > 4$. Hence by the chain rule for KL-divergence, 
\[
D_{KL}(Q_\nu || Q_{\nu'}) \leq 3p\sqrt{d} 
= 3\frac{d}{\sqrt{n}}
\]
for all $\nu, \nu' \in \mathcal{V}$. Thus, for any $w \in \WW$, Fano's method yields \[
\max_{\nu \in \mathcal{V}} \expec \|w - w_\nu\|^2 \geq \frac{1}{2}\left(1 - \frac{3p\sqrt{d} + \log(2)}{(d/20)} \right) = \frac{1}{2}\left(1 - \frac{60\frac{d}{\sqrt{n}} - 20\log(2)}{d} \right),
\]
which is $\Omega(1)$ for 
$d \geq 40 > 20 \log(2)$ and $n > 7202$.
Combining this with \cref{eq:tingyy} and plugging in $\|\mu\| = %
\sqrt{\frac{d}{n}}
$ shows that \[
\expec F_{Q_\nu}(\Al(X)) - F_{Q_\nu}^* = 
\Omega\left(\sqrt{\frac{d}{n}}\right)
\]
for some $\nu \in \mathcal{V}$ (for any algorithm $\Al$), where $X \sim Q_{\nu}^n$. 

Next, we scale our hard instance. For the non-private lower bound just proved, we scale the distribution $Q_\nu \to \tilde{Q}_\nu = \gamma_k^{1/k} Q_\nu$, which is supported on $\tilde{\XX} = \{\pm \gamma_k^{1/k}\}^d$. Denote its mean by $\expec Q_\nu := \wt{\mu}_\nu = \gamma_k^{1/k} \mu_\nu$. Also we scale $\WW \to \wt{\WW} = D \WW = B_2^d(0, D)$. So our final (linear, convex, smooth) hard instance for the non-private lower bound is $f: \wt{\WW} \times \wt{\XX} \to \mathbb{R}$, $f(\wt{w},\wt{x}) = 
- \langle \wt{w},\wt{x} \rangle$, $\wt{F}(\wt{w}) := \expec_{\wt{x} \sim \tilde{Q}_\nu} f(\wt{w}, \wt{x})$. 
Denote $F(w) := \expec_{x \sim Q_\nu} f(w, x)$. Note that
\begin{align*}
\expec |\langle \nabla f(\wt{w},\wt{x}) - \nabla \wt{F}(\wt{w}), e_j \rangle|^k &= \expec|\langle -\wt{x} + \expec \wt{x}, e_j \rangle|^k \\
&= \expec|\wt{x}_j - \wt{\mu}_{\nu_j}|^k \\
&= \expec|\gamma_k^{1/k}(x_j - \mu_{\nu_j})|^k \leq \gamma_k.
\end{align*}
Further, we have $\ws := \argmin_{w \in \WW} F(w) = \frac{\mu_\nu}{\|\mu_\nu\|}$ and $\wt{w}^* := \argmin_{\wt{w} \in \wt{\WW}} \wt{F}(\wt{w}) = D \ws$. Therefore, for any $w \in \WW, \wt{w} = Dw \in \wt{\WW}$, we have \begin{align}
\wt{F}(\wt{w}) - \wt{F}(\wt{w}^*) &= - \langle \wt{w}, \wt{\mu}_\nu \rangle + \langle \wt{w}^*, \wt{\mu}_\nu \rangle \\
&= \langle D(\ws - w), \gamma_k^{1/k} \mu_\nu \rangle\\
&= D \gamma_k^{1/k}[F(w) - F(\ws)]. 
\end{align}
Thus, \[
\expec \wt{F}(\Al(\wt{X})) - \wt{F}^* = \gamma_k^{1/k}D[ \expec F(\Al(X)) - F^*] \geq \Omega\left(\gamma_2^{1/2} D \sqrt{\frac{d}{n}}\right),
\]
by applying the non-private lower bound for the case $D = \gamma_2 = 1$ (i.e. for the unscaled $F$) and using monotonicity of moments. A similar scaling argument applied to the unscaled  \textit{private} lower bound instance of~\cite{klz21} completes the proof of part 3 of the theorem.

\noindent 1. We will use nearly the same unscaled hard instances used to prove the private and non-private terms of the lower bound in part 3, but the scaling will differ. Starting with the \textit{non-private} term, we scale the distribution $Q_\nu \to \widetilde{Q}_\nu = \frac{\wt{r}_2}{\sqrt{d}} Q_\nu$ and $\XX \to \widetilde{\XX} = \frac{\wt{r}_2}{\sqrt{d}} \XX$. Also, scale $\WW \to \wt{\WW} = D \WW = B_2^d(0, D)$. Let $f(\wt{w},\wt{x}) :=
- \langle \wt{w}, \wt{x}\rangle$, which 
satisfies all the hypotheses of the theorem. 
Also, \[
\expec_{\wt{x} \sim \wt{Q}_\nu} \left[\sup_{\wt{w}}\| \nabla f(\wt{w}, \wt{x})\|^k\right] =  \left(\frac{\wt{r}_k}{\sqrt{d}}\right)^k  \expec_{x \sim Q_\nu} \|x\|^k \leq  \left(\frac{\wt{r}_k}{\sqrt{d}}\right)^k d^{k/2} = \wt{r}^{(k)}. 
\]
Now $\wt{w}^* = D w^*$ as before and letting $\wt{F}(\cdot) := \expec_{\wt{x} \sim \wt{Q}_\nu} f(\cdot, \wt{x})$,  we have \[
\wt{F}(\wt{w}) - \wt{F}(\wt{w}^*) = \frac{\wt{r}_k D}{\sqrt{d}}[F(w) - F^*]
\]
for any $\wt{w} = Dw$. Thus, applying the unscaled non-private lower bound established above yields the desired lower bound of $\Omega\left(\frac{\wt{r}_2 D}{\sqrt{n}}\right)$ on the non-private excess risk of our scaled instance.

Next, we turn to the scaled \textit{private} lower bound. The unscaled hard distribution $Q'_\nu$ given by \[
Q'_\nu = 
\begin{cases}
0 &\mbox{with probability $1 - p$} \\
p^{-1/k} \nu &\mbox{with probability $p$}
\end{cases} 
\]
(with the same linear $f$ and same $\WW$)
provides the unscaled lower bound \[
\expec F(\Al(X)) - F^* = \Omega\left( \sqrt{d}\min\left\{1, \left(\frac{\sqrt{d}}{\sqrt{\rho} n}\right)^{\frac{k-1}{k}}\right\}\right),
\]
by the proof of \cite[Theorem 6.4]{klz21}. We scale $Q'_\nu \to \wt{Q}'_\nu = \frac{\wt{r}}{\sqrt{d}} Q'_\nu$, $\XX \to \wt{\XX} = \frac{\wt{r}}{\sqrt{d}} \XX$, and $\WW \to \wt{\WW} = D \WW$. Then for any $\wt{w} \in \wt{\WW}$, \[
\expec_{\wt{x} \sim \wt{Q'}} \left[\sup_{\wt{w}}\| \nabla f(\wt{w}, \wt{x})\|^k\right] = \left(\frac{\wt{r}_k}{\sqrt{d}}\right)^k \expec_{x \sim Q'} \|x\|^k = p\|p^{-1/k} \nu\|^k = \wt{r}^{(k)}. 
\]
Moreover, excess risk scales by a factor of $\frac{\wt{r}_k D}{\sqrt{d}}$. Thus, applying the unscaled lower bound completes the proof of part 1. \\
2. We use an identical construction to that used above in part 1 except that the scaling factor $\wt{r}_k$ gets replaced by $r_k$. 
It is easy to see that $\expec\left[\sup_{w \in \WW} \| \nabla f(w,x)\|^k\right] = \sup_{w \in \WW} \expec\left[\|\nabla f(w,x)\|^k\right]$ for our construction, hence the result follows. 
\end{proof}

\begin{remark}
\label{rem: lower bounds vs. klz}
The main differences in our proof of part 3 of \cref{thm: convex lower bound} from the proof of \cite[Theorem 6.4]{klz21} (for $\gamma_k = D = 1$) are: 1) we construct a Bernoulli product distribution (built on~\cite[Example 7.7]{duchinotes}) instead of a Gaussian, which establishes a lower bound that holds for all $k \geq 2$ instead of just $k = \mathcal{O}(1)$; and 2) we choose a different parameter value (larger $p$ in the notation of the proof) in our application of Fano's method, which results in a tighter lower bound: the term $\min\{1, \sqrt{d/n}\}$ in \cite[Theorem 6.4]{klz21} gets replaced with $\sqrt{d/n}$.\footnote{Note that \cite[Theorem 6.4]{klz21} writes $\sqrt{d/n}$ for the first term. However, the proof (see Equation 16 in their paper) only establishes the bound $\min\{1, \sqrt{d/n}\}$.} Also, there exist parameter settings for which our lower bound is indeed strictly greater than the lower bound in \cite[Theorem 6.4]{klz21}: for instance, if $d > n > d/\rho$ and $k \to \infty$, then our lower bound simplifies to $\Omega(\sqrt{\frac{d}{n}})$. On the other hand, the lower bound in \cite[Theorem 6.4]{klz21} breaks as $k \to \infty$ (since the $k$-th moment of their Gaussian goes to infinity); however, even if were extended to $k \to \infty$ (e.g. by replacing their Gaussian with our Bernoulli distribution), then the resulting lower bound $\Omega(1 + \frac{d}{\sqrt{\rho} n})$ would still be smaller than the one we prove above.\footnote{By Lemma~\ref{lem: comparing assumptions}, lower bounds under~\cref{ass:coordinatewise} imply lower bounds under~\cref{ass:boundednoncentral} with $\gamma_k^{1/k}$ replaced by $r/\sqrt{d}$. Nevertheless, we provide direct proofs under both assumptions for additional clarity.}\end{remark}

\begin{theorem}[Precise Statement of \cref{thm: strongly convex lower bound}]
Let $k \geq 2$, $\mu, \gamma_k, \wt{r}_k, r_k > 0$, $n \in \mathbb{N}$, $d \geq 40$, and $\rho \leq d$. Then, for any $\rho$-zCDP algorithm $\mathcal{A}$, there exist convex, compact sets $\WW, \XX \subset \mathbb{R}^d$ of diameter $D$, a $\mu$-smooth, $\mu$-strongly convex (in $w$) loss $f: \WW \times \XX \to \mathbb{R}$, and distributions $\mathcal{D}$ and $\mathcal{D'}$ on $\XX$ such that:\\
1. \cref{ass:tilde} holds with $D \approx \frac{\wt{r}_k}{\mu}$,
and if $X' \sim \mathcal{D'}^n$, then 
\[
\expec F(\Al(X')) - F^* = \Omega\left(\frac{1}{\mu}\left(\frac{\wt{r}_2^2}{n} + \wt{r}_k^2 \min\left\{1, \left(\frac{\sqrt{d}}{\sqrt{\rho} n}\right)^{\frac{2k-2}{k}}\right\}\right)\right). 
\]
2. \cref{ass:boundednoncentral} holds with $D \approx \frac{r_k}{\mu}$,
and if $X' \sim \mathcal{D'}^n$, then 
\[
\expec F(\Al(X')) - F^* = \Omega\left(\frac{1}{\mu}\left(\frac{r_2^2}{n} + r_k^2\min\left\{1, \left(\frac{\sqrt{d}}{\sqrt{\rho} n}\right)^{\frac{2k-2}{k}}\right\}\right)\right). 
\]
3. %
\cref{ass:coordinatewise} holds, 
$D \approx \frac{\gamma_k^{1/k} \sqrt{d}}{\mu}$,
and if $X \sim \mathcal{D}^n$, then
\[
\expec F(\Al(X)) - F^* = \Omega\left(\frac{1}{\mu}\left(\frac{\gamma_2^{2/k} d}{n} + \gamma_k^{2/k}d\min\left\{1, \left(\frac{\sqrt{d}}{\sqrt{\rho} n}\right)^{\frac{2k-2}{k}}\right\}\right)\right). 
\]
\end{theorem}
\begin{proof}
We will prove part 3 first. 
3. We first consider $\gamma_k = \mu = 1$ and then scale our hard instance. For $f(w,x) := \frac{1}{2}\|w - x\|^2$, 
\cite{klz21} construct a convex/compact domain $\WW \times \XX \subset \mathbb{R}^d \times \mathbb{R}^d$ and distribution $\DD$ on $\XX$ such that \[
\expec F(\Al(X)) - F^* = \Omega\left( d\min\left\{1, \left(\frac{\sqrt{d}}{\sqrt{\rho} n}\right)^{\frac{2k-2}{k}}\right\}\right) 
\]
for any $k$ and any $\rho$-zCDP algorithm $\Al: \XX^n \to \WW$ if $X \sim \DD^n$.\footnote{In fact, $\WW$ and $\XX$ can be chosen to be Euclidean balls of radius $\sqrt{d} p^{-1/k}$ for $p$ defined in the proof of \cite[Lemma 6.3]{klz21}, which ensures that $\expec \DD \in \WW = \XX$.} So, it remains to a) prove the first term ($d/n$) in the lower bound, and then b) show that the scaled instance satisfies the exact hypotheses in the theorem and has excess loss that scales by a factor of $\gamma_k^{2/k}/\mu$. We start with task a). Observe that for $f$ defined above and any distribution $\DD$ such that $\expec \DD \in \WW$, we have 
\begin{equation}
\label{eq:reductiontoME}
    \expec F(\Al(X)) - F^* = \frac{1}{2}\expec \| \Al(X) - \expec \DD\|^2
\end{equation}
(see \cite[Lemma 6.2]{klz21}), and \[
\expec | \langle \nabla f(w,x) - \nabla F(w), e_j \rangle |^k = \expec | \langle x - \expec x, e_j \rangle |^k.
\] Thus, it suffices to prove that $\expec \| \Al(X) - \expec \DD\|^2 \gtrsim \frac{d}{n}$ for some $\DD$ such that $\expec | \langle x - \expec x, e_j \rangle |^k \leq 1$. This is a known result for products of Bernoulli distributions; nevertheless, we provide a detailed proof below. First consider the case $d=1$. Then the proof follows along the lines of \cite[Example 7.7]{duchinotes}. Define the following pair of distributions on $\{\pm 1\}$: \[
P_0 := 
\begin{cases}
1 &\mbox{with probability $\frac{1}{2}$} \\
-1 &\mbox{with probability $\frac{1}{2}$}
\end{cases}\] and 
\[P_1 := 
\begin{cases}
1 &\mbox{with probability $\frac{1 + \delta}{2}$} \\
-1 &\mbox{with probability $\frac{1 - \delta}{2}$}
\end{cases}\]
for $\delta \in (0,1)$ to be chosen later. Notice that if $X$ is a random variable with distribution $P_{\nu}$ ($\nu \in \{0,1\}$), then $\expec|X - \mu |^k \leq \expec|X|^k \leq 1$. Also, $\expec P_{\nu} = \delta \nu$ for $\nu \in \{0,1\}$ and $|\expec P_{1} - \expec P_{0}| = \delta$ (i.e. the two distributions are $\delta$-separated with respect to the metric $\rho(a,b) = |a - b|$). Then by LeCam's method (see e.g. \cite[Eq. 7.33]{duchinotes} and take $\Phi(\cdot) = (\cdot)^2$), 
\[
\max_{\nu \in \{0,1\}} \expec_{X \sim P_{\nu}^n} |\Al(X) - \expec P_{\nu}|^2 \geq \frac{\delta^2}{8}\left[1 - \|P_0^n - P_1^n\|_{TV} \right].
\] Now, by Pinsker's inequality and the chain rule for KL-divergence, we have \[
\|P_0^n - P_1^n\|_{TV}^2 \leq \frac{1}{2} D_{KL}(P_0^n || P_1^n) = \frac{n}{2} D_{KL}(P_0 || P_1) = \frac{n}{2}\log\left(\frac{1}{1-\delta^2}\right).
\]
Choosing $\delta = \frac{1}{\sqrt{2n}} < \frac{1}{\sqrt{2}}$ implies $\|P_0^n - P_1^n\|_{TV}^2 \leq n \delta^2 = \frac{1}{2}$. Hence there exists a distribution $\mathcal{\hat{D}} \in \{P_0, P_1\}$ on $\mathbb{R}$ such that \begin{equation*}
\expec_{X \sim \mathcal{\hat{D}}^n} |\Al(X) - \expec \mathcal{\hat{D}}|^2 \geq \frac{\delta^2}{8}\left[1 - \frac{1}{\sqrt{2}}\right] \geq \frac{1}{64n}
\end{equation*}
For general $d \geq 1$, we take the product distribution $\mathcal{D} := \mathcal{\hat{D}}^d$ on $\XX = \{\pm 1\}^d$ and choose $\WW = B_2^d(0, \sqrt{d})$ to ensure $\expec \DD \in \WW$ (so that~\cref{eq:reductiontoME} holds). Clearly, $\expec|\langle \DD - \expec \DD, e_j \rangle|^k \leq 1$ for all $j \in [d]$. Further, the mean squared error of any algorithm for estimating the mean of $\DD$ is \begin{equation}
\label{eq:unscalednonprivLB}
\expec_{X \sim \DD^n} \|\Al(X) - \expec \DD\|^2 = \sum_{j=1}^d \expec |\Al(X)_j - \expec \DD_j|^2 \geq \frac{d}{64n},
\end{equation}
by applying the $d=1$ result to each coordinate.  

Next, we move to task b). For this, we re-scale each of our hard distributions (non-private given above, and private given in the proof of \cite[Lemma 6.3]{klz21} and below in our proof of part 2 of the theorem--see \cref{eq:harddist}): $\mathcal{D} \to \frac{\gamma_k^{1/k}}{\mu}\mathcal{D} = \widetilde{\DD}$,
$\XX \to \frac{\gamma_k^{1/k}}{\mu} \XX = \wt{\XX}$, $\WW \to \frac{\gamma_k^{1/k}}{\mu} \WW = \wt{\WW}$ and $f: \WW \times \XX \to \mu f = \widetilde{f}: \widetilde{\WW} \times \widetilde{\XX}$. Then $\widetilde{f}(\cdot, \widetilde{x})$ is $\mu$-strongly convex and $\mu$-smooth for all $\widetilde{x} \in \widetilde{\XX}$ and \[
\expec | \langle \nabla \wt{f}(\wt{w}, \wt{x}) - \nabla \wt{F}(\wt{w}), e_j \rangle |^k = \mu^k \expec | \langle \wt{x} - \expec \wt{x}, e_j \rangle |^k = \mu^k \expec \left|\left(\frac{\gamma_k^{1/k}}{\mu}\right) \langle x - \expec x, e_j \rangle \right|^k = \gamma_k \expec |\langle x - \expec x, e_j \rangle |^k \leq \gamma_k
\]
for any $j \in [d]$, $x \sim \DD, ~\wt{x} \sim \wt{\DD}, ~\wt{w} \in \wt{\WW}$. Thus, the scaled hard instance is in the required class of functions/distributions. Further, denote $\widetilde{F}(w) = \expec \widetilde{f}(w,x)$, $\wt{w}^* := \argmin_{\wt{w} \in \wt{\WW}} \wt{F}(\wt{w}) = \expec \wt{\DD} = \frac{\gamma_k^{1/k}}{\mu} \expec \DD$. Then, for any $w \in \WW, ~\wt{w}:= \frac{\gamma_k^{1/k}}{\mu}w$, we have:
\begin{align}
\wt{F}(\wt{w}) - \wt{F}(\wt{w}^*) &= \frac{\mu}{2}\expec\left[\|\wt{w} - \wt{x}\|^2 - \|\wt{w}^* - \wt{x}\|^2 \right] \\
&= \frac{\mu}{2}\left(\frac{\gamma_k^{2/k}}{\mu^2}\right)\expec\left[\|w - x\|^2 - \|w^* - x\|^2 \right] \\
&= \frac{\gamma_k^{2/k}}{\mu}[F(w) - F(w^*)].
\end{align}
In particular, for $w := \mathcal{A}(X)$ and $\wt{w} := \frac{\gamma_k^{1/k}}{\mu} \mathcal{A}(X)$, we get \[
\expec_{\Al, X \sim \DD^n} \left[ \wt{F}\left( \frac{\gamma_k^{1/k}}{\mu} \mathcal{A}(X)\right) - \wt{F}^* \right] = \frac{\gamma_k^{2/k}}{\mu}\expec_{\Al, X\sim \DD^n}\left[ F(\mathcal{A}(X)) - F^*\right]
\]
for any algorithm $\Al: \XX^n \to \WW$. Writing $\tilde{\Al}(\tilde{X}):=\frac{\gamma_k^{1/k}}{\mu} \mathcal{A}(X)$ and $\tilde{X} := \frac{\gamma_k^{1/k}}{\mu} X$ for $X \in \XX^n$, we conclude \[
\expec_{\wt{\Al}, \tilde{X} \sim \tilde{\DD}^n} \left[ \wt{F}\left( \tilde{\Al}(\tilde{X})\right) - \wt{F}^* \right] = \frac{\gamma_k^{2/k}}{\mu}\expec_{\Al, X\sim \DD^n}\left[ F(\mathcal{A}(X)) - F^*\right] 
\]
for any $\wt{\Al}: \wt{\XX}^n \to \wt{\WW}$. Therefore, an application of the unscaled lower bound \[
\expec_{\Al, X\sim \DD^n}\left[ F(\mathcal{A}(X)) - F^*\right] = \Omega\left(\frac{d}{n} + d\min\left\{1,  \left(\frac{\sqrt{d}}{n\sqrt{\rho}}\right)^{\frac{2k-2}{k}} \right\}\right),
\]
which follows by combining part 3a) above with \cite[Lemma 6.3]{klz21}, completes the proof of part 3.  \\

\noindent 1. We begin by proving the first (non-private) term in the lower bound: For our \textit{unscaled} hard instance, we will take the same distribution $\mathcal{D} = P_\nu^d$ (for some $\nu \in \{0,1\}$) on $\XX = \{\pm 1\}^d$ and quadratic $f$ described above in part 1a with $\WW := B_2^d(0, \sqrt{d})$. The choice of $\WW$ ensures $\expec \DD \in \WW$ so that \cref{eq:reductiontoME} holds. Further, %
\[
\expec \left[\sup_{w \in \WW} \| \nabla f(w,x) \|^k\right] = \expec \left[\sup_{w \in \WW} \| w - x \|^k \right] \leq \expec[\|3x\|^k] \leq (9d)^{k/2}.
\]
Thus, if we scale $f \to \widetilde{f} = \mu f$, $\WW \to \widetilde{\WW} := \frac{\wt{r}_k}{\mu \sqrt{9d}} \WW$, $\XX \to \widetilde{\XX} := \frac{\wt{r}_k}{\mu \sqrt{9d}}\XX$ and $\mathcal{D} \to \widetilde{\mathcal{D}} = \frac{\wt{r}_k}{\mu \sqrt{9d}} \mathcal{D}$, then $\wt{f}(\cdot, \widetilde{x})$ is $\mu$-strongly convex and $\mu$-smooth, and \[
\expec \left[\sup_{\wt{w} \in \wt{\WW}} \left\| \nabla \widetilde{f}(\widetilde{w}, \widetilde{x}) \right\|^k\right] = \expec \left[\sup_{\wt{w} \in \wt{\WW}}\left\| \widetilde{w} - \widetilde{x} \right\|^k\right] = \mu^k \wt{r}^{(k)} \left(\frac{1}{\mu \sqrt{9d}}\right)^k \expec\left[\sup_{w \in \WW} \left\| \nabla f(w,x) \right\|^k \right] \leq \wt{r}^{(k)}.
\]
Moreover, if $\left(\frac{\wt{r}_k}{3 \mu \sqrt{d}}\right) \mathcal{A} = \widetilde{\mathcal{A}}: \widetilde{X}^n \to \widetilde{\WW}$ is any algorithm and $\widetilde{X} \sim \widetilde{\mathcal{D}}^n$, then by \cref{eq:unscalednonprivLB} and \cref{eq:reductiontoME}, we have \[
\expec \widetilde{F}(\widetilde{\mathcal{A}}(\widetilde{X})) - \widetilde{F}^* = \frac{\mu}{2}\expec \|\widetilde{\mathcal{A}}(\widetilde{X}) - \expec \widetilde{\mathcal{D}} \|^2 = \frac{\mu}{2}\left(\frac{\wt{r}_k}{\mu \sqrt{9d}}\right)^2 \expec \|\mathcal{A}(X) -  \expec \mathcal{D}\|^2 \gtrsim \frac{\wt{r}_k^2}{\mu n}. 
\]

Next, we prove the second (private) term in the lower bound. Let $f$ be as defined above. For our unscaled hard distribution, we follow \cite{bd14, klz21} and define a family of distributions $\{Q_\nu\}_{\nu \in \mathcal{V}}$ on $\mathbb{R}^d$, where $\mathcal{V} \subset \{\pm 1\}^d$ will be defined later. For any given $\nu \in \mathcal{V}$, we define the distribution $Q_\nu$ as follows: $X_\nu \sim Q_\nu$ iff \begin{equation}
\label{eq:harddist}
X_\nu = 
\begin{cases}
0 &\mbox{with probability $1 - p$} \\
p^{-1/k} \nu &\mbox{with probability $p$}
\end{cases}
\end{equation}
where $p := \min\left(1, \frac{\sqrt{d}}{n \sqrt{\rho}}\right)$. Now, we select a set $\mathcal{V} \subset \{\pm 1\}^d$ such that $|\mathcal{V}| \geq 2^{d/20}$ and $d_{\text{Ham}}(\nu, \nu') \geq \frac{d}{8}$ for all $\nu, \nu' \in \mathcal{V}, \nu \neq \nu'$: such $\mathcal{V}$ exists by standard Gilbert-Varshamov bound (see e.g. \cite[Lemma 6]{asz21}). For any $\nu \in \mathcal{V}$, if $x \sim Q_\nu$ and $w \in \WW := B_2^d(0,\sqrt{d} p^{-1/k})$, then \[
\expec \left[\sup_{w \in \WW} \| \nabla f(w,x) \|^k \right] = \expec \left[\sup_{w \in \WW} \| w - x \|^k\right] \leq \expec[\|2x\|^k] = 2^k(p\|p^{-1/k} \nu\|^k) = 2^k\|\nu\|^k = 2^k d^{k/2}.
\]
Note also that our choice of $\WW$ and $p \leq 1$ ensures that $\expec[Q_\nu] \in \WW$. 
Moreover, as in the proof of \cite[Lemma 6.3]{klz21}, zCDP Fano's inequality (see \cite[Theorem 1.4]{klz21}) implies that for any $\rho$-zCDP algorithm $\mathcal{A}$, \begin{equation}
\label{eq:tt}
\sup_{\nu \in \mathcal{V}} \expec_{X \sim Q_\nu^n, \mathcal{A}} \|\mathcal{A}(X) - \expec Q_\nu\|^2 = \Omega\left(d \min\left\{1, \left(\frac{\sqrt{d}}{n\sqrt{\rho}}\right)^{\frac{2k-2}{k}}\right\} \right).
\end{equation}
Thus, \[
\expec_{X \sim Q_\nu^n, \mathcal{A}} F(\mathcal{A}(X)) - F^* = \Omega\left(d \min\left\{1, \left(\frac{\sqrt{d}}{n\sqrt{\rho}}\right)^{\frac{2k-2}{k}}\right\} \right)  
\]
for some $\nu \in \mathcal{V}$, by \cref{eq:reductiontoME}. Now we scale our hard instance: $f \to \widetilde{f} = \mu f$, $\WW \to \widetilde{\WW} := \frac{\wt{r}_k}{2 \mu \sqrt{d}} \WW$, $\XX \to \widetilde{\XX} := \frac{\wt{r}_k}{2 \mu \sqrt{d}}\XX$ and $\mathcal{D} \to \widetilde{\mathcal{D}} = \frac{\wt{r}_k}{2 \mu \sqrt{d}} \mathcal{D}$. Then $\wt{f}(\cdot, \widetilde{x})$ is $\mu$-strongly convex and $\mu$-smooth, and \[
\expec \left[\sup_{\wt{w} \in \wt{\WW}}\| \nabla \widetilde{f}(\widetilde{w}, \widetilde{x}) \|^k\right] = \expec \left[\sup_{\wt{w} \in \wt{\WW}}\| \widetilde{w} - \widetilde{x} \|^k \right]= \mu^k \left(\frac{\wt{r}_k}{2 \mu \sqrt{d}}\right)^k \expec \left[\sup_{w \in \WW}\| \nabla f(w,x) \|^k \right]\leq \wt{r}^{(k)}.
\]
Moreover, if $\left(\frac{\wt{r}_k}{2 \mu \sqrt{d}}\right) \mathcal{A} = \widetilde{\mathcal{A}}: \widetilde{X}^n \to \widetilde{\WW}$ is any $\rho$-zCDP algorithm and $\widetilde{X} \sim \widetilde{\mathcal{D}}^n$, then  
\begin{equation} \nonumber
\begin{split}
\expec \widetilde{F}(\widetilde{\mathcal{A}}(\widetilde{X})) - \widetilde{F}^* &= \frac{\mu}{2}\expec \|\widetilde{\mathcal{A}}(\widetilde{X}) - \expec \widetilde{\mathcal{D}} \|^2 \\
&= \frac{\mu}{2}\left(\frac{\wt{r}_k}{2 \mu \sqrt{d}}\right)^2 \expec \|\mathcal{A}(X) -  \expec \mathcal{D}\|^2 \\
&\geq \frac{\wt{r}_k^2}{16 \mu d}\Omega\left(d \min\left\{1, \left(\frac{\sqrt{d}}{n\sqrt{\rho}}\right)^{\frac{2k-2}{k}}\right\} \right),
\end{split}
\end{equation}
by \cref{eq:tt}. \\
2. We use an identical construction to that used above in part 1 except that the scaling factor $\wt{r}_k$ gets replaced by $r_k$. 
It is easy to see that $\expec\left[\sup_{w \in \WW} \| \nabla f(w,x)\|^k\right] \approx \sup_{w \in \WW} \expec\left[\|\nabla f(w,x)\|^k\right]$ for our construction, and the lower bound in part 2 follows just as it did in part 1. This completes the proof. 
\end{proof}

\begin{remark}
Note that the lower bound proofs construct bounded (hence subexponential) distributions and uniformly $L_f$-Lipschitz, $\beta_f$-smooth losses that easily satisfy the conditions in all of our upper bound theorems, including those in~\cref{app: asymptotic}.
\end{remark}

\section{Details and Proofs for~\cref{sec: linear time}: Linear Time Algorithms for Smooth (Strongly) Convex Losses}
\label{app: linear time}

\subsection{Noisy Clipped Accelerated SGD for Smooth Convex Losses (\cref{sec: convex})}
\label{app: convex accel}
We first present~\cref{alg: generic ACSA}, which is a generalized version of \cref{alg: ACSA} that allows for any \texttt{MeanOracle}. This will be useful for our analysis. 
\begin{algorithm}[ht]
\caption{Generic Framework for DP Accelerated Stochastic Approximation (AC-SA)}
\label{alg: generic ACSA}
\begin{algorithmic}[1]
\STATE {\bfseries Input:} 
Data $X \in \XX^n$, iteration number $T \leq n$, %
stepsize parameters $\{\eta_t \}_{t \in [T]}, \{\alpha_t \}_{t \in [T]}$ with $\alpha_1 = 1, \alpha_t \in (0,1)$~$\forall t \geq 2$, 
DP $\texttt{MeanOracle}$.
\STATE Initialize $w_0^{ag} = w_0 \in \WW$ and $t = 1$. 
 \FOR{$t \in [T]$} 
 \STATE $w_t^{md} := 
 (1- \alpha_t)w_{t-1}^{ag} + \alpha_t w_{t-1}$. 
 \STATE Draw new batch $\mathcal{B}_t$ (without replacement) of $n/T$ samples from $X$.
 \STATE $\tilt(w_t^{md}) := \texttt{MeanOracle}(\{\nabla f(w_t^{md}, x)\}_{x \in \mathcal{B}_t}; \frac{n}{T}; \frac{\varepsilon^2}{2})$ 
 \STATE $w_{t} := 
 \argmin_{w \in \mathcal{W}}\left\{\alpha_t\langle \tilt(w_t^{md}), w\rangle + \frac{\eta_t}{2}\|w_{t-1} - w\|^2\right\}.
 $
 \STATE $w_{t}^{ag} := \alpha_t w_t + (1-\alpha_t)w_{t-1}^{ag}.$
\ENDFOR \\
\STATE {\bfseries Output:} $w_T^{ag}$.
\end{algorithmic}
\end{algorithm}
Proposition~\ref{prop: ACSA generic} provides excess risk guarantees for \cref{alg: generic ACSA} in terms of the bias and variance of the \texttt{MeanOracle}.
\begin{proposition}
\label{prop: ACSA generic}
Consider \cref{alg: generic ACSA} run with a \texttt{MeanOracle} satisfying $\tilt(w_t^{md}) = \nabla F(w_t^{md}) + b_t + N_t$, where $\|b_t\| \leq B$ (with probability $1$), $\expec N_t = 0$, $\expec\|N_t\|^2 \leq \Sigma^2$ for all $t \in [T-1]$, and $\{N_t\}_{t=1}^T$ are independent. Assume that $F: \WW \to \mathbb{R}$ is convex and  $\beta$-smooth, $F(w_0) - F^* \leq \Delta$, and $\|w_0 - \ws\| \leq D$. Suppose parameters are chosen in \cref{alg: generic ACSA} so that for all $t \in [T]$, $\eta_t > \beta \alpha_t^2$ and $\eta_t/\gamma_t = \eta_1/\gamma_1$, where  \[
\gamma_t := 
\begin{cases}
1, &\mbox{$t = 1$}\\
(1 - \alpha_t) \gamma_t, &\mbox{$t \geq 2$}.
\end{cases}
\]
Then, 
\[
\expec F(w_T^{ag}) - F^* \leq  \frac{\gamma_T \eta_1 D^2}{2} + \gamma_T \sum_{t=1}^T\left[\frac{2 \alpha_t^2(\Sigma^2 + B^2)}{\gamma_t(\eta_t - \beta \alpha_t^2)} + \frac{\alpha_t}{\gamma_t} BD\right].
\]
In particular, choosing $\alpha_t = \frac{2}{t + 1}$ and $\eta_t = \frac{4\eta}{t(t+1)}, ~\forall t \geq 1$, where $\eta \geq 2\beta$ implies \[
\expec F(w_T^{ag}) - F^* \leq \frac{4 \eta D^2}{T(T+1)} + \frac{4(\Sigma^2 + B^2)(T+2)}{3\eta} + BD.
\]
Further, setting $\eta = \max\left\{ 2 \beta, \frac{T^{3/2}\sqrt{\Sigma^2 + B^2}}{D}\right\}$ implies \begin{equation}
\expec F(w_T^{ag}) - F^* \lesssim \frac{\beta D^2}{T^2} + \frac{D(\Sigma + B)}{\sqrt{T}} + BD. 
\end{equation}
\end{proposition}

\begin{proof}
We begin by extending \cite[Proposition 4]{ghadimilan1} to biased/noisy stochastic gradients. Fix any $w_{t-1}, w_{t-1}^{ag} \in \WW$. By \cite[Lemma 3]{ghadimilan1}, we have \begin{equation}
\label{eq: gl lem3}
    F(w_t^{ag}) \leq (1-\alpha_t) F(w_{t-1}^{ag}) + \alpha[F(z) + \langle \nabla F(z), w_t - z \rangle] + \frac{\beta}{2}\|w_t^{ag} - z\|^2,
\end{equation}
for any $z \in \WW$. Denote \[
\Upsilon_t(w) := \alpha_t \langle N_t + b_t, w - w_{t-1} \rangle + \frac{\alpha_t^2 \|N_t + b_t\|^2}{\eta_t - \beta \alpha_t^2}
\]
and $d_t := w_t^{ag} - w_t^{md} = \alpha_t(w_t - w_{t-1})$. Then using \cref{eq: gl lem3} with $z = w_t^{md}$, we have \begin{align}
\label{eq: 3.20}
    F(w_t^{ag}) &\leq (1 - \alpha_t) F(w_{t-1}^{ag}) + \alpha_t[F(w_t^{md}) + \langle \nabla F(w_t^{md}), w_t - w_t^{md} \rangle] + \frac{\beta}{2}\|d_t\|^2 \nonumber \\
    &= (1 - \alpha_t) F(w_{t-1}^{ag}) + \alpha_t[F(w_t^{md}) + \langle \nabla F(w_t^{md}), w_t - w_t^{md}\rangle] \nonumber\\ 
    &\quad\quad \quad+ \frac{\eta_t}{2}\|w_{t-1} - w_t\|^2 - \frac{\eta_t - \beta \alpha_t^2}{2 \alpha_t^2}\|d_t\|^2,
\end{align}
by the expression for $d_t$. Now we apply \cite[Lemma 2]{ghadimilan1} with $p(u) = \alpha_t [\langle \tilt(w_t^{md}, u \rangle], ~\mu_1 = 0, ~\mu_2 = \eta_t, ~ \tilde{x} = w_t^{md}$, and $\tilde{y} = w_{t-1}$ to obtain (conditional on all randomness) for any $w \in \WW$: 
\begin{align*}
    &\alpha_t[F(w_t^{md}) + \langle \tilt(w_t^{md}), w_t - w_t^{md} \rangle] + \frac{\eta_t}{2}\|w_{t-1} - w_t\|^2 \\
    &\leq \alpha_t[F(w_t^{md}) + \langle \nabla F(w_t^{md}), w - w_t^{md}\rangle] \\
    &\;\;\;+ \alpha_t \langle N_t + b_t, w - w_t^{md} \rangle + \frac{\eta_t}{2}\|w_{t-1} - w\|^2 - \frac{\eta_t}{2}\|w_t - w\|^2.
\end{align*}
Next, we combine the above inequality with \cref{eq: 3.20} to get \begin{align}
\label{eq: 75}
    F(w_t^{ag}) &\leq (1-\alpha_t)F(w_{t-1}^{ag}) + \alpha_t[F(w_t^{md}) + \langle \nabla F(w_t^{md}), w - w_t^{md}\rangle] \nonumber\\
    &\;\;\; + \frac{\eta_t}{2}\left[\|w_{t-1} - w\|^2 - \|w_t - w\|^2 \right] %
     + \underbrace{-\frac{\eta_t - \beta \alpha_t^2}{2\alpha_t^2}\|d_t\|^2 + \alpha_t\langle N_t + b_t, w - w_t\rangle }_{U_t},
\end{align}
for all $w \in \WW$. By Cauchy-Schwartz, we can bound \begin{align}
\label{eq: U_t}
    U_t &\leq - \frac{\eta_t - \beta \alpha_t^2}{2 \alpha_t^2}\|d_t\|^2 + \|N_t + b_t\| \|d_t\| + \alpha_t \langle N_t + b_t, w - w_{t-1}\rangle \nonumber \\
    &\leq \Upsilon_t(w),
\end{align} 
where the last inequality follows from maximizing the concave quadratic function $q(\|d_t\|) := -\left[\frac{\eta_t - \beta \alpha_t^2}{2\alpha_t^2}\right]\|d_t\|^2 + \|N_t + b_t\| \|d_t\|$ with respect to $\|d_t\|$. Plugging the bound \cref{eq: U_t} back into \cref{eq: 75} shows that \begin{align}
\label{eq: 77}
    F(w_t^{ag}) &\leq (1-\alpha_t)F(w_{t-1}^{ag}) + \alpha_t[F(w_t^{md}) + \langle \nabla F(w_t^{md}), w - w_t^{md}\rangle] + \frac{\eta_t}{2}\left[\|w_{t-1} - w\|^2 - \|w_t - w\|^2 \right] \nonumber \\
    &\;\;\;+\Upsilon_t(w). 
\end{align}
Then it can be shown (see \cite[Proposition 5]{ghadimilan1}) that the assumptions on $\eta_t$ and $\alpha_t$ imply that \begin{align}
\label{eq: gl prop5}
    F(w_T^{ag}) - \gamma_T \sum_{t=1}^T\left[\frac{\alpha_t}{\gamma_t}\left(F(w_t^{md}) + \langle \nabla F(w_t^{md}), w - w_t^{md} \rangle \right) \right] 
    &\leq \gamma_T \sum_{t=1}^T \frac{\eta_t}{2\gamma_t}[\|w_{t-1} - w\|^2 - \|w_t - w\|^2] \\
    &\;\;\; + \gamma_T \sum_{t=1}^T \frac{\Upsilon_t(w)}{\gamma_t},
\end{align}
for any $w \in \WW$ and any $T \geq 1$. Now, \[
\sum_{t=1}^T \frac{\alpha_t}{\gamma_t} = \frac{1}{\gamma_T}
\]
by definition. Hence by convexity of $F$, \[
\sum_{t=1}^T\left[\frac{\alpha_t}{\gamma_t}\left(F(w_t^{md}) + \langle \nabla F(w_t^{md}), w - w_t^{md} \rangle \right) \right] \leq F(w), ~\forall w \in \WW. 
\]
Also, since $\gamma_t/\eta_t = \gamma_1/\eta_1$ for all $t \geq 1$, we have \[
\gamma_T \sum_{t=1}^T \frac{\eta_t}{2\gamma_t}[\|w_{t-1} - w\|^2 - \|w_t - w\|^2] = \gamma_t \frac{\eta_1}{2\gamma_1}[\|w_{0} - w\|^2 - \|w_T - w\|^2] \leq \gamma_T \eta_1 \frac{1}{2}\|w_0 - w\|^2,
\]
since $\gamma_1 = 1$. Substituting the above bounds into \cref{eq: gl prop5}, we get \begin{align}
    F(w_T^{ag}) - F(w) &\leq \gamma_T \eta_1  \frac{1}{2}\|w_0 - w\|^2 + \gamma_T \sum_{t=1}^T \frac{\Upsilon_t(w)}{\gamma_t}, ~\forall w \in \WW. 
\end{align}
Now, setting $w = \ws$ and taking expectation yields \begin{align}
    \expec[F(w_T^{ag}) - F^*] &\leq \frac{\gamma_T \eta_1  D^2}{2} + \gamma_T\sum_{t=1}^T \frac{\expec \Upsilon_t(\ws)}{\gamma_t} \\
    &\leq \frac{\gamma_T \eta_1  D^2}{2} + \gamma_T\sum_{t=1}^T \left[\frac{1}{\gamma_t}\left(\alpha_t \expec\langle b_t, \ws - w_{t-1} \rangle + \frac{2\alpha_t^2(\Sigma^2 + B^2)}{\eta_t - \beta \alpha_t^2}\right) \right] \\
    &\leq \frac{\gamma_T \eta_1  D^2}{2} + \gamma_T\sum_{t=1}^T \left[\frac{1}{\gamma_t}\left(\alpha_t BD + \frac{2\alpha_t^2(\Sigma^2 + B^2)}{\eta_t - \beta \alpha_t^2}\right) \right],
\end{align}
where we used conditional independence of $N_t$ and $\ws - w_{t-1}$ given $w_{t-1}$, Young's inequality, Cauchy-Schwartz, and the definitions of $B^2$ and $\Sigma^2$. This establishes the first claim of the theorem. The second and third claims are simple corollaries, which can be verified as in \cite[Proposition 7]{ghadimilan1} and the ensuing discussion. 
\end{proof}

\begin{theorem}[Complete Version of~\cref{thm: convex ACSA one pass}]
Grant~\cref{ass:boundednoncentral}.
Let $\varepsilon > 0$ and assume $F$ is convex and $\beta$-smooth. Then, there are algorithmic parameters such that \cref{alg: ACSA}
is $\frac{\varepsilon^2}{2}$-zCDP. Further,
if \[n \geq T := \Bigg \lceil \min\left\{
\left(\frac{\beta D}{r_k}\right)^{2k/(5k-1)}
\left(\frac{\varepsilon n}{\sqrt{d}}\right)^{(2k-2)/(5k-1)}, \sqrt{\frac{\beta D}{r_k}}n^{1/4}\right\} \Bigg \rceil,\] then, 
\[
\eplac \lesssim 
\frac{r_2 D}{\sqrt{n}} + r_k D\left[
\left(\frac{\sqrt{d}}{\varepsilon n}\right)^{\frac{k-1}{k}} + \min\left\{
\left(\left(\frac{\beta D}{r_k}\right)^{1/4} \frac{\sqrt{d}}{\varepsilon n} \right)^{\frac{4(k-1)}{5k-1}},
\left(\frac{\beta D}{r_k}\right)^{\frac{k-1}{4k}}\left(\frac{\sqrt{d}}{\varepsilon n}\right)^{\frac{k-1}{k}}n^{\frac{k-1}{8k}}
\right\}
\right].
\]
\end{theorem}

\begin{proof}
\textbf{Privacy:} Choose $\sigma^2 = \frac{4C^2 T^2}{\varepsilon^2 n^2}$. First, the collection of all $\tilt(w_t^{md}), ~t \in [T]$ is $\frac{\varepsilon^2}{2}$-zCDP: since the batches of data drawn in each iteration are disjoint, it suffices (by parallel composition~\cite{mcsherry2009privacy}) to show that $\tilt(w_t^{md})$ is $\frac{\varepsilon^2}{2}$-zCDP for all $t$. Now, the $\ell_2$ sensitivity of each clipped gradient update is bounded by $\Delta = \sup_{w, X \sim X'} \|\frac{T}{n} \sum_{x \in \mathcal{B}_t} \Pi_{C}(\nabla f(w, x)) - \sum_{x' \in \mathcal{B}'_t} \Pi_{C}(\nabla f(w, x'))\| =  \sup_{w, x, x'} \|\frac{T}{n} \Pi_{C}(\nabla f(w, x)) - \Pi_{C}(\nabla f(w, x'))\| \leq \frac{2CT}{n}$. Thus, $\tilt(w_t^{md})$ is $\frac{\varepsilon^2}{2}$-zCDP by Proposition~\ref{prop: gauss}. Second, the iterates $w_t^{ag}$ are deterministic functions of $\tilt(w_t^{md})$, so the post-processing property of differential privacy~\cite{dwork2014, bun16} ensures that \cref{alg: ACSA} is $\frac{\varepsilon^2}{2}$-zCDP. \\
\noindent \textbf{Excess risk:} 
Consider round $t \in [T]$ of \cref{alg: ACSA}, where \cref{alg: MeanOracle2} is run on input data $\{\nabla f(w_t, x_i^t)\}_{i=1}^{n/T}$. Denote the bias of~\cref{alg: MeanOracle2} by $b_t:= \expec \tilt(w_t) - \nabla F(w_t)$, where $\tilt(w_t) = \widetilde{\nu}$ in the notation of \cref{alg: MeanOracle2}. Also let $\hilt(w_t) := \hat{\mu}$ (in the notation of Lemma~\ref{lem: bias and variance of bd14}) and denote the noise by $N_t = \tilt(w_t) - \nabla F(w_t) - b_t = \tilt(w_t) - \expec \tilt(w_t)$. Then we have $B := \sup_{t \in [T]}\|b_t\| \leq \frac{r^{(k)}}{(k-1) C^{k-1}}$ and $\Sigma^2 := \sup_{t \in [T]} \expec[\|N_t\|^2] \leq d\sigma^2 + \frac{r^{(2)} T}{n} \lesssim \frac{d C^2 T^2}{\varepsilon^2 n^2} + \frac{r^{(2)} T}{n}$, by~Lemma~\ref{lem: bias and variance of bd14}. Plugging these estimates for $B$ and $\Sigma^2$ into~Proposition~\ref{prop: ACSA generic} and setting $C = r_k(\frac{\varepsilon n}{\sqrt{dT}})^{1/k}$, we get \begin{align}
\label{eq: thingg}
\expec F(w_T^{ag}) - F^* &\lesssim \frac{\beta D^2}{T^2} + \frac{D(\Sigma + B)}{\sqrt{T}} + BD \nonumber \\
&\lesssim \frac{\beta D^2}{T^2} + \frac{CD \sqrt{dT}}{\varepsilon n} + \frac{r_2 D}{\sqrt{n}} + \frac{r^{(k)} D}{C^{k-1}} \nonumber \\
&\lesssim \frac{\beta D^2}{T^2} + D\left[\frac{r_2}{\sqrt{n}} + r_k\left(\frac{\sqrt{dT}}{\varepsilon n}\right)^{(k-1)/k}\right].
\end{align}
Now, our choice of $T$ implies that $\frac{\beta D^2}{T^2} \leq r_k D\left[\frac{1}{\sqrt{n}} + \left(\frac{\sqrt{dT}}{\varepsilon n}\right)^{(k-1)/k}\right]$ and we get the result upon plugging in $T$. 
\end{proof}

\subsection{Noisy Clipped SGD for Strongly Convex Losses (\cref{sec: strongly convex})}
\label{app: strongly convex}
We begin by presenting the pseudocode for our noisy clipped SGD in~\cref{alg: vanilla SGD}. 
\begin{algorithm}[ht]
\caption{Noisy Clipped SGD for Heavy-Tailed DP SCO}
\label{alg: vanilla SGD}
\begin{algorithmic}[1]
\STATE {\bfseries Input:} 
Data $X \in \XX^n$, $T \leq n$, stepsizes $\{\eta_t\}_{t=0}^{T}$, averaging weights $\{\zeta_t\}_{t=0}^{T}$, $w_0 \in \WW$.
 \FOR{$t \in \{0, 1, \cdots, T\}$} 
 \STATE Draw new batch $\mathcal{B}_t$ (without replacement) of $n/T$ samples from $X$.
 \STATE $\tilt(w_t) := \texttt{MeanOracle1}(\{\nabla f(w_t, x)\}_{x \in \mathcal{B}_t}; \frac{n}{T}; \frac{\varepsilon^2}{2})$ 
 \STATE $w_{t+1} = \Pi_{\WW}\left[w_t - \eta_t \tilt(w_t) \right]
 $
\ENDFOR \\
\STATE {\bfseries Output:} $\widehat{w}_T := \frac{1}{Z_T} \sum_{t=0}^{T} \zeta_t w_{t+1}$, where $Z_T = \sum_{t=0}^T \zeta_t$.
\end{algorithmic}
\end{algorithm}

\cref{alg: generalized vanilla SGD} is a generalized version of~\cref{alg: vanilla SGD} that allows for any \texttt{MeanOracle} and will be useful in our analysis.
\begin{algorithm}[ht]
\caption{
Generic Noisy Clipped SGD Framework for Heavy-Tailed SCO
}
\label{alg: generalized vanilla SGD}
\begin{algorithmic}[1]
\STATE {\bfseries Input:} 
Data $X \in \XX^n$, $T \leq n$, 
$\texttt{MeanOracle}$,
stepsizes $\{\eta_t\}_{t=0}^{T}$, averaging weights $\{\zeta_t\}_{t=0}^{T}$.
\STATE Initialize $w_0 \in \WW$. 
 \FOR{$t \in \{0, 1, \cdots, T\}$} 
 \STATE Draw new batch $\mathcal{B}_t$ (without replacement) of $n/T$ samples from $X$.
 \STATE $\tilt(w_t) := \texttt{MeanOracle}(\{\nabla f(w_t, x)\}_{x \in \mathcal{B}_t}; \frac{n}{T}; \frac{\varepsilon^2}{2})$ 
 \STATE $w_{t+1} = \Pi_{\WW}\left[w_t - \eta_t \tilt(w_t) \right]
 $
\ENDFOR \\
\STATE {\bfseries Output:} $\widehat{w}_T := \frac{1}{Z_T} \sum_{t=0}^{T} \zeta_t w_{t+1}$, where $Z_T = \sum_{t=0}^T \zeta_t$.
\end{algorithmic}
\end{algorithm}
In Proposition~\ref{prop: strongly convex biased sgd}, we provide the convergence guarantees for~\cref{alg: generalized vanilla SGD} in terms of the bias and variance of the \texttt{MeanOracle}. 
\begin{proposition}
\label{prop: strongly convex biased sgd}
Let $F: \WW \to \mathbb{R}$ be $\mu$-strongly convex and $\beta$-smooth with condition number $\kappa := \frac{\beta}{\mu}$. Let $w_{t+1} := \Pi_{\WW}[w_t - \eta_t \tilt(w_t)]$, where $\tilt(w_t) = \nabla F(w_t) + b_t + N_t$, such that the bias and noise (which can depend on $w_t$ and the samples drawn) satisfy $\|b_t\| \leq B$ (with probability $1$), $\expec N_t = 0$, $\expec\|N_t\|^2 \leq \Sigma^2$ for all $t \in [T-1]$, and that $\{N_t\}_{t=1}^T$ are independent. Then, there exist stepsizes $\{\eta_t\}_{t=1}^T$ and weights $\{\zeta_t\}_{t=0}^T$ such that the average iterate $\widehat{w}_T := \frac{1}{\sum_{t=0}^T \zeta_t} \sum_{t=0}^T \zeta_t w_{t+1}$ satisfies
\[
\EPL \leq 32\beta D^2 \exp\left(-\frac{T}{4 \kappa}\right) + \frac{72\Sigma^2}{\mu T} + \frac{2 B^2}{\mu}.
\]
\end{proposition}
\begin{proof}
Define $g(w_t) = -\frac{1}{\eta_t}(w_{t+1} - w_t)$. Then 
\begin{align}
\label{eq:ving}
    \expec\|w_{t+1} - \ws\|^2 &= \expec\|w_t - \eta_t g(w_t) - \ws\|^2 \nonumber \\
    &= \expec\|w_t - \ws\|^2 - 2\eta_t \expec\langle g(w_t), w_t - \ws \rangle + \eta_t^2 \expec\|g(w_t)\|^2. 
\end{align}
Now, conditional on all randomness, we use smoothness and strong convexity to write:
\begin{align*}
F(w_{t+1}) - F(\ws) 
&= F(w_{t+1}) - F(w_t) + F(w_t) - F(\ws) \\
&\leq \langle F(w_t), w_{t+1} - w_t \rangle + \frac{\beta}{2}\|w_{t+1} - w_t\|^2 + \langle \nabla F(w_t), w_t - \ws \rangle - \frac{\mu}{2}\|w_t - \ws\|^2 \\
&= \langle \tilt(w_t), w_{t+1} - \ws \rangle + \langle \nabla F(w_t) - \tilt(w_t), w_{t+1} - \ws \rangle + \frac{\beta \eta_t^2}{2}\|g(w_t)\|^2 \\
&\;\;\;\; - \frac{\mu}{2}\|w_t - \ws\|^2  \\
&\leq \langle g(w_t), w_{t+1} - \ws \rangle + \langle \nabla F(w_t) - \tilt(w_t), w_{t+1} - \ws \rangle + \frac{\beta \eta_t^2}{2}\|g(w_t)\|^2 
\\
&\;\;\;\; - \frac{\mu}{2}\|w_t - \ws\|^2 \\
&= \langle g(w_t), w_{t+1} - w_t \rangle + \langle g(w_t), w_{t} - \ws \rangle - \langle b_t + N_t, w_{t+1} - \ws \rangle + \frac{\beta \eta_t^2}{2}\|g(w_t)\|^2 \\
&\;\;\;\;- \frac{\mu}{2}\|w_t - \ws\|^2 \\
&= \langle g(w_t), w_{t} - \ws \rangle - \langle b_t + N_t, w_{t+1} - \ws \rangle + \left(\frac{\beta \eta_t^2}{2} - \eta_t\right)\|g(w_t)\|^2  - \frac{\mu}{2}\|w_t - \ws\|^2,
\end{align*}
where we used the fact that $\langle \Pi_{\WW}(y) - x, \Pi_{\WW}(y) - y \rangle \leq 0$ for all $x \in \WW, y\in \mathbb{R}^d$ (c.f.~\cite[Lemma 3.1]{bubeck2015convex}) to obtain the last inequality. Thus, \begin{align*}
-2 \eta_t \expec \langle g(w_t), w_{t} - \ws \rangle &\leq -2\eta_t \expec[F(w_{t+1}) - F^*] \\
&\;\;\;\; + 2\eta_t \expec\left[- \langle b_t + N_t, w_{t+1} - \ws \rangle + \left(\frac{\beta \eta_t^2}{2} - \eta_t\right)\|g(w_t)\|^2 - \frac{\mu}{2}\|w_t - \ws\|^2 \right].
\end{align*}
Combining the above inequality with~\cref{eq:ving}, we get \begin{align}
\label{eq:bing}
\expec\|w_{t+1} - \ws\|^2 &\leq (1 - \mu \eta_t)\expec\|w_t - \ws\|^2 - 2\eta_t\expec[F(w_{t+1}) - F^*] - 2\eta_t \expec \langle b_t + N_t, w_{t+1} - \ws \rangle \nonumber \\ 
&\;\;\; + 2\eta_t\left(\frac{\eta_t^2 \beta}{2} - \eta_t\right)\expec\|g(w_t)\|^2. 
\end{align}

Next, consider
\begin{align*} 
|\expec \langle b_t + N_t, w_{t+1} - \ws \rangle| &\leq |\expec \langle b_t + N_t, w_{t+1} - w_t \rangle| + |\expec \langle b_t + N_t, w_t - \ws \rangle|\\
&= |\expec \langle b_t + N_t, w_{t+1} - w_t \rangle| + |\expec \langle b_t, w_t - \ws \rangle| \\
&\leq |\expec \langle b_t + N_t, w_{t+1} - w_t \rangle| + \frac{B^2}{\mu} + \frac{\mu}{4} \expec\|w_t - \ws\|^2
\end{align*}
by independence of $N_t$ (which has zero mean) and $w_t - \ws$, and Young's inequality. Next, note that $v := w_t -\eta_t(\nabla F(w_t) + b_t)$ is independent of $N_t$, so $\expec \langle N_t, \Pi_{\WW}(v)\rangle = 0$. Thus, 
\begin{align*}
|\expec \langle N_t, w_{t+1} - w_t \rangle| &= |\expec \langle N_t, w_{t+1} \rangle | \\
&= |\expec \langle N_t, \Pi_{\WW}\left[w_t - \eta_t\left(\nabla F(w_t) + b_t + N_t\right)\right] \rangle | \\
&= |\expec \langle N_t, \Pi_{\WW}\left[v - \eta_t N_t\right] \rangle |\\
&= |\expec \langle N_t, \Pi_{\WW}\left[v] - \Pi_{\WW}[v - \eta_t N_t\right] \rangle | \\
&\leq \expec\left[\|N_t\| \|\Pi_{\WW}\left[v] - \Pi_{\WW}[v - \eta_t N_t\right]\|\right]\\
&\leq \expec\left[\|N_t\| \|\eta_t N_t\|\right] \\
&\leq \eta_t \Sigma^2,
\end{align*}
by Cauchy-Schwartz and non-expansiveness of projection. 
Further, \begin{align*}
|\expec \langle b_t, w_{t+1} - w_t \rangle| &= |\expec \langle b_t, -\eta_t g(w_t) \rangle|\\
&\leq \frac{B^2}{\mu} + \frac{\eta_t^2 \mu}{4}\expec\|g(w_t)\|^2,
\end{align*}
by Young's inequality. Therefore, \begin{align*}
-2\eta_t \expec \langle b_t + N_t, w_{t+1} - \ws \rangle \leq 2\eta_t\left[\frac{2B^2}{\mu} + \frac{\eta_t^2 \mu}{4}\expec\|g(w_t)\|^2 + \eta_t \Sigma^2 + \frac{\mu}{4}\expec\|w_t - \ws\|^2\right].
\end{align*}
Plugging this bound back into~\cref{eq:bing} and choosing $\eta_t \leq \frac{1}{\beta} \leq \frac{1}{\mu}$ 
yields: \begin{align*}
\expec\|w_{t+1} - \ws\|^2 &\leq \left(1 - \frac{\mu \eta_t}{2}\right)\expec\|w_t - \ws\|^2 - 2\eta_t\expec[F(w_{t+1}) - F^*] + \frac{4\eta_t B^2}{\mu} + 2\eta_t^2 \Sigma^2 \\
&\;\;\; + 2\eta_t\left(\frac{\eta_t^2 \beta}{2} - \eta_t + \frac{\eta_t^2 \mu}{4}\right)\expec\|g(w_t)\|^2 \\
&\leq \left(1 - \frac{\mu \eta_t}{2}\right)\expec\|w_t - \ws\|^2 - 2\eta_t\expec[F(w_{t+1}) - F^*] + \frac{4\eta_t B^2}{\mu} + 2\eta_t^2 \Sigma^2.
\end{align*}
Next, we apply Lemma~\ref{lem: stich stepsize} (see below) with $r_t := \expec \|w_{t} - \ws\|^2, ~s_t := \expec F(w_{t+1}) - F^* - \frac{2B^2}{\mu}$, $a := \frac{\mu}{2}$, $b := 2$, $c = 2 \Sigma^2$, and $g = \beta$. We may assume $s_t \geq 0$ for all $t$: if this inequality breaks for some $t$, then simply return $w_{t+1}$ instead of $\widehat{w}_T$ to obtain $\expec F(w_t) - F^* < \frac{2B^2}{\mu}$. Thus, \[
\frac{1}{\gamma_T}\sum_{t=0}^T \gamma_t \expec[F(w_{t+1}) - F^*] \leq \frac{1}{2}\left[32 \beta D^2 \exp\left(\frac{-\mu T}{4 \beta}\right) + \frac{144 \Sigma^2}{\mu T} + \frac{2 B^2}{\mu}\right]
\]
Finally, Jensen's inequality yields the proposition. 
\end{proof}

\begin{lemma} \cite[Lemma 3]{stich19unified}
\label{lem: stich stepsize}
Let $b > 0$, let $a, c \geq 0,$ and $\{\eta_t\}_{t \geq 0}$ be non-negative step-sizes such that $\eta_t \leq \frac{1}{g}$ for all $t \geq 0$ for some parameter $g \geq a$. Let $\{r_t\}_{t \geq 0}$ and $\{s_t\}_{t \geq 0}$ be two non-negative sequences of real numbers which satisfy \[
r_{t+1} \leq (1 - a \eta_t)r_t - b \eta_t s_t + c \eta_t^2
\]
for all $t \geq 0.$
Then there exist particular choices of step-sizes $\eta_t \leq \frac{1}{g}$ and averaging weights $\zeta_t \geq 0$ such that \[
\frac{b}{\gamma_T}\sum_{t=0}^T s_t \zeta_t + a r_{T+1} \leq 32 g r_0 \exp\left(\frac{-aT}{2g}\right) + \frac{36c}{aT},
\]
where $\gamma_T := \sum_{t=0}^{T} \gamma_t.$
\end{lemma}

We are now prepared to prove~\cref{thm: strongly convex smooth upper bound}.
\begin{theorem}[Precise statement of~\cref{thm: strongly convex smooth upper bound}]
Grant~\cref{ass:boundednoncentral}. Let $\varepsilon > 0$, and assume $F$ is $\mu$-strongly convex and $\beta$-smooth with $\kappa = \frac{\beta}{\mu} \leq n/\ln(n)$.
Then, there are parameters such that~\cref{alg: vanilla SGD} is $\frac{\varepsilon^2}{2}$-zCDP, and \begin{equation}
\EPL \lesssim \frac{1}{\mu}\left(\frac{r_2^2}{n} + r_k^2\left(\frac{\sqrt{d \kappa \ln(n)}}{\varepsilon n}\right)^{\frac{2k-2}{k}} \right). 
\end{equation}
\end{theorem}
\begin{proof}
\textbf{Privacy:} Choose $\sigma^2 = \frac{4C^2 T^2}{\varepsilon^2 n^2}$. Since the batches of data drawn in each iteration are disjoint, it suffices (by parallel composition~\cite{mcsherry2009privacy}) to show that $\tilt(w_t)$ is $\frac{\varepsilon^2}{2}$-zCDP for all $t$. Now, the $\ell_2$ sensitivity of each clipped gradient update is bounded by $\Delta = \sup_{w, X \sim X'} \|\frac{T}{n} \sum_{x \in \mathcal{B}_t} \Pi_{C}(\nabla f(w, x)) - \sum_{x' \in \mathcal{B}'_t} \Pi_{C}(\nabla f(w, x'))\| =  \sup_{w, x, x'} \|\frac{T}{n} \Pi_{C}(\nabla f(w, x)) - \Pi_{C}(\nabla f(w, x'))\| \leq \frac{2CT}{n}$. Hence~Proposition~\ref{prop: gauss} implies that the algorithm is $\frac{\varepsilon^2}{2}$-zCDP.   \\
\textbf{Excess risk:} For any iteration $t \in [T]$, denote the bias of~\cref{alg: MeanOracle2} by $b_t:= \expec \tilt(w_t) - \nabla F(w_t)$, where $\tilt(w_t) = \widetilde{\nu}$ in the notation of \cref{alg: MeanOracle2}. Also let $\hilt(w_t) := \hat{\nu}$ (in the notation of Lemma~\ref{lem: bias and variance of bd14}) and denote the noise by $N_t = \tilt(w_t) - \nabla F(w_t) - b_t = \tilt(w_t) - \expec \tilt(w_t)$. Then we have $B := \sup_{t \in [T]}\|b_t\| \leq \frac{r^{(k)}}{(k-1) C^{k-1}}$ and $\Sigma^2 := \sup_{t \in [T]} \expec[\|N_t\|^2] \leq d\sigma^2 + \frac{r^{(2)} T}{n} \lesssim \frac{d C^2 T^2}{\varepsilon^2 n^2} + \frac{r^{(2)} T}{n}$, by~Lemma~\ref{lem: bias and variance of bd14}. Plugging these bias and variance estimates into~Proposition~\ref{prop: strongly convex biased sgd}, we get \[
\EPL \lesssim \beta D^2 \exp\left(- \frac{T}{4 \kappa}\right) + \frac{1}{\mu T}\left(\frac{d C^2 T^2}{\varepsilon^2 n^2} + \frac{r^{(2)} T}{n}\right) + \frac{(r^{(k)})^2}{C^{2k-2} \mu}. 
\]
Choosing $C = r_k\left(\frac{\varepsilon^2 n^2}{dT}\right)^{1/2k}$ implies \[
\EPL \lesssim \beta D^2 \exp\left(- \frac{T}{4 \kappa}\right) + \frac{1}{\mu}\left(\frac{r^{(2)}}{n} + r_k^2\left(\frac{dT}{\varepsilon^2 n^2}\right)^{(k-1)/k}\right). 
\]
Finally, choosing $T = \left\lceil 4\kappa\ln\left(\frac{\mu \beta D^2}{r_k^2}\left(n + \left(\frac{\varepsilon^2 n^2}{d}\right)^{(k-1)/k}\right)\right) \right\rceil \lesssim \kappa \ln(n)$ yields the result. \\
\end{proof}

\section{Details and Proofs for~\cref{sec: optimal rates}: Algorithm for Non-Smooth (Strongly) Convex Losses}
\label{app: optimal rates}

In order to precisely state (sharper forms of)
\cref{thm: localization convex,thm: localization strongly convex}, we will need to introduce some notation. 
\subsection{Notation}
\label{subsub: notation}
For a batch of data $X \in \XX^m$, we define the $k$-th \textit{empirical moment} of $f(w, \cdot)$ by \[
\widehat{r}_m(X)^{(k)} = \sup_{w \in \WW}  \sup_{\{\nabla f(w, x_i) \in \partial_w f(w, x_i)\}} \frac{1}{m}\sum_{i=1}^m \|\nabla f(w, x_i) \|^k,\]
where the supremum is also over all subgradients $\nabla f(w, x_i) \in \partial_w f(w, x_i)$ in case $f$ is not differentiable. 
For $X \sim \DD^m$, we denote the $k$-th \textit{expected empirical moment} by 
\begin{equation*}
    \widetilde{e}_m^{(k)} := \expec[\widehat{r}_m(X)^{(k)}]
\end{equation*}
\normalsize
and let \[
\widetilde{r}_{k,m} :=  (\widetilde{e}_m^{(k)})^{1/k}.\]
Note that $\widetilde{r}_{k,1} = \wt{r}_k$. 
Our excess risk upper bounds will depend on a weighted average of the expected empirical moments for different batch sizes $m \in \{1, 2, 4, 8, \cdots, n\}$, with more weight being given to $\widetilde{r}_m$ for large $m$ (which are smaller, by Lemma~\ref{lem: empirical moments} below): for $n = 2^l$, define
    \[\widetilde{R}_{k,n} := \sqrt{\sum_{i=1}^{l} 2^{-i} \widetilde{r}_{k, n_i}^2},\]
where $n_i = 2^{-i} n$.  

\begin{lemma}
\label{lem: empirical moments}
Under~\cref{ass:boundednoncentral,ass:tilde}, we have:
$\wt{r}^{(k)} = \widetilde{e}_1^{(k)} \geq \widetilde{e}_2^{(k)} \geq \widetilde{e}_4^{(k)} \geq \widetilde{e}_8^{(k)} \geq \cdots \geq r^{(k)}$. In particular, $\wt{R}_{k,n} \leq \wt{r}_k$.
\end{lemma}
\begin{proof}
Let $l \in \mathbb{N}$, $n = 2^l$ and consider \begin{align*}
\widehat{r}_n(X)^{(k)} &= \frac{1}{n} \sup_{w}\left(\sum_{i=1}^{n/2} \|\nabla f(w,x_i)\|^k + \sum_{i= n/2 + 1}^{n} \|\nabla f(w,x_i)\|^k \right)\\
    &\leq \frac{1}{n} \left( \sup_{w}\sum_{i=1}^{n/2} \|\nabla f(w,x_i)\|^k + \sup_{w}\sum_{i= n/2 + 1}^{n} \|\nabla f(w,x_i)\|^k  \right).
\end{align*} 
Taking expectations over the random draw of $X \sim \DD^n$ yields $\wt{e}_n^{(k)} \leq \wt{e}_{n/2}^{(k)}$. Thus, $\wt{R}_{k,n} \leq \wt{r}_k$ by the definition of $\wt{R}_n$. 
\end{proof}

\subsection{Localized Noisy Clipped Subgradient Method (\cref{sec: localization})}
\label{app: localization}

We begin by proving the technical ingredients that will be used in the proof of~\cref{thm: localization convex}. First, we will prove a variant of Lemma~\ref{lem: bias and variance of bd14} that bounds the bias and variance of the subgradient estimator in~\cref{alg: clipped GD}. 
\begin{lemma}
\label{lem: empirical bias variance reg ERM}
Let $\widehat{F}_\lambda(w) = \frac{1}{n} \sum_{i=1}^n f(w, x_i) + \frac{\lambda}{2}\|w - w_0\|^2$ be a regularized empirical loss on a closed convex domain $\WW$ with $\ell_2$-diameter $D$. 
Let $\widetilde{\nabla} F_\lambda(w_t) = \nabla \widehat{F}_\lambda(w_t) + b_t + N_t = \frac{1}{n} \sum_{i=1}^n \Pi_C(\nabla f(w, x_i)) + \lambda (w - w_0) + N_t$ be the biased, noisy subgradients of the regularized empirical loss in~\cref{alg: clipped GD}, with $N_t \sim \mathcal{N}(0, \sigma^2 \mathbf{I}_d)$ and $b_t = \frac{1}{n}\sum_{i=1}^n \Pi_C(\nabla f(w_t, x_i)) - \frac{1}{n}\sum_{i=1}^n \nabla f(w_t, x_i)$. Assume $\small \widehat{r}_n(X)^{(k)} \geq \sup_{w \in \WW}\left\{\frac{1}{n} \sum_{i=1}^n \|\nabla f(w, x_i)\|^k\right\}$ for all $\nabla f(w,x_i) \in \partial_w f(w, x_i)$. Then, for any $T \geq 1$, we have: \[
\hat{B} := \sup_{t \in [T]}\|b_t\| \leq \frac{\widehat{r}_n(X)^{(k)}}{(k-1) C^{k-1}}
\]
and \[
\hat{\Sigma}^2 := \sup_{t \in [T]} \expec\|N_t\|^2 = d\sigma^2. 
\]
\end{lemma}
\begin{proof}
Fix any $t$. We have \begin{align}
\label{eq:001}
    \|b_t\| &= \left\|\frac{1}{n}\sum_{i=1}^n \Pi_C(\nabla f(w_t, x_i)) - \frac{1}{n}\sum_{i=1}^n \nabla f(w_t, x_i) \right\| \nonumber \\
    &\leq \frac{1}{(k-1) C^{k-1}} \left[\frac{1}{n} \sum_{i=1}^n \|\nabla f(w_t, x_i)\|^k\right],
\end{align}
by Lemma~\ref{lem: bias and variance of bd14} applied with $\DD$ as the empirical distribution on $X$, and $z_i$ in Lemma~\ref{lem: bias and variance of bd14} corresponding to $\nabla f(w_t, x_i)$ in \cref{eq:001}. Taking supremum over $t$ of both sides of \cref{eq:001} and recalling the definition of $\widehat{r}_n(X)^{(k)}$ proves the bias bound. 
The noise variance bound is immediate from the distribution of $N_t$. 
\end{proof}

Using Lemma~\ref{lem: empirical bias variance reg ERM}, we can obtain the following convergence guarantee for~\cref{alg: clipped GD}: 
\begin{lemma}[Re-statement of~Lemma~\ref{lem: subgrad ERM bound}]
Fix $X \in \XX^n$ and let $\widehat{F}_\lambda(w) = \frac{1}{n} \sum_{i=1}^n f(w, x_i) + \frac{\lambda}{2}\|w - w_0\|^2$ for $w_0 \in \WW$, where $\WW$ is a closed convex domain with diameter $D$. Assume $f(\cdot, x)$ is convex and $\small \widehat{r}_n(X)^{(k)} \geq \sup_{w \in \WW}\left\{\frac{1}{n} \sum_{i=1}^n \|\nabla f(w, x_i)\|^k\right\}$ for all $\nabla f(w,x_i) \in \partial_w f(w, x_i)$. \normalsize Denote $\widehat{r}_n(X) = \left[\widehat{r}_n(X)^{(k)}\right]^{1/k}$ and $\hat{w} = \argmin_{w \in \WW} \widehat{F}_\lambda(w)$. Let $\eta \leq \frac{2}{\lambda}$. Then, 
the output of~\cref{alg: clipped GD} satisfies 
\[
\expec\|w_T - \hat{w}\|^2 \leq \exp\left(-\frac{\lambda \eta T}{2}\right)\|w_0 - \hw\|^2 + \frac{8 \eta}{\lambda}\left(\widehat{r}_n(X)^2 + \lambda^2 D^2 + d \sigma^2 \right) + \frac{20}{\lambda^2}\left(\frac{\widehat{r}_n(X)^{(k)}}{(k-1) C^{k-1}}\right)^2,
\]
\normalsize
where $\sigma^2 = \frac{4 C^2 T}{n^2 \varepsilon^2}$. 
\end{lemma}
\begin{proof}
We use the notation of Lemma~\ref{lem: empirical bias variance reg ERM} and write $\till(w_t) = \nabla \widehat{F}_\lambda(w_t) + b_t + N_t = \frac{1}{n} \sum_{i=1}^n \Pi_C(\nabla f(w, x_i)) + \lambda (w - w_0) + N_t$ as the biased, noisy subgradients of the regularized empirical loss in~\cref{alg: clipped GD}, with $N_t \sim \mathcal{N}(0, \sigma^2 \mathbf{I}_d)$ and $b_t = \frac{1}{n}\sum_{i=1}^n \Pi_C(\nabla f(w_t, x_i)) - \frac{1}{n}\sum_{i=1}^n \nabla f(w_t, x_i)$.
Denote $y_{t+1} = w_t - \eta \till(w_t)$, so that $w_{t+1} = \Pi_{\WW}(y_{t+1})$. 
For now, condition on the randomness of the algorithm (noise). By strong convexity, we have \begin{align*}
    \hf_\lambda(w_t) - \hf_\lambda(\hat{w}) &\leq \langle \nabla \hf_\lambda(w_t), w_t - \hat{w}\rangle - \frac{\lambda}{2}\|w_t - \hw\|^2 \\
    &= \langle \till(w_t), w_t - \hw \rangle - \frac{\lambda}{2}\|w_t - \hw\|^2 + \langle \nabla \hf_\lambda(w_t) - \till(w_t), w_t - \hw \rangle \\
    &= \frac{1}{2\eta}\left(\|w_t - \hw\|^2 + \|w_t - y_{t+1}\|^2 - \|y_{t+1} - \hw\|^2 \right)- \frac{\lambda}{2}\|w_t - \hw\|^2 \\
    &\;\;\;+ \langle \nabla \hf_\lambda(w_t) - \till(w_t), w_t - \hw \rangle \\
    &= \frac{1}{2\eta}\left(\|w_t - \hw\|^2(1 - \lambda \eta) - \|y_{t+1} - \hw\|^2 \right) + \frac{\eta}{2}\|\till(w_t)\|^2 \\
    &\;\;\; + \langle \nabla \hf_\lambda(w_t) - \till(w_t), w_t - \hw \rangle \\
    &\leq \frac{1}{2\eta}\left(\|w_t - \hw\|^2(1 - \lambda \eta) - \|w_{t+1} - \hw\|^2 \right) + \frac{\eta}{2}\|\till(w_t)\|^2 - \langle b_t + N_t, w_t - \hw \rangle,
\end{align*} 
where we used non-expansiveness of projection and the definition of $\till(w_t)$ in the last line. Now, re-arranging this inequality and taking expectation, we get \begin{align*}
    \expec [\|w_{t+1} - \hw\|^2] &\leq -2\eta\expec[\hf_\lambda(w_t) - \hf_\lambda(\hw)] + \expec\|w_t - \hw\|^2 (1 - \lambda \eta) + \eta^2 \expec \|\till(w_t)\|^2 \\
    &\;\;\;\;  - 2\eta \expec \langle b_t + N_t, w_t - \hw \rangle \\
    &\leq \expec\|w_t - \hw\|^2 (1 - \lambda \eta) + \eta^2 \expec \|\till(w_t)\|^2 - 2\eta \expec \langle b_t, w_t - \hw \rangle,
\end{align*}
by optimality of $\hw$ and the assumption that the noise $N_t$ is independent of $w_t - \hw$ and zero mean. Also, \begin{align*}
\expec\|\till(w_t)\|^2 &\leq 2\left(\expec\|\nabla \hf_\lambda(w_t)\|^2 + \|b_t\|^2 + \expec\|N_t\|^2 \right)\\
&\leq 2\left(2\widehat{r}_n(X)^2 + 2 \lambda^2 D^2 + \hat{B}^2 + \hat{\Sigma}^2 \right),
\end{align*}
where $
\hat{B} := \sup_{t \in [T]}\|b_t\| \leq \frac{\widehat{r}_n(X)^{(k)}}{(k-1) C^{k-1}}
$
and $
\hat{\Sigma}^2 := \sup_{t \in [T]} \expec\|N_t\|^2 = d\sigma^2. 
$
by Lemma~\ref{lem: empirical bias variance reg ERM}. We also used Young's and Jensen's inequalities and the fact that $\expec N_t = 0$. Further, \[
|\expec \langle b_t, w_t - \hw \rangle | \leq \frac{\hat{B}^2}{\lambda} + \frac{\lambda}{4} \expec\|w_t - \hw\|^2,
\]
by Young's inequality. Combining these pieces yields \begin{equation}
\label{eq:ring}
\expec\|w_{t+1} - \hw\|^2 \leq \left(1 - \frac{\lambda \eta}{2}\right) \expec\|w_t - \hw\|^2 + 4\eta^2\left(\widehat{r}_n(X)^2 + \lambda^2 D^2 + \hat{B}^2 + \hat{\Sigma}^2 \right) + \frac{2\eta \hat{B}^2}{\lambda}. 
\end{equation}
Iterating~\cref{eq:ring} gives us \begin{align*}
\expec\|w_T - \hw\|^2 &\leq \left(1 - \frac{\lambda \eta}{2}\right)^{T} \|w_0 - \hw\|^2 + \left[4\eta^2\left(\widehat{r}_n(X)^2 + \lambda^2 D^2 + \hat{B}^2 + \hat{\Sigma}^2 \right) + \frac{2\eta \hat{B}^2}{\lambda} \right]\sum_{t=0}^{T-1}\left(1 - \frac{\lambda \eta}{2}\right)^t \\
&\leq \left(1 - \frac{\lambda \eta}{2}\right)^{T} \|w_0 - \hw\|^2 + \left[4\eta^2\left(\widehat{r}_n(X)^2 + \lambda^2 D^2 + \hat{B}^2 + \hat{\Sigma}^2 \right) + \frac{2\eta \hat{B}^2}{\lambda} \right]\left(\frac{2}{\lambda \eta}\right) \\
&= \left(1 - \frac{\lambda \eta}{2}\right)^{T} \|w_0 - \hw\|^2 + \frac{8 \eta}{\lambda}\left(\widehat{r}_n(X)^2 + \lambda^2 D^2 + \hat{B}^2 + \hat{\Sigma}^2 \right) + \frac{4 \hat{B}^2}{\lambda^2} \\
&\leq \exp\left(-\frac{\lambda \eta T}{2}\right)\|w_0 - \hw\|^2 + \frac{8 \eta}{\lambda}\left(\widehat{r}_n(X)^2 + \lambda^2 D^2 + \hat{B}^2 + \hat{\Sigma}^2 \right) + \frac{4 \hat{B}^2}{\lambda^2} \\
&\leq \exp\left(-\frac{\lambda \eta T}{2}\right)\|w_0 - \hw\|^2 + \frac{8 \eta}{\lambda}\left(\widehat{r}_n(X)^2 + \lambda^2 D^2 + \hat{\Sigma}^2 \right) + \frac{20 \hat{B}^2}{\lambda^2},
\end{align*}
since $\eta \leq \frac{2}{\lambda}$. Plugging in the bounds on $\hat{B}$ and $\hat{\Sigma}$ from Lemma~\ref{lem: empirical bias variance reg ERM} completes the proof. 
\end{proof}

\begin{proposition}[Precise statement of Proposition~\ref{prop: stability implies generalization}]
Let $f(\cdot, x)$ be convex for all $x$ and grant~\cref{ass:boundednoncentral} for $k=2$. Suppose $\Al: \XX^n \to \WW$ is $\alpha$-on-average model stable. Then for any $\zeta > 0$, we have \[
\expec[F(\Al(X)) - \hf_X(\Al(X))] \leq \frac{r^{(2)}}{2 \zeta} + \frac{\zeta}{2} \alpha^2.
\]
\end{proposition}
\begin{proof}
Let $X, X', X^i$ be constructed as in~Definition~\ref{def: stability}.
We may write $\expec[F(\Al(X)) - \hf_X(\Al(X))] = \expec[\frac{1}{n} \sum_{i=1}^n f(\Al(X^i), x_i) - f(\Al(X), x_i)]$, by symmetry and independence of $x_i$ and $\Al(X^i)$ (c.f.~\cite[Equation B.2]{lei2020fine}). Then by convexity, we have \begin{align*}
\expec[F(\Al(X)) - \hf_X(\Al(X))] &\leq \frac{1}{n}\sum_{i=1}^n \expec[\langle \Al(X^i) - \Al(X), \nabla f(\Al(X^i), x_i) \rangle] \\
&\leq \frac{1}{n}\sum_{i=1}^n \expec\left[\frac{\zeta}{2}\|\Al(X^i) - \Al(X)\|^2 + \frac{1}{2 \zeta}\|\nabla f(\Al(X^i), x_i)\|^2\right].
\end{align*}
Now, since $\Al(X^i)$ is independent of $x_i$, we have: \begin{align*}
\expec\|\nabla f(\Al(X^i), x_i)\|^2 &= \expec[\expec[\|\nabla f(\Al(X^i), x_i)\|^2 | \Al(X^i) 
]] \\
&\leq \sup_{w \in \WW} \expec[\|\nabla f(\Al(X^i), x_i)\|^2 | \Al(X^i) = w]\\
&= \sup_{w \in \WW} \expec[\|\nabla f(w, x_i)\|^2]\\
&\leq r^{(2)}. 
\end{align*}
Combining the above inequalities and recalling~Definition~\ref{def: stability} yields the result. 
\end{proof}

To prove our excess risk bound for regularized ERM (i.e. Proposition~\ref{cor: reg ERM excess risk}), we require the following bound on the generalization error of ERM with strongly convex loss: 
\begin{proposition}
\label{prop: strongly convex ERM generalization error}
Let $f(\cdot, x)$ be $\lambda$-strongly convex, and grant~\cref{ass:boundednoncentral}. Let $\Al(X) := \argmin_{w \in \WW} \hf_X(w)$ be the ERM algorithm. Then,
\[
\expec[F(\Al(X)) - \hf_X(\Al(X))] \leq \frac{2r^{(2)}}{\lambda n}.
\]
\end{proposition}
\begin{proof}
We first bound the stability of ERM and then use Proposition~\ref{prop: stability implies generalization} to get a bound on the generalization error. The beginning of the proof is similar to the proof of~\cite[Proposition D.6]{lei2020fine}: Let $X, X', X^i$ be constructed as in~Definition~\ref{def: stability}. By strong convexity of $\hf_{X^i}$ and optimality of $\Al(X^i)$, we have \[
\frac{\lambda}{2}\|\Al(X) - \Al(X^i)\|^2 \leq \hf_{X^i}(\Al(X)) - \hf_{X^i}(\Al(X^i)),
\]
which implies \begin{align}
\label{eq:wing}
\frac{1}{n}\sum_{i=1}^n \|\Al(X) - \Al(X^i)\|^2 \leq \frac{2}{\lambda n}\sum_{i=1}^n\left[\hf_{X^i}(\Al(X)) - \hf_{X^i}(\Al(X^i))\right]. 
\end{align}
Now, for any $w \in \WW$, \begin{align*}
    n \sum_{i=1}^n \hf_{X^i}(w) &= \sum_{i=1}^n [f(w, x'_i) + \sum_{j \neq i} f(w, x_j)] \\
    &= (n-1)n \hf_X(w) + n \hf_{X'}(w).
\end{align*}
Hence \begin{align*}
    \frac{1}{n}\expec\left[\sum_{i=1}^n \hf_{X^i}(\Al(X))\right] &= \left(\frac{n-1}{n}\right) \expec\hf_X(\Al(X)) + \frac{1}{n}\expec \hf_{X'}(\Al(X)) \\
    &= \left(\frac{n-1}{n}\right) \frac{1}{n} \expec\left[\sum_{i=1}^n \hf_{X^i}(\Al(X^i))\right] + \frac{1}{n}\expec F(\Al(X)),
\end{align*}
by symmetry and independence of $\Al(X)$ and $X'$. Re-arranging the above equality and using symmetry yields \begin{equation}
\label{eq:eing}
\frac{1}{n} \expec\left[\sum_{i=1}^n \hf_{X^i}(\Al(X)) - \hf_{X^i}(\Al(X^i))\right] = \frac{1}{n} \expec\left[F(\Al(X)) - \hf_X(\Al(X))\right].
\end{equation}
Combining~\cref{eq:wing} with~\cref{eq:eing} shows that ERM is $\alpha$-on-average model stable for \begin{equation}
\label{eq:qing}
\alpha^2 = \expec\left[\frac{1}{n} \sum_{i=1}^n \|\Al(X) - \Al(X^i)\|^2\right] \leq \frac{2}{\lambda n}\expec\left[F(\Al(X)) - \hf_X(\Al(X))\right].
\end{equation}
The rest of the proof is where we depart from the analysis of~\cite{lei2020fine} (which required smoothness of $f(\cdot, x)$): Bounding the right-hand side of~\cref{eq:qing} by~Proposition~\ref{prop: stability implies generalization} yields \[
\alpha^2 \leq \frac{2}{\lambda n}\left(\frac{r^{(2)}}{2 \zeta} + \frac{\zeta}{2}\alpha^2 \right)
\]
for any $\zeta > 0$. Choosing $\zeta = \frac{\lambda n}{2}$, we obtain \[
\frac{\alpha^2}{2} \leq \frac{r^{(2)}}{\lambda n \zeta} = \frac{2 r^{(2)}}{\lambda^2 n^2}, 
\]
and $\alpha^2 \leq \frac{4 r^{(2)}}{\lambda^2 n^2}$. Applying~Proposition~\ref{prop: stability implies generalization} again yields (for any $\zeta' > 0$)
\begin{align*}
\expec[F(\Al(X)) - \hf_X(\Al(X))] &\leq \frac{r^{(2)}}{2 \zeta'} + \frac{\zeta'}{2}\left(\frac{4r^{(2)}}{\lambda^2 n^2} \right) \\
&\leq \frac{2 r^{(2)}}{\lambda n},
\end{align*}
by the choice $\zeta' = \frac{\lambda n}{2}$.
\end{proof}

\begin{proposition}[Precise statement of Proposition~\ref{cor: reg ERM excess risk}]
Let $f(\cdot, x)$ be convex, $w_{i-1}, y \in \WW$, and $\hw_i := \argmin_{w \in \WW} \hf_i(w)$, where $\hf_i(w) := \frac{1}{n_i} \sum_{j \in \mathcal{B}_i} f(w, x_j) + \frac{\lambda_i}{2}\|w - w_{i-1}\|^2$ (c.f. line 6 of~\cref{alg: localization}). Then, 
\[
\expec[F(\hw_i)] - F(y) \leq \frac{2r^{(2)}}{\lambda_i n_i} + \frac{\lambda_i}{2}\|y - w_{i-1}\|^2,
\]
\normalsize
where the expectation is over both the random draws of $X$ from $\DD$ and $\mathcal{B}_i$ from $X$.  
\end{proposition}
\begin{proof}
Denote the regularized population loss by $G_i(w) := \expec[\hf_i(w)] = F(w) + \frac{\lambda_i}{2}\|w - w_{i-1}\|^2$.  By Proposition~\ref{prop: strongly convex ERM generalization error}, we have \[
\expec[G_i(\hw_i) - \hf_i(\hw_i)] \leq \frac{2r^{(2)}}{\lambda_i n_i}.
\] 
Thus, \begin{align}
\label{eq: aing}
    \frac{\lambda_i}{2}\expec\|\hw_i - w_{i-1}\|^2 + \expec F(\hw_i) &= \expec G_i(\hw_i) \nonumber \\
    &\leq \frac{2r^{(2)}}{\lambda_i n_i} + \expec[\hf_i(\hw_i)] \nonumber \\
     &\leq \frac{2r^{(2)}}{\lambda_i n_i} + \frac{\lambda_i}{2}\|y - w_{i-1}\|^2 + F(y),
\end{align}
since $\expec[\hf_i(\hw_i)] = \expec[\min_{w \in \WW} \hf_i(w)]  \leq \min_{w \in \WW} \expec[\hf_i(w)] = \min_{w \in \WW} G_i(w) \leq \frac{\lambda_i}{2}\|y - w_{i-1}\|^2 + F(y)$. Subtracting $F(y)$ from both sides of~\cref{eq: aing} completes the proof. 
\end{proof}

We are ready to state and prove the precise form of~\cref{thm: localization convex}, using the notation of~\cref{subsub: notation}: 
\begin{theorem}[Precise statement of~\cref{thm: localization convex}]
Grant~\cref{ass:tilde}. 
Let $f(\cdot, x)$ be convex
and let $\varepsilon \leq \sqrt{d}$. 
Then, there are algorithmic parameters such that~\cref{alg: localization} is $\frac{\varepsilon^2}{2}$-zCDP, and 
has excess risk 
\begin{equation*}
    \expec F(w_l) - F^* \lesssim \wt{R}_{2k,n} D\left(\frac{1}{\sqrt{n}} + \left(\frac{\sqrt{d \ln(n)}}{\varepsilon n}\right)^{\frac{k-1}{k}}\right).
\end{equation*}
Moreover, this excess risk is attained in %
$\wt{\mathcal{O}}(n^{2 + 1/k})$
subgradient evaluations. 
\end{theorem}
\begin{proof}
We choose $\sigma_i^2 = \frac{4 C_i^2 T_i}{n_i^2 \varepsilon^2}$ for $C_i$ and $T_i$ to be determined exactly later. Note that for $\lambda_i$ and $\eta_i$ defined in~\cref{alg: localization}, we have $\eta_i \leq \frac{2}{\lambda_i}$ for all $i \in [l]$. \\

\noindent \textbf{Privacy:} Since the batches $\{\mathcal{B}_i\}_{i=1}^l$ are disjoint, it suffices (by parallel composition~\cite{mcsherry2009privacy}) to show that $w_i$ (produced by $T_i$ iterations of~\cref{alg: clipped GD} in line 7 of~\cref{alg: localization}) is $\frac{\varepsilon^2}{2}$-zCDP for all $i \in [l]$. With clip threshold $C_i$ and batch size $n_i$, the $\ell_2$ sensitivity of the clipped subgradient update is bounded by $\Delta = \sup_{w, X \sim X'} \frac{1}{n_i} \|\sum_{j=1}^{n_i} \Pi_{C_i}(\nabla f(w, x_j)) - \Pi_{C_i}(\nabla f(w, x'_j))\| = \frac{1}{n_i}  \sup_{w, x, x'} \|\Pi_{C_i}(\nabla f(w, x)) - \Pi_{C_i}(\nabla f(w, x'))\| \leq \frac{2C_i}{n_i}$. (Note that the terms arising from regularization cancel out.) Thus, by~Proposition~\ref{prop: gauss}, conditional on the previous updates $w_{1: i}$, the $(i+1)$-st update in line 5 of~\cref{alg: clipped GD} satisfies $\frac{\varepsilon^2}{2T_i}$-zCDP. Hence,~Lemma~\ref{lem: composition} implies that $w_i$ (in line 7 of~\cref{alg: localization}) is $\frac{\varepsilon^2}{2}$-zCDP. \\

\noindent \textbf{Excess risk:}
Recall Chebyshev's inequality: 
\[
\mathbb{P}\left(Y \geq t \expec[Y^k]^{1/k}\right) \leq \frac{1}{t^k}
\]
for any non-negative random variable $Y$, $t \geq 0$, and $k \geq 2$. This implies \[
\mathbb{P}\left(f(\cdot, x)~\text{is not $s$-Lipschitz on $\WW$}\right) = \mathbb{P}\left(\sup_{w \in \WW} \|\nabla f(w,x)\| > s\right) \leq \frac{\wt{r}_k^k}{s
^k}
\]
for any $s \geq 0$. Thus, for $i \in [l]$, the following event holds with probability at least $1 - \frac{1}{t_i^k}$  (over the random draw of $\mathcal{B}_i$): 
\begin{align*}
\widehat{F}_i(\hat{w}_i) = \frac{1}{n_i}\sum_{j \in \mathcal{B}_i} f(\hat{w}_i, x_j) + \frac{\lambda_i}{2}\|\hat{w}_i - w_{i-1}\|^2 &\leq \widehat{F}_i(w_{i-1}) = \frac{1}{n_i}\sum_{j \in \mathcal{B}_i} f(w_{i-1}, x_j) \\ 
\implies \frac{\lambda_i}{2}\|\hat{w}_i - w_{i-1}\|^2 &\leq t_i \wt{r}_k \|\hat{w}_i - w_{i-1}\| \\
\implies \|\hat{w}_i - w_{i-1}\| &\leq \frac{2 t_i \wt{r}_k}{\lambda_i},
\end{align*}
by definition of $\hat{w}_i$ and (with high probability) $s := t_i \wt{r}_k$-Lipschitz continuity of $\widehat{F}_i(\cdot)$. 
Thus, for $i \in [l]$, if we choose $D_i \geq \frac{2 \wt{r}_k t_i}{\lambda_i}$, then $\hw_i \in \WW_i$ with probability at least $1 - \frac{1}{t_i^k}$. By a union bound, it suffices to choose $t_i = \left(2^i \sqrt{n}\right)^{1/k}$ to ensure \begin{align*}
 \mathbb{P}\left(\exists i \in [l], \hw_i \notin \WW_i\right)\leq \sum_{i=1}^{\log_2(n)} \frac{1}{t_i^k} \leq \frac{1}{\sqrt{n}}. 
\end{align*}
Therefore, $D_i = \frac{4 \wt{r}_k 2^{i/k} \sqrt{n}^{1/k}}{\lambda_i}$ ensures that $\hat{w}_i \in \WW_i$ for all $i \in [l]$ with probability at least $1 - \frac{1}{\sqrt{n}}$. Conditional on the event $\hat{w}_i \in \WW_i$ for all $i \in [l]$, we have by Lemma~\ref{lem: subgrad ERM bound}: \begin{equation*}
    \expec\|w_i - \hat{w}_i\|^2 \leq \exp\left(-\frac{\lambda_i \eta_i T_i}{2}\right)\|w_{i-1} - \hw_i\|^2 + \frac{8 \eta_i}{\lambda_i}\left(\widehat{r}_{n_i}(\mathcal{B}_i)^{(2)} + \lambda_i^2 D_i^2 + d \sigma_i^2 \right) + \frac{20}{\lambda_i^2}\left(\frac{\widehat{r}_{n_i}(\mathcal{B}_i)^{(k)}}{(k-1) C_i^{k-1}}\right)^2,
\end{equation*}
conditional on $w_{i-1}$ and the draws of $X \sim \DD^n$ and $\mathcal{B}_i \sim X^{n_i}$. Denote $p := 1 + 1/k$. Taking expectation over the random sampling yields  
\begin{equation*}
    \expec\|w_i - \hat{w}_i\|^2 \leq \exp\left(-\frac{\lambda_i \eta_i T_i}{2}\right)\|w_{i-1} - \hw_i\|^2 + \frac{8 \eta_i}{\lambda_i}\left(\widetilde{e}_{n_i}^{(2)} + \lambda_i^2 D_i^2 + d \sigma_i^2 \right) + \frac{20}{\lambda_i^2} \frac{\wt{e}_{n_i}^{(2k)}}{C_i^{2k-2} (k-1)^2},
\end{equation*}
where $d \sigma_i^2 \leq \frac{4 d C_i^2 T_i}{n_i^2 \varepsilon^2}$. 
Choosing
$T_i = \frac{1}{\lambda_i \eta_i}\ln\left(\frac{D^2 \lambda_i}{d \sigma_i^2 \eta_i} \right) \lesssim n_i^p \ln\left(n \right)$ and $\eta$ to be determined later (polynomial in $n$), we get \begin{align}
\label{eq: sing}
    \expec\|w_i - \hat{w}_i\|^2 &\lesssim\frac{\eta_i}{\lambda_i}\left(\left(\wt{r}_k 2^{i/k}\sqrt{n}^{1/k}\right)^2 + d \sigma_i^2 \right) + \frac{\widetilde{e}_{n_i}^{(2k)}}{\lambda_i^2 C_i^{2k-2}}\nonumber \\
    &\lesssim \eta_i^2 n_i^p \left(\left(\wt{r}_k 2^{i/k}\sqrt{n}^{1/k}\right)^2 + d\sigma_i^2\right) + \frac{\eta_i^2 n_i^{2p} \widetilde{e}_{n_i}^{(2k)}}{C_i^{2k-2}} \nonumber \\
    &\lesssim \left(\frac{\eta^2 n^p}{16^i 2^{ip}}\left(\left(\wt{r}_k 2^{i/k}\sqrt{n}^{1/k}\right)^2 + \frac{d C_i^2 T_i}{\varepsilon^2 n_i^2} +  \frac{n^p\widetilde{e}_{n_i}^{(2k)}}{C_i^{2k-2} 2^{ip}}\right) \right).
\end{align}
Note that under~\cref{ass:tilde}, $F$ is $L$-Lipshitz, where $L = \sup_{w \in \WW} \|\nabla F(w)\| \leq \wt{r}_k$ by Jensen's inequality. 
Now, following the strategy used in the proof of~\cite[Theorem 4.4]{fkt20}, we write \[
\expec F(w_l) - F(\ws) = \expec[F(w_l) - F(\hat{w}_l)] + \sum_{i=1}^l \expec[F(\hat{w}_i) - F(\hat{w}_{i-1})],
\]
where $\hat{w}_0 := \ws$. Using~\cref{eq: sing}, the first term can be bounded as follows: 
\begin{align*}
\expec[F(w_l) - F(\hat{w}_l)] &\leq L\sqrt{\expec\|w_l - \hat{w}_l\|^2} \\
&\lesssim L \sqrt{\eta_l^2\left(\left(\wt{r}_k 2^{l/k}\sqrt{n}^{1/k}\right)^2 + \frac{C_l^2 d}{\varepsilon^2} + \frac{\wt{e}_{n_l}^{(2k)}}{C_l^{2k-2}}\right)}\\
&\lesssim L \left[\frac{\eta}{n^2}\left(\wt{r}_k 2^{l/k}\sqrt{n}^{1/k} + \frac{\sqrt{d} C_l}{\varepsilon} + \frac{\wt{r}_{2k}^k}{C_l^{k-1}} \right) \right] \\
&\lesssim L \left[\frac{\eta}{n^2}\left(\wt{r}_k n^{3/2k} + \wt{r}_{2k} \left(\frac{\sqrt{d}}{\varepsilon}\right)^{(k-1)/k}\right) \right] \\
\end{align*}
if we choose $C_l = \wt{r}_{2k}\left(\frac{\varepsilon}{\sqrt{d}}\right)^{1/k}$. Therefore,  \begin{equation}
\label{eq:ding}
    \expec[F(w_l) - F(\hat{w}_l)] \lesssim \wt{R}_{2k,n} D\left(\frac{1}{\sqrt{n}} + \left(\frac{\sqrt{d \ln(n)}}{\varepsilon n}\right)^{\frac{k-1}{k}}\right),
\end{equation}
if we choose \[
\eta \lesssim \frac{\wt{R}_{2k,n} D n^2}{L}
\min\left(\frac{1}{\wt{r}_k n^{3/2k}}, \frac{1}{\wt{r}_{2k}} \left(\frac{\varepsilon}{\sqrt{d}}\right)^{(k-1)/k} \right)
\left(\frac{1}{\sqrt{n}} + \left(\frac{\sqrt{d \ln(n)}}{\varepsilon n}\right)^{\frac{k-1}{k}}\right) =: \eta_A. \]
Next, Proposition~\ref{cor: reg ERM excess risk} implies \begin{align*}
\expec[F(\hw_i) - F(\hw_{i-1})] &\leq \frac{2 r^2}{\lambda_i n_i} + \frac{\lambda_i}{2}\expec\|\hw_{i-1} - w_{i-1}\|^2 
\end{align*}
for all $i \in [l]$. 
Hence \begin{align*}
\sum_{i=1}^l \expec[F(\hw_i) - F(\hw_{i-1})] &\lesssim \frac{r^2}{\lambda_1 n_1} + \lambda_1 D^2 + \sum_{i=2}^l\left[ \frac{r^2}{\lambda_i n_i} + \lambda_i\eta_i^2\left(n_i^p\left(\left( 2^{i/k} \sqrt{n}^{1/k} \wt{r}_k \right)^2 + d\sigma_i^2\right) + \frac{n_i^{2p} \wt{e}_{n_i}^{(2k)}}{C_i^{2k-2}} \right)\right] \\
&\lesssim r^2 \eta n^{p-1} + \frac{D^2}{\eta n^p} + \sum_{i=2}^l r^2 \eta_i n_i^{p-1} +  \sum_{i=2}^l \frac{\eta_i}{n_i^p}\left(n_i^p\left(\left( 2^{i/k} \sqrt{n}^{1/k} \wt{r}_k \right)^2 + \frac{d C_i^2 T_i}{\varepsilon^2 n_i^2}\right) + \frac{n_i^{2p} \wt{e}_{n_i}^{(2k)}}{C_i^{2k-2}}\right) \\
&\lesssim r^2 \eta n^{p-1} + \frac{D^2}{\eta n^p} +  \sum_{i=2}^l \eta_i\left(\left( 2^{i/k} \sqrt{n}^{1/k} \wt{r}_k \right)^2 + \frac{d C_i^2 n_i^p \ln(n)}{\varepsilon^2 n_i^2} + \frac{n_i^{p} \wt{e}_{n_i}^{(2k)}}{C_i^{2k-2}}\right). 
\end{align*}
Choosing $C_i = \wt{r}_{2k, n_i} \left(\frac{\varepsilon n_i}{\sqrt{d \ln(n)}} \right)^{1/k}$ approximately equalizes the two terms above involving $C_i$ and we get
\begin{align*}
\sum_{i=1}^l \expec[F(\hw_i) - F(\hw_{i-1})] &\lesssim r^2 \eta n^{p-1} + \frac{D^2}{\eta n^p} +  \eta \wt{r}_k^2 n^{1/k} + \eta 
\sum_{i=2}^l 4^{-i} n_i^p \wt{r}_{2k, n_i}^{2} \left(\frac{d \ln(n)}{\varepsilon^2 n_i^2}\right)^{\frac{k-1}{k}} \\
&\lesssim \eta\left[r^2 n^{p-1} +  \wt{r}_k^2 n^{1/k} + \wt{R}_{2k, n}^2 n^p 
\left(\frac{d \ln(n)}{\varepsilon^2 n^2}\right)^{\frac{k-1}{k}}\right] +  \frac{D^2}{\eta n^p}.
\end{align*}
Now, choosing \[
\eta = \min\left(\eta_A, \frac{D}{n^{p/2}}\min\left\{\frac{1}{r n^{(p-1)/2}}, \frac{1}{\wt{r}_k n^{1/2k}}, \frac{1}{\wt{R}_{2k, n} n^{p/2}}\left(\frac{\varepsilon n}{\sqrt{d \ln(n)}} \right)^{(k-1)/k}\right\} \right)
\]
yields
\begin{align*}
\sum_{i=1}^l \expec[F(\hw_i) - F(\hw_{i-1})] &\lesssim \wt{R}_{2k, n} D\left(\frac{1}{\sqrt{n}} + \left(\frac{\sqrt{d \ln(n)}}{\varepsilon n} \right)^{\frac{k-1}{k}} \right) + \frac{\wt{r}_k D n^{1/2k}}{n^{p/2}} + \frac{D^2}{\eta_A n^p} \\
&\lesssim \wt{R}_{2k, n} D\left(\frac{1}{\sqrt{n}} + \left(\frac{\sqrt{d \ln(n)}}{\varepsilon n} \right)^{\frac{k-1}{k}} \right).
\end{align*}
Combining the above pieces, we obtain the desired excess risk bound conditional on the high-probability event that $\hw_i \in \WW_i$ for all $i \in [l]$. Finally, by the law of total expectation, we have (unconditionally) \begin{align*}
\expec[F(w_l) - F(\ws)] &\lesssim \wt{R}_{2k, n} D\left(\frac{1}{\sqrt{n}} + \left(\frac{\sqrt{d \ln(n)}}{\varepsilon n} \right)^{\frac{k-1}{k}} \right) + \frac{LD}{\sqrt{n}} \\
&\lesssim \wt{R}_{2k, n} D\left(\frac{1}{\sqrt{n}} + \left(\frac{\sqrt{d \ln(n)}}{\varepsilon n} \right)^{\frac{k-1}{k}} \right). 
\end{align*}

\noindent \textbf{Subgradient complexity:} Our choice of $T_i = \widetilde{\Theta}\left(\frac{1}{\lambda_i \eta_i}\right) \lesssim n_i^p \ln\left(n \right)$ implies that \cref{alg: localization} uses $\sum_{i=1}^l n_i T_i \lesssim \ln(n) n^{p+1} = \ln(n) n^{2 + 1/k}$ subgradient evaluations. 
\end{proof}

\begin{remark}[Details of Remark~\ref{rem: computation}]
\label{rem: choice of T_i}
If one desires $(\varepsilon, \delta)$-DP or $(\varepsilon, \delta)$-SDP instead of zCDP, then the gradient complexity of~\cref{alg: localization} can be improved to $\mathcal{O}(n^{p + \frac{1}{2}} \ln(n)) = \mathcal{O}(n^{\frac{3}{2} + 1/k} \ln(n)) $ by using Clipped Noisy \textit{Stochastic} Subgradient Method instead of~\cref{alg: clipped GD} as the subroutine in line 7 of~\cref{alg: localization}. Choosing batch sizes $m_i \approx \sqrt{n_i} < n_i$ in this subroutine (and increasing $\sigma_i^2$ by a factor of $\mathcal{O}(\log(1/\delta))$) ensures $(\varepsilon, \delta)$-DP by~\cite[Theorem 1]{abadi16} via privacy amplification by subsampling. The same excess risk bounds hold for any minibatch size $m_i \in [n_i]$, as the proof of~\cref{thm: localization convex} shows. 
\end{remark} 

\subsection{The Strongly Convex Case (\cref{sec: localization strongly})}
\label{app: localization strongly}
Our algorithm is an instantiation of the meta-algorithm described in~\cite{fkt20}: Initialize $w_0 \in \WW$. For $j \in [M] := \lceil \log_2(\log_2(n)) \rceil$, let $N_j = 2^{j-2} n/\log_2(n)$, $\mathcal{C}_j = \left\{\sum_{h < j} N_h + 1, \ldots, \sum_{h \leq j} N_h\right\}$, and let $w_j$ be the output of~\cref{alg: localization} run with input data $X_j = (x_s)_{s \in \mathcal{C}_j}$ initialized at $w_{j-1}$. Output $w_M$. Assume without loss of generality that $N_j = 2^p$ for some $p \in \mathbb{N}$. Then, with the notation of~\cref{subsub: notation}, we have the following guarantees:  
\begin{theorem}[Precise Statement of~\cref{thm: localization strongly convex}]
Grant~\cref{ass:tilde}. 
Let $\varepsilon \leq \sqrt{d}$ and $f(\cdot, x)$ be $\mu$-strongly convex.
Then, there is a polynomial-time $\frac{\varepsilon^2}{2}$-zCDP algorithm $\Al$ based on~\cref{alg: localization} with excess risk 
\begin{equation*}
    \expec F(\Al(X)) - F^* \lesssim  \frac{ \wt{R}_{2k, n/4}^2}{\mu}\left(\frac{1}{n} + \left(\frac{\sqrt{d \ln(n)}}{\varepsilon n}\right)^{\frac{2k-2}{k}}\right).  
\end{equation*}
\normalsize
\end{theorem}
\begin{proof}
\textbf{Privacy}: Since the batches $X_j$ used in each phase of the algorithm are disjoint and~\cref{alg: localization} is $\frac{\varepsilon^2}{2}$-zCDP, privacy of the algorithm follows from parallel composition of DP~\cite{mcsherry2009privacy}.  \\
\noindent \textbf{Excess risk:}
Note that $N_j$ samples are used in phase $j$ of the algorithm. For $j \geq 0$, let $D_j^2 = \expec[\|w_j - \ws\|^2]$ and $\Delta_j = \expec[F(w_j) - F^*]$. By strong convexity, we have $D_j^2 \leq \frac{2 \Delta_j}{\mu}$. Also, \begin{align}
\label{eq:002}
\Delta_{j+1} &\leq a \wt{R}_{2k, N_j} D_j \left(\frac{1}{\sqrt{N_j}} + \left(\frac{\sqrt{d \ln(N_j)}}{\varepsilon N_j} \right)^{\frac{k-1}{k}} \right) \nonumber \\
&\leq a \wt{R}_{2k, N_j} \sqrt{\frac{2 \Delta_j}{\mu}} \left(\frac{1}{\sqrt{N_j}} + \left(\frac{\sqrt{d \ln(N_j)}}{\varepsilon N_j} \right)^{\frac{k-1}{k}} \right) 
\end{align}
for an absolute constant $a \geq 1$, by~\cref{thm: localization convex}. Denote $E_j = \left[a \wt{R}_{2k, N_j} \sqrt{\frac{2}{\mu}} \left(\frac{1}{\sqrt{N_j}} + \left(\frac{\sqrt{d \ln(N_j)}}{\varepsilon N_j} \right)^{\frac{k-1}{k}} \right)\right]^2$. Then since $N_j= 2N_{j+1}$, we have \begin{align}
\label{eq: zing}
    \frac{E_j}{E_{j+1}} &\leq 4\left(\frac{\wt{R}_{2k, N_j}}{\wt{R}_{2k, N_{j+1}}}\right)^2 \nonumber \\
    &\leq 8,
\end{align}
where the second inequality holds because for any $m = 2^q$, we have: \[
\wt{R}_{2k, m/2}^2 = \sum_{i=1}^{\log_2(m) - 1} 2^{-i} \wt{r}^2_{2k, 2^{-(i+1)} m} = \sum_{i=2}^{\log_2(m)} 2^{-(i-1)} \wt{r}^2_{2k, 2^{-i} m} = 2\sum_{i=2}^{\log_2(m)} 2^{-i} \wt{r}^2_{2k, 2^{-i}m} \leq 2 \wt{R}^2_{2k, m}.
\]
Now, \cref{eq: zing} implies that~\cref{eq:002} can be re-arranged as 
\begin{equation}
\label{eq: cing}
    \frac{\Delta_{j+1}}{64 E_{j+1}} \leq \sqrt{\frac{\Delta_j}{64 E_j}} \leq \left(\frac{\Delta_0}{64 E_0}\right)^{1/2^{j+1}}.
\end{equation}
Further, if $M \geq \log \log\left(\frac{\Delta_0}{E_0}\right)$, then \[
\frac{\Delta_{M}}{64 E_{M}} \leq \left(\frac{\Delta_0}{64 E_0}\right)^{1/2^{M}} \leq \left(\frac{\Delta_0}{64 E_0}\right)^{1/\log(\Delta_0/E_0)} \leq 2^A \left(\frac{1}{64}\right)^{1/\log(\Delta_0/E_0)} \leq 2^A,
\]
for an absolute constant $A > 0$, 
since $\Delta_0 \leq \frac{2L^2}{\mu}$ and $E_0 \geq \frac{2L^2}{\mu n}$ implies $\Delta_0/E_0 = \frac{n}{a^2} \leq n$ and $\frac{1}{\log(\Delta_0/E_0)} = \frac{1}{\log(n) - 2 \log(a)} \leq \frac{A}{\log(n)}$ for some $A > 0$, so that $\left(\frac{\Delta_0}{E_0}\right)^{1/\log(\Delta_0/E_0)} \leq n^{A/\log(n)}\leq 2^A$. Therefore, \[
\Delta_M \leq 2^A 64 E_M = \mathcal{O}\left(\frac{\wt{R}_{2k, n/4}^2}{\mu}\left(\frac{1}{n} + \left(\frac{\sqrt{d \ln(n)}}{\varepsilon n}\right)^{\frac{2k-2}{k}}\right)\right),
\]
since $N_M = n/4$. 
\end{proof}

\subsection{Asymptotic Upper Bounds Under~\cref{ass:boundednoncentral,ass:coordinatewise}
}
\label{app: asymptotic}
We first recall the notion of \textit{subexponential} distribution:
\begin{definition}[Subexponential Distribution]
A random variable $Y$ is subexponential if there is an absolute constant $s > 0$ such that  $\mathbb{P}(|Y| \geq t) \leq 2 \exp\left(-\frac{t}{s}\right)$ for all $t \geq 0$. For subexponential $Y$, we define $\|Y\|_{\psi_1} := \inf\left\{s > 0: \mathbb{P}(|Y| \geq t) \leq 2 \exp\left(-\frac{t}{s}\right) ~\forall~t \geq 0\right\}$.
\end{definition}
\noindent Essentially all (heavy-tailed) distributions that arise in practice are subexponential~\cite{mckay2019probability}.

\vspace{.15cm}
Now, we establish asymptotic upper bounds for a broad subclass of the problem class considered in~\cite{wx20, klz21}: namely, \textit{subexponential} stochastic subgradient distributions satisfying~\cref{ass:boundednoncentral} or~\cref{ass:coordinatewise}. In~\cref{thm: asymptotic optimality} below (which uses the notation of~\cref{subsub: notation}), we give upper bounds under~\cref{ass:boundednoncentral}:
\begin{theorem}
\label{thm: asymptotic optimality}
Let $f(\cdot, x)$ be convex. 
Assume $\widetilde{r}_{2k} < \infty$ and $Y_i = \|\nabla f(w, x_i)\|^{2k}$ is subexponential with $E_n \geq \max_{i \in [n]} \left(\|Y_i\|_{\psi_1}\right)$~$\forall w \in \WW$, $\nabla f(w, x_i) \in \partial_w f(w, x_i)$. Assume that for sufficiently large $n$, we have 
$\sup_{w,x}\|\nabla f(w,x)\|^{2k} \leq n^q r^{(2k)}$ for some $q \geq 1$ and  $\max\left(\frac{E_n}{r^{(2k)}}, \frac{E_n^2}{(r^{(2k)})^2}\right) \ln\left(\frac{3n D \beta}{4 r_{2k}}\right) \leq \frac{n}{d q}$, where $\|\nabla f(w,x) - \nabla f(w', x)\| \leq \beta \|w - w'\|$ for all $w, w' \in \WW, x\in \XX$, and subgradients $\nabla f(w,x) \in \partial_w f(w,x)$. Then, 
$\lim_{n \to \infty} \wt{R}_{2k,n} \leq 4 r_{2k}.$
Further, there exists $N \in \mathbb{N}$ such that for all $n \geq N$, 
the output of~\cref{alg: localization} satisfies 
\[
\expec F(w_l) - F^* = \mathcal{O}\left(r_{2k} D\left(\frac{1}{\sqrt{n}} + \left(\frac{\sqrt{d \ln(n)}}{\varepsilon n} \right)^{\frac{k-1}{k}}\right) \right).
\]
\normalsize
If $f(\cdot, x)$ is $\mu$-strongly convex, then the output of algorithm $\Al$ (in~\cref{sec: localization strongly}) satisfies
\[
\expec F(\Al(X)) - F^* = \mathcal{O}\left(\frac{r_{2k}^2}{\mu}\left(\frac{1}{n} + \left(\frac{\sqrt{d \ln(n)}}{\varepsilon n}\right)^{\frac{2k-2}{k}}\right)\right). \]
\normalsize
\end{theorem}

\vspace{.1cm}
While a bound on $\sup_{w,x}\|\nabla f(w,x)\|$ is needed in~\cref{thm: asymptotic optimality}, it can grow as fast as any polynomial in $n$ and only needs to hold for sufficiently large $n$. As $n \to \infty$, this assumption is usually satisfied. Likewise,~\cref{thm: asymptotic optimality} depends only logarithmically on the Lipschitz parameter of the subgradients $\beta$, so the result still holds up to constant factors if, say, $\beta \leq n^p (r/D)$ as $n \to \infty$ for some $p\geq 1$. Crucially, our excess risk bounds do not depend on $L_f$ or $\beta$. 

\vspace{.1cm}
Asymptotic upper bounds for~\cref{ass:coordinatewise} are an immediate consequence of Lemma~\ref{lem: comparing assumptions} combined with~\cref{thm: asymptotic optimality}. Namely, under~\cref{ass:coordinatewise}, the upper bounds in \cref{thm: asymptotic optimality} hold with $r_k$ replaced by $\sqrt{d}\gamma_k^{1/k}$ (by Lemma~\ref{lem: comparing assumptions}). 

\begin{proof}[Proof of \cref{thm: asymptotic optimality}]
\noindent \textbf{Step One:} \textit{There exists $N \in \mathbb{N}$ such that $\wt{r}_{2k, n}^2 \leq 16 r_{2k}^2$ for all $n \geq N$.} \\
We will first use a covering argument to show that $\widehat{r}_n(X)^{(2k)}$ is upper bounded by $2^{2k+1} r^{(2k)}$ with high probability. For any $\alpha > 0$, we may choose an $\alpha$-net with $N_\alpha \leq \left(\frac{3 D}{2 \alpha}\right)^d$ balls centered around points in $\WW_\alpha = \{w_1, w_2, \cdots, w_{N_{\alpha}}\} \subset \WW$ such that for any $w \in \WW$ there exists $i \in [N_\alpha]$ with $\|w - w_i\| \leq \alpha$ (see e.g.~\cite{kolmogorov1959varepsilon} for the existence of such $\WW_\alpha$). For $w \in \WW$, let $\wt{w}$ denote the element of $\WW_\alpha$ that is closest to $w$, so that $\|w - \wt{w}\| \leq \alpha$. Now, for any $X \in \XX^n$, we have 
\begin{align*}
\widehat{r}_n(X)^{(2k)} &= \sup_{w} \left\{\frac{1}{n} \sum_{i=1}^n \|\nabla f(w, x_i)) - \nabla f(\wt{w}, x_i) +  \nabla f(\wt{w}, x_i)\|^{2k} \right\}\\
&\leq 2^{2k} \sup_{w} \left\{\frac{1}{n} \sum_{i=1}^n \|\nabla f(w, x_i) - \nabla f(\wt{w}, x_i)\|^{2k} +  \|\nabla f(\wt{w}, x_i)\|^{2k} \right\} \\
&\leq 2^{2k}\left[\beta^{2k} \alpha^{2k} + \frac{1}{n} \max_{j \in [N_\alpha]} \sum_{i=1}^n \|\nabla f(w_j, x_i)\|^{2k} \right],
\end{align*}
where we used Cauchy-Schwartz and Young's inequality for the first inequality, and the assumption of $\beta$-Lipschitz subgradients plus the definition of $\WW_\alpha$ for the second inequality. Further, \begin{align*}
    \mathbb{P}\left(\frac{2^{2k}}{n} \max_{j \in [N_\alpha]} \sum_{i=1}^n \|\nabla f(w_j, x_i)\|^{2k} \geq 2^{2k+1} r^{(2k)} \right) &\leq N_\alpha \max_{j \in [N_\alpha]} \mathbb{P}\left(\sum_{i=1}^n \|\nabla f(w_j, x_i)\|^{2k} \geq 2^{2k+1} r^{(2k)} \right) \\
    &\leq N_\alpha \exp\left(- n \min\left(\frac{r^{(2k)}}{E_n}, \frac{(r^{(2k)})^2}{E_n^2} \right) \right),
\end{align*} 
by a union bound and Bernstein's inequality (see e.g. \cite[Corollary 2.8.3]{vershynin2018high}). Choosing $\alpha = \frac{2r_{2k}}{\beta}$ ensures that $\mathbb{P}(2^{2k} \beta^{2k} \alpha^{2k} > 2^{2k+1} r^{(2k)}) = 0$ and hence (by union bound) \begin{align*}
\mathbb{P}\left(\widehat{r}_n(X)^{(2k)} \geq 2^{2k+1} r^{(2k)} \right) &\leq N_\alpha \exp\left(- n \min\left(\frac{r^{(2k)}}{E_n}, \frac{(r^{(2k)})^2}{E_n^2} \right) \right) \\
&\leq \left(\frac{3D \beta}{4 r_{2k}} \right)^d \exp\left(- n \min\left(\frac{r^{(2k)}}{E_n}, \frac{(r^{(2k)})^2}{E_n^2} \right) \right) \\
&\leq \frac{1}{n^q},
\end{align*}
by the assumption on $n$. Next, we use this concentration inequality to derive a bound on $\wt{e}_n^{(2k)}$:
\begin{align*}
    \wt{e}_n^{(2k)} &= \expec\left[\widehat{r}_n(X)^{(2k)}\right] \leq \expec\left[\widehat{r}_n(X)^{(2k)} | \widehat{r}_n(X)^{(2k)} \geq 2^{2k+1} r^{(2k)}\right] \frac{1}{n^q} + 2^{2k+1} r^{(2k)} \\
    &\leq \frac{\sup_{w, x}\|\nabla f(w,x)\|^{2k}}{n^q} + 2^{2k+1} r^{(2k)} \\
    &\leq (1 + 2^{2k+1}) r^{(2k)},
\end{align*}
for sufficiently large $n$. Thus, $\wt{r}_{2k, n}^2 \leq 16 r_{2k}^2$ for all  sufficiently large $n$. This establishes Step One. \\

\noindent \textbf{Step Two:} \textit{$\lim_{n \to \infty} \wt{R}_{2k, n} \leq 4 r_{2k}$.}\\
\noindent For all $n = 2^l, l, i \in \mathbb{N}$, define $h_n(i) = 2^{-i} \wt{r}_{2k, 2^{-i}n}^{2} \mathbbm{1}_{\{i \in [\log_2(n)] \}}$. Note that $0 \leq h_n(i) \leq g(i) := 2^{-i} \wt{r}_{2k}^2$ for all $n, i$, and that $\sum_{i=1}^{\infty} g(i) = \wt{r}_{2k}^2 < \infty$ (i.e. $g$ is integrable with respect to the counting measure). Furthermore, the limit $\lim_{n \to \infty} h_n(i) = 2^{-i} \lim_{n \to \infty} \wt{r}_{2k, 2^{-i}n}^{2}$ exists since~Lemma~\ref{lem: empirical moments} implies that the sequence $\{\wt{r}_{2k, 2^{-i}n}^{2}\}_{n=1}^{\infty}$ is monotonic and bounded for every $i \in \mathbb{N}$. Thus, by Lebesgue's dominated convergence theorem, we have \begin{align*}
\lim_{n \to \infty} \wt{R}_{2k,n}^{2} &= \lim_{n \to \infty} \sum_{i=1}^{\infty} h_n(i) \\
&= \sum_{i=1}^{\infty} \lim_{n \to \infty} h_n(i) \\
&\leq \sum_{i=1}^{\infty} 2^{-i} \lim_{n \to \infty} \wt{r}_{2k, 2^{-i}n}^{2} \\
&\leq 16 \sum_{i=1}^{\infty} 2^{-i} r_{2k}^{2} \\
& = 16r_{2k}^{2},
\end{align*}
where the last inequality follows from \textbf{Step One}. Therefore, $\lim_{n \to \infty} \wt{R}_{2k,n} \leq 4 r_{2k}$. By~\cref{thm: localization convex} and~\cref{thm: localization strongly convex}, this also implies the last two claims in~\cref{thm: asymptotic optimality}. 
\end{proof}

\section{Details and Proofs for~\cref{sec: PL}: Non-Convex Proximal PL Losses}
\label{app: PL}
Pseudocode for our algorithm for PPL losses is given in~\cref{alg: zCSDP SGD}.
\begin{algorithm}[ht]
\caption{
Noisy Clipped Proximal SGD for Heavy-Tailed DP SO
}
\label{alg: zCSDP SGD}
\begin{algorithmic}[1]
\STATE {\bfseries Input:} 
Data $X \in \XX^n$, $T \leq n$, stepsizes $\{\eta_t\}_{t=0}^{T-1}$.
\STATE Initialize $w_0 \in \WW$. 
 \FOR{$t \in \{0, 1, \cdots, T-1\}$} 
 \STATE Draw new batch $\mathcal{B}_t$ (without replacement) of $n/T$ samples from $X$.
 \STATE $\tilt^0(w_t) := \texttt{MeanOracle1}(\{\nabla f^0(w_t, x)\}_{x \in \mathcal{B}_t}; \frac{n}{T}; \frac{\varepsilon^2}{2})$ 
 \STATE $w_{t+1} = 
 \prox_{\eta_t f^1}\left(w_t - \eta_t \tilt^0(w_t) \right)
 $
\ENDFOR \\
\STATE {\bfseries Output:} $w_T$. 
\end{algorithmic}
\end{algorithm}
\cref{alg: generic prox SGD} is a generalization of~\cref{alg: zCSDP SGD} which allows for arbitrary \texttt{MeanOracle}. This will be useful for our analysis.
\begin{algorithm}[ht]
\caption{
Generic Noisy Proximal SGD Framework for Heavy-Tailed DP SO
}
\label{alg: generic prox SGD}
\begin{algorithmic}[1]
\STATE {\bfseries Input:} 
Data $X \in \XX^n$, $T \leq n$, 
$\texttt{MeanOracle}$ (and truncation/minibatch parameters),
privacy parameter $\rho = \varepsilon^2/2$, stepsizes $\{\eta_t\}_{t=0}^{T-1}$.
\STATE Initialize $w_0 \in \WW$. 
 \FOR{$t \in \{0, 1, \cdots, T-1\}$} 
 \STATE Draw new batch $\mathcal{B}_t$ (without replacement) of $n/T$ samples from $X$.
 \STATE $\tilt^0(w_t) := \texttt{MeanOracle}(\{\nabla f^0(w_t, x)\}_{x \in \mathcal{B}_t}; \frac{n}{T}; \frac{\varepsilon^2}{2})$ 
 \STATE $w_{t+1} = 
 \prox_{\eta_t f^1}\left(w_t - \eta_t \tilt^0(w_t) \right)
 $
\ENDFOR \\
\STATE {\bfseries Output:} $w_T$. 
\end{algorithmic}
\end{algorithm}
Proposition~\ref{lemma:extendsAS21Thm6} provides a convergence guarantee for~\cref{alg: generic prox SGD} in terms of the bias and variance of the \texttt{MeanOracle}. 

\begin{proposition}
\label{lemma:extendsAS21Thm6}
Consider \cref{alg: generic prox SGD} with biased, noisy stochastic gradients: 
$\tilt^0(w_t) = \nabla F^0(w_t) + b_t + N_t$, and stepsize $\eta = \frac{1}{2\beta}$. Assume that the bias and noise (which can depend on $w_t$ and the samples drawn) satisfy $\|b_t\| \leq B$ (with probability $1$), $\expec N_t = 0$, $\expec\|N_t\|^2 \leq \Sigma^2$ for all $t \in [T-1]$, and that $\{N_t\}_{t=1}^T$ are independent. Assume further that $F$ is $\mu$-Proximal-PL, $F^0$ is $\beta$-smooth, and $F(w_0) - F^* \leq \Delta$. Then, \[
\expec F(w_T) - F^* \leq \left(1 - \frac{\mu}{2\beta}\right)^T \Delta + \frac{4(B^2 + \Sigma^2)}{\mu}.
\]
\end{proposition}

\begin{proof}
Our proof extends the ideas in \cite{lowy2022NCFL} to generic \textit{biased} and noisy gradients without using Lipschitz continuity of $f$. 
By $\beta$-smoothness, for any $r \in [T-1]$, we have \begin{align}
    \expec F(w_{r+1}) &= \expec[F^0(w_{r+1}) + f^1(w_r) + f^1(w_{r+1}) - f^1(w_r)] \nonumber \\
    &\leq \expec\left\{F(w_r) + \left[\langle \tilr(w_r), w_{r+1} - w_r \rangle + \frac{\beta}{2}\|w_{r+1} - w_r\|^2 + f^1(w_{r+1}) - f^1(w_r)\right] \right\}\nonumber \\
    &\;\;\; + \expec \langle \nabla F^0(w_r) - \tilr(w_r), w_{r+1} - w_r \rangle \nonumber \\
    &= \expec F(w_r) + \expec \Big[\langle \nabla F^0(w_r), w_{r+1} - w_r \rangle + \frac{\beta}{2}\|w_{r+1} - w_r\|^2 + f^1(w_{r+1}) - f^1(w_r) \nonumber \\ 
     &\;\;\; + \langle b_r + N_r, w_{r+1} - w_r \rangle \Big] - \expec \langle b_r + N_r, w_{r+1} - w_r \rangle \nonumber \\
     &\leq \expec F(w_r) + \expec \Big[\langle \nabla F^0(w_r), w_{r+1} - w_r \rangle + \beta\|w_{r+1} - w_r\|^2 + f^1(w_{r+1}) - f^1(w_r) \nonumber  \\
     &\;\;\; + \langle b_r + N_r, w_{r+1} - w_r \rangle \Big]
     + \frac{B^2 + \Sigma^2}{\beta},
    \label{eq27}
\end{align}
where we used Young's inequality to bound \begin{align}
-\expec \langle b_r + N_r, w_{r+1} - w_r \rangle &\leq \frac{B^2 + \Sigma^2}{\beta} + \frac{\beta}{2}\|w_{r+1} - w_r\|^2.
\end{align}
Next, we will bound the quantity
\begin{align*}
\expec \left[\langle \nabla F^0(w_r), w_{r+1} - w_r \rangle + \beta\|w_{r+1} - w_r\|^2 + f^1(w_{r+1}) - f^1(w_r) + \langle b_r + N_r, w_{r+1} - w_r \rangle \right].
\end{align*}
Denote 
$H^{\text{priv}}_r(y):= \langle \nabla F^0(w_r), y - w_r \rangle + \beta\|y - w_r\|^2 + f^1(y) - f^1(w_r) + \langle b_r + N_r, y - w_r \rangle$ and $H_r(y):=  \langle \nabla F^0(w_r), y - w_r \rangle + \beta\|y - w_r\|^2 + f^1(y) - f^1(w_r)$. \normalsize
Note that $H_r$ and $H^{\text{priv}}_r$ are $2\beta$-strongly convex. Denote the minimizers of these two functions by $y_*$ and $y_*^{\text{priv}}$ respectively. Now, conditional on $w_r$ and $N_r + b_r$, we claim that \begin{equation}
\label{claimyyy}
    H_r(y_*^{\text{priv}}) - H_r(y_*) \leq \frac{\|N_r + b_r\|^2}{2\beta}.
\end{equation}
To prove~\cref{claimyyy}, we will need the following lemma: 
\begin{lemma}{\cite[Lemma B.2]{lowy2021}}
\label{lemmaB2outpert}
Let $H(y), h(y)$ be convex functions on some convex closed set $\mathcal{Y} \subseteq \mathbb{R}^d$ and suppose that $H$ is $2\beta$-strongly convex. Assume further that $h$ is $L_{h}$-Lipschitz. Define $y_{1} = \arg\min_{y \in \mathcal{Y}} H(y)$ and $y_{2} = \arg\min_{y \in \mathcal{Y}} [H(y) + h(y)]$. Then $\|y_{1} - y_{2}\|_2 \leq \frac{L_{h}}{2\beta}.$ 
\end{lemma}
We apply~\cref{lemmaB2outpert} with $H(y):= H_r(y), ~h(y):= \langle N_r + b_r, y\rangle$, $L_h = \|N_r + b_r\|$, $y_1 = y_*$, and $y_2 = y_*^{\text{priv}}$ to get \[
\|y_* - y_*^{\text{priv}}\| \leq \frac{\|N_r + b_r\|}{2\beta}.
\] On the other hand, \[
H_r^{\text{priv}}(y_*^{\text{priv}}) = H_r(y_*^{\text{priv}}) + \langle N_r + b_r, y_*^{\text{priv}} \rangle \leq H_r^{\text{priv}}(y_*) = H_r(y_*) + \langle N_r + b_r, y_* \rangle.
\]
Combining these two inequalities yields \begin{align}
    H_r(y_*^{\text{priv}}) - H_r(y_*) &\leq \langle N_r + b_r,   y_* - y_*^{\text{priv}}  \rangle \nonumber \\
    &\leq \|N_r + b_r\| \|y_* - y_*^{\text{priv}}\| \nonumber \\
    &\leq \frac{\|N_r + b_r\|^2}{2\beta},
\end{align}
as claimed. Also, note that $w_{r+1} = y_*^{\text{priv}}$. Hence \begin{align*}
    &\expec \left[\langle \nabla F^0(w_r), w_{r+1} - w_r \rangle + \beta\|w_{r+1} - w_r\|^2 + f^1(w_{r+1}) - f^1(w_r) + \langle b_r + N_r, w_{r+1} - w_r \rangle \right] \\
    &\;\;\;\; =\expec\left[ \min_{y \in \mathbb{R}^d} H_r^{\text{priv}}(y) \right] 
\end{align*}
satisfies
\begin{align}
\expec\left[ \min_{y \in \mathbb{R}^d} H_r^{\text{priv}}(y) \right]&\leq \expec \left[\min_y \left\{ \langle \nabla F^0(w_r), y - w_r \rangle + \beta \|y - w_r\|^2 + f^1(y) - f^1(w_r)\right\} \right] + \frac{\Sigma^2 + B^2}{\beta} \\
    &\leq -\frac{\mu}{2 \beta}\expec\left[F(w_r) - F^* \right] + \frac{\Sigma^2 + B^2}{\beta},   
\end{align}
where we used the assumptions that $F$ is $\mu$-PPL and $F^0$ is $2\beta$-smooth in the last inequality. 
Plugging the above bounds back into~\cref{eq27}, we obtain \begin{equation}
\expec F(w_{r+1}) \leq \expec F(w_r) - \frac{\mu}{2\beta}[F(w_r) - F^*] + \frac{2(\Sigma^2 + B^2)}{\beta},
\end{equation}
whence \begin{equation}
\label{eq: thingz}
\expec[F(w_{r+1}) - F^*] \leq \expec[F(w_r) - F^*](1 - \frac{\mu}{2\beta}) + \frac{2(\Sigma^2 + B^2)}{\beta}.
\end{equation}
Using~\cref{eq: thingz} recursively and summing the geometric series, we get \begin{equation}
    \expec[F(w_T) - F^*] \leq \Delta \left(1 - \frac{\mu}{2\beta}\right)^T + \frac{4(\Sigma^2 + B^2)}{\mu}. 
\end{equation}
\end{proof}

\begin{theorem}[Precise statement of \cref{thm: PL upper bound}]
Grant~\cref{ass:boundednoncentral}. Let $\varepsilon > 0$ and assume $F(w) = F^0(w) + f^1(w)$ is $\mu$-PPL for $\beta$-smooth $F^0$, with 
$\kappa = \frac{\beta}{\mu} \leq n/\ln(n)$. 
Then, there are parameters such that \cref{alg: zCSDP SGD} is $\frac{\varepsilon^2}{2}$-zCDP, and 
\[
\EPLL 
\lesssim \frac{r_k^2}{\mu}\left(\left(\frac{\sqrt{d}}{\varepsilon n} \kappa \ln(n) \right)^{\frac{2k-2}{k}} + \frac{\kappa \ln(n)}{n}\right).
\]
\end{theorem}
\begin{proof}
We choose $\sigma^2 = \frac{4C^2 T^2}{\varepsilon^2 n^2}$. \\
\noindent \textbf{Privacy:} By parallel composition (since each sample is used only once) and the post-processing property of DP (since the iterates are deterministic functions of the output of \texttt{MeanOracle1}), it suffices to show that $\tilt(w_t)$ is $\frac{\varepsilon^2}{2}$-zCDP for all $t \geq 0$. By our choice of $\sigma^2$ and Proposition~\ref{prop: gauss}, $\tilt(w_t)$ is $\frac{\varepsilon^2}{2}$-zCDP, since it's sensitivity is bounded by $\sup_{X \sim X', w} \frac{T}{n}\left\|\sum_{x \in \mathcal{B}_t} \Pi_C[\nabla f^0(w, x)] -  \sum_{x' \in \mathcal{B'}_t} \Pi_C[\nabla f^0(w, x')]\right\| \leq \frac{T}{n} \sup_{x, x', w} \| \Pi_C[\nabla f^0(w, x)] -  \Pi_C[\nabla f^0(w, x')]\| \leq \frac{2C T}{n}$.
\\
\textbf{Excess risk:} For any iteration $t \in [T]$, denote the bias of \texttt{MeanOracle1} (\cref{alg: MeanOracle2}) by $b_t:= \expec \tilt(w_t) - \nabla F(w_t)$, where $\tilt(w_t) = \widetilde{\nu}$ in the notation of \cref{alg: MeanOracle2}. Also let $\hilt(w_t) := \hat{\nu}$ (in the notation of Lemma~\ref{lem: bias and variance of bd14}) and denote the noise by $N_t = \tilt(w_t) - \nabla F(w_t) - b_t = \tilt(w_t) - \expec \tilt(w_t)$. Then we have $B := \sup_{t \in [T]}\|b_t\| \leq \frac{r^{(k)}}{(k-1) C^{k-1}}$ and $\Sigma^2 := \sup_{t \in [T]} \expec[\|N_t\|^2] \leq d\sigma^2 + \frac{r^{(2)} T}{n} \leq \frac{4d C^2 T^2}{\varepsilon^2 n^2} + \frac{r_2^2 T}{n}$, by~Lemma~\ref{lem: bias and variance of bd14}. Plugging these bounds on $B^2$ and $\Sigma^2$ into Proposition~\ref{lemma:extendsAS21Thm6}, and choosing $T = 2\left\lceil \kappa \ln\left(\frac{\Delta \mu}{B^2 + \Sigma^2} \right)\right\rceil \lesssim \kappa \ln(n)$ where $\Delta \geq F(w_0) - F^*$, 
we have: \[
\EPLL \leq \frac{5(B^2 + \Sigma^2)}{\mu} \leq \frac{5}{\mu}\left(\frac{2r^{(2)} T}{n} + \frac{2 (r^{(k)})^2}{(k-1)^2 C^{2k-2}} + \frac{2d C^2 T^2}{\varepsilon^2 n^2}\right),
\]
for any $C > 0$. 
Choosing $C = r_k\left(\frac{\varepsilon^2 n^2}{d T^2}\right)^{1/2k}$ makes the last two terms in the above display equal, and we get \[
\EPLL \lesssim \frac{1}{\mu}\left(\left(\frac{r_k^2\sqrt{d}}{\varepsilon n} \kappa \ln(n) \right)^{\frac{2k-2}{k}} + \frac{r_2^2 \kappa \ln(n)}{n}\right),
\]
as desired. \\
\end{proof}

\section{Shuffle Differentially Private Algorithms}
\label{app: SDP mean estimators}
In the next two subsections, we present two SDP algorithms for DP heavy-tailed mean estimation. The first is an SDP version of~\cref{alg: MeanOracle2} and the second is an SDP version of the coordinate-wise protocol of~\cite{ksu20, klz21}. Both of our algorithms offer the same utility guarantees as their zCDP counterparts (up to logarithms). In particular, this implies that the upper bounds obtained in the main body of this paper can also be attained via SDP protocols that do not require individuals to trust any third party curator with their sensitive data (assuming the existence of a secure shuffler). 

\subsection{$\ell_2$ Clip Shuffle Private Mean Estimator}
For heavy-tailed SO problems satisfying~\cref{ass:boundednoncentral}, we propose using the SDP mean estimation protocol described in~\cref{alg: sdp clip MeanOracle}. ~\cref{alg: sdp clip MeanOracle} relies on the shuffle private vector summation protocol of~\cite{cheu2021shuffle}, which is given in~\cref{alg: Pvec}. The useful properties of~\cref{alg: Pvec} are contained in~Lemma~\ref{lem: pvec}.
\begin{algorithm}[ht]
\caption{$\ell_2$ Clip $\texttt{ShuffleMeanOracle1}(\{x_i\}_{i=1}^s; C; (\varepsilon, \delta))$}
\label{alg: sdp clip MeanOracle}
\begin{algorithmic}[1]
\STATE {\bfseries Input:} 
$X = \{x_i\}_{i=1}^s$, $x_i = (x_{i,1}, \cdots, x_{i,d}) \in \mathbb{R}^d$, $C>0$, $(\varepsilon, \delta) \in (\mathbb{R}_+ \times (0,1/2))$. 
 \FOR{$i \in [s]$}
 \STATE $z_i := \Pi_{C}(x_i)$.
 \ENDFOR \\
\STATE $\widetilde{\nu} := \mathcal{P}_{\text{vec}}(\{z_i\}_{i=1}^s; s; C; (\varepsilon, \delta))$.
\STATE {\bfseries Output:} $\widetilde{\nu}$. 
\end{algorithmic}
\end{algorithm}

\begin{algorithm}
\caption{$\mathcal{P}_{\text{vec}}$, a shuffle private protocol for vector summation}
\label{alg: Pvec}
\begin{algorithmic}[1]
\STATE {\bfseries Input:} database of $d$-dimensional vectors $\mathbf{X} = (\mathbf{x}_1, \cdots, \mathbf{x}_s$ with maximum norm bounded by $C > 0$; privacy parameters $(\varepsilon, \delta)$. 
\STATE {\bfseries procedure:} Local Randomizer $\mathcal{R}_{\text{vec}}(\mathbf{x}_i)$
\begin{ALC@g}
\FOR{$j \in [d]$} 
\STATE Shift component to enforce non-negativity: $\mathbf{w}_{i,j} \gets \mathbf{x}_{i,j} + C$
\STATE $\mathbf{m}_j \gets \mathcal{R}_{1D}(\mathbf{w}_{i,j})$
\ENDFOR
\STATE Output labeled messages $\{(j, \mathbf{m}_j)\}_{j \in [d]}$
\end{ALC@g}
\STATE {\bfseries end procedure}
\STATE {\bfseries procedure: Analyzer} $\mathcal{A}_{\text{vec}}(\mathbf{y})$ 
\begin{ALC@g}
\FOR{$j \in [d]$}
\STATE Run analyzer on coordinate $j$'s messages $z_j \gets \mathcal{A}_{\text{1D}}(\mathbf{y}_j)$ 
\STATE Re-center: $o_j \gets z_j - L$
\ENDFOR
\STATE Output the vector of estimates $\mathbf{o} = (o_1, \cdots o_d)$
\end{ALC@g}
\STATE {\bfseries end procedure}
\end{algorithmic}
\end{algorithm}

\begin{lemma}{\cite[Theorem 3.2]{cheu2021shuffle}}
\label{lem: pvec}
Let $\varepsilon \leq 15, \delta \in (0, 1/2)$, $d, s \in \mathbb{N}$ and $C > 0$. There are choices of parameters $b, g, p$ for $\mathcal{P}_{1D}$ such that for an input data set $\mathbf{X} = (\bx_1, \cdots, \bx_s)$ of vectors with maximum norm $\|\bx_i\| \leq C$, the following holds: \\
1) ~\cref{alg: Pvec} is $(\varepsilon, \delta)$-SDP. \\
2) $\mathcal{P}_{\text{vec}}(\mathbf{X})$ is an unbiased estimate of $\sum_{i=1}^s \bx_i$ with bounded variance \[
\expec\left[\left\|\mathcal{P}_{\text{vec}}(\mathbf{X}) - \sum_{i=1}^s \bx_i\right\|^2 \right] = \mathcal{O}\left(\frac{d C^2}{\varepsilon^2} \ln^2(d/\delta) \right). 
\]
\end{lemma}

By the post-processing property of DP, we immediately obtain: 

\begin{lemma}[Privacy, Bias, and Variance of \cref{alg: sdp clip MeanOracle}]
\label{lem: bias and variance of sdp bd14}
Let $\{z_i\}_{i=1}^s \sim \DD^s$ have mean $\expec z_i = \nu$ and $\expec\|z_i\|^k \leq r^{(k)}$ for $k \geq 2$. Denote the noiseless average of clipped samples in \cref{alg: sdp clip MeanOracle} by $\widehat{\nu} := \frac{1}{n}\sum_{i=1}^s \Pi_C(z_i)$. Then, there exist algorithmic parameters such that~\cref{alg: sdp clip MeanOracle} is $(\varepsilon, \delta)$-SDP and such that the following bias and variance bounds hold: 
\begin{equation}
\label{eq: sdp bd14 bias}
\|\expec \widetilde{\nu} - \nu \| = \|\expec \widehat{\nu} - \nu \| \leq \expec \|\widehat{\nu} - \nu \| \leq \frac{r^{(k)}}{(k-1)C^{k-1}},
\end{equation}
and \begin{equation}
    \label{eq: sdp bd14 variance}
    \expec\|\widetilde{\nu} - \expec \widetilde{\nu}\|^2 = \expec \| \widetilde{\nu} - \expec \widehat{\nu}\|^2 = \mathcal{O}\left(\frac{d C^2 \ln^2(d/\delta)}{\varepsilon^2 s^2} + \frac{r^2}{s}\right).
\end{equation}
\end{lemma}
\begin{proof}
\textbf{Privacy:} The privacy claim is immediate from~Lemma~\ref{lem: pvec} and the post-processing property of DP~\cite[Proposition 2.1]{dwork2014}.

\noindent \textbf{Bias:} The bias bound follows as in~Lemma~\ref{lem: bias and variance of bd14}, since $\mathcal{P}_{\text{vec}}$ is an unbiased estimator (by~Lemma~\ref{lem: pvec}).

\noindent \textbf{Variance:} We have
\begin{align*}
\expec\|\widetilde{\nu} - \expec\widetilde{\nu}\|^2 &= \expec \| \widetilde{\nu} - \expec \widehat{\nu}\|^2 \\
&= \expec \| \widetilde{\nu} - \widehat{\nu}\|^2 + \expec \| \widehat{\nu} - \expec \widehat{\nu}\|^2\\
&\lesssim \frac{d C^2 \ln^2(d/\delta)}{\varepsilon^2 s^2} + \frac{1}{s}\expec\|\Pi_C(z_1) - \expec\Pi_C(z_1)\|^2 \\
&\leq \frac{d C^2 \ln^2(d/\delta)}{\varepsilon^2 s^2} +  \frac{1}{s}\expec\|z_1 - \expec z_1\|^2\\
&\leq \frac{d C^2 \ln^2(d/\delta)}{\varepsilon^2 s^2} + \frac{r^{(2)}}{s},
\end{align*}
where we used that the samples $\{z_i\}_{i=1}^s$ are i.i.d., the variance bound in~Lemma~\ref{lem: pvec}, and~\cite[Lemma 4]{bd14}, which states that $\expec\|\Pi_C(X) - \expec \Pi_C(X)\|^2 \leq \expec \|X - \expec X\|^2$ for any random vector $X$. 
\end{proof}

\begin{remark}
Comparing Lemma~\ref{lem: bias and variance of sdp bd14} to~Lemma~\ref{lem: bias and variance of bd14}, we see that the bias and variance of the two \texttt{MeanOracle}s are the same up to logarithmic factors. Therefore, replacing~\cref{alg: MeanOracle2} by~\cref{alg: sdp clip MeanOracle} in our stochastic optimization algorithms yields SDP algorithms with excess risk that matches the bounds provided in this paper (via~\cref{alg: MeanOracle2}) up to logarithmic factors. 
\end{remark}

\subsection{Coordinate-wise Shuffle Private Mean Estimation Oracle}
For SO problems satisfying~\cref{ass:coordinatewise}, we propose~\cref{alg: CSDP MeanOracle} as a shuffle private mean estimation oracle. \cref{alg: CSDP MeanOracle} is a shuffle private variation of~\cref{alg: KLZ MeanOracle}, which was employed by~\cite{ksu20, klz21}. 
\begin{algorithm}[ht]
\caption{Coordinate-wise Private $\texttt{MeanOracle2}(\{x_i\}_{i=1}^s; s; \tau; \frac{\varepsilon^2}{2}; m)$ \cite{ksu20, klz21}}
\label{alg: KLZ MeanOracle}
\begin{algorithmic}[1]
\STATE {\bfseries Input:} 
$X = \{x_i\}_{i=1}^s$, $x_i = (x_{i,1}, \cdots, x_{i,d}) \in \mathbb{R}^d$, $\varepsilon > 0, \tau > 0$, $m \in [s]$ such that $m$ divides $s$. 
 \FOR{$j \in [d]$} 
 \STATE Partition $j$-th coordinates of data into $m$ disjoint groups of size $s/m$.
 \FOR{$i \in [m]$}
 \STATE Clip data in $i$-th group: $Z_j^i := \left\{\Pi_{[-\tau, \tau]}(x_{(i-1)\frac{s}{m}+1, j}), \cdots, \Pi_{[-\tau, \tau]}(x_{i\frac{s}{m}, j})\right\}$. 
 \STATE Compute average of $Z_j^i$: $\widehat{\nu}_j^i := \frac{m}{s} \sum_{z \in Z_j^i} z$. 
 \ENDFOR
 \STATE Compute median of group means: $\widehat{\nu}_j := \text{median}(\widehat{\nu}_j^1, \cdots, \widehat{\nu}_j^m)$.
 \STATE Draw $u \sim \mathcal{N}(0, \sigma^2 \mathbf{I}_d)$, with $\sigma^2 = \frac{4\tau^2 m^2 d}{s^2 \varepsilon^2}$. 
\ENDFOR \\
\STATE {\bfseries Output:} $\widetilde{\nu} = (\widehat{\nu}_1, \cdots, \widehat{\nu}_d) + u$. 
\end{algorithmic}
\end{algorithm}

The bias/variance and privacy properties of~\cref{alg: KLZ MeanOracle} are summarized in~Lemma~\ref{lem: priv bias var klz}. 
\begin{lemma}[Privacy, Bias, and Variance of \cref{alg: KLZ MeanOracle}, \cite{ksu20, klz21}]
\label{lem: priv bias var klz}
Let $\{x_i\}_{i=1}^s \sim \DD^s$ have mean $\expec x_i = \nu$, $\|\nu\| \leq L$, and $\expec|x_{i,j} - \nu_j|^k \leq \gamma_k$ for some $k \geq 2$. Denote by $\widehat{\nu}$ the output of the non-private algorithm that is identical to~\cref{alg: KLZ MeanOracle} except without the added Gaussian noise. Then, \cref{alg: KLZ MeanOracle} with $\sigma^2 = \frac{72 \tau^2 m^2 d}{\varepsilon^2 s^2}$
is $\frac{\varepsilon^2}{2}$-zCDP. Further, if $\tau \geq 2L$, then there is $m = \widetilde{\mathcal{O}}(1) \in [s]$ such that: 
\begin{equation}
   \|\expec \widetilde{\nu} - \nu \| = \|\expec \widehat{\nu} - \mu \| \leq \expec \|\widehat{\nu} - \nu \| = \widetilde{\mathcal{O}}\left(\sqrt{d}\left(\frac{\gamma_k^{1/k}}{\sqrt{s}} + \left(\frac{2}{\tau}\right)^{k-1} \gamma_k\right) \right) =: B,
\end{equation}
and \begin{equation}
    \expec\|N\|^2 = \expec\|\widetilde{\nu} - \expec \widetilde{\nu}\|^2 = \expec \| \widetilde{\nu} - \expec \widehat{\nu}\|^2 = \widetilde{\mathcal{O}}\left(B^2 + %
    d\sigma^2
    \right).
\end{equation}
\end{lemma}
\begin{algorithm}[ht]
\caption{Coordinate-wise Shuffle Private $\texttt{ShuffleMeanOracle2}(\{x_i\}_{i=1}^s; s; \tau; \varepsilon, \delta; m)$}
\label{alg: CSDP MeanOracle}
\begin{algorithmic}[1]
\STATE {\bfseries Input:} 
$X = \{x_i\}_{i=1}^s$, $x_i = (x_{i,1}, \cdots, x_{i,d}) \in \mathbb{R}^d$, $\varepsilon > 0, \delta \in (0, 1)$, $m \in [s]$. 
 \FOR{$j \in [d]$} 
 \STATE $\varepsilon_j := \frac{\varepsilon}{4\sqrt{2d \ln(1/\delta_j)}}, \delta_j := \frac{\delta}{2d}$. 
  \STATE Partition $j$-th coordinates of data into $m$ disjoint groups of size $s/m$.
 \FOR{$i \in [m]$}
 \STATE Clip $j$-th coordinate of data in $i$-th group: $Z_j^i := \left\{\Pi_{[-\tau, \tau]}(x_{(i-1)\frac{s}{m}+1, j}), \cdots, \Pi_{[-\tau, \tau]}(x_{i\frac{s}{m}, j})\right\}$. 
 \STATE Shift to enforce non-negativity: $Z_j^i \gets Z_j^i + (\tau, \cdots, \tau)$. 
 \STATE Compute noisy average of $s/m$ scalars in $Z_j^i$: $\widetilde{\nu}_j^i := \frac{m}{s} \mathcal{P}_{1D}(Z_j^i; \frac{s}{m}; 2\tau; \varepsilon_j, \delta_j)$. 
 \STATE Re-center: $\widetilde{\nu}_j^i \gets \widetilde{\nu}_j^i - \tau$. 
 \ENDFOR
 \STATE Compute median of noisy means: $\widetilde{\nu}_j := \text{median}(\widetilde{\nu}_j^1, \cdots, \widetilde{\nu}_j^m)$. 
\ENDFOR \\
\STATE {\bfseries Output:} $\widetilde{\nu} = (\widetilde{\nu}_1, \cdots, \widetilde{\nu}_d)$. 
\end{algorithmic}
\end{algorithm}
The $\mathcal{P}_{1D}$ subroutine used in \cref{alg: CSDP MeanOracle} is an SDP protocol for summing scalars that we borrow from~\cite{cheu2021shuffle}. It is outlined in~\cref{alg: P1D}. \cref{alg: P1D} decomposes into a local randomizer $\mathcal{R}$ that individuals execute on their own data, and an analyzer component $\mathcal{A}$ that the shuffler executes. $\mathcal{S}(\mathbf{y})$ denotes the shuffled vector $\mathbf{y}$: i.e. the vector whose components are random permutations of the components of $\mathbf{y}$. We describe the privacy guarantee, bias, and variance of~\cref{alg: CSDP MeanOracle} in~Proposition~\ref{prop: sdp bias var}.
\begin{algorithm}
\caption{$\mathcal{P}_{\text{1D}}$, a shuffle private protocol for summing scalars~\cite{cheu2021shuffle}}
\label{alg: P1D}
\begin{algorithmic}[1]
\STATE {\bfseries Input:} 
Scalars $Z = (z_1, \cdots z_s) \in [0,\tau]^s$; design parameters $g, b \in \mathbb{N}; p \in (0, \frac{1}{2})$. 
\STATE {\bfseries procedure: Local Randomizer $\mathcal{R}_{1D}(z_i)$}
\begin{ALC@g}
\FOR{$i \in [s]$}
\STATE $\widebar{z}_i \gets \lfloor z_i g/\tau \rfloor$.
\STATE Sample rounding value $\eta_1 \sim \textbf{Ber}(z_i g/\tau - \widebar{z}_i)$.
\STATE Set $\hat{z}_i \gets \widebar{z}_i + \eta_1$.
\STATE Sample privacy noise value $\eta_2 \sim \textbf{Bin}(b,p)$.
\STATE Report $y_i = (y_{i,1}, \cdots, y_{i, g+b}) \in \{0,1\}^{g + b}$ containing $\hat{z}_i + \eta_2$ copies of $1$ and $g + b - (\hat{z}_i + \eta_2)$ copies of $0$.
\ENDFOR
\end{ALC@g}
\STATE {\bfseries end procedure}
\STATE{\bfseries procedure: Shuffler} $\mathcal{S}(\mathbf{y})$
\begin{ALC@g}
\STATE Shuffler receives $\mathbf{y}:= (y_1, \cdots, y_s)$, draws a uniformly random permutation $\pi$ of $[g+b] \times [s]$, and sends $\mathcal{S}(\mathbf{y}):=  (y_{\pi(1, 1)}, \cdots, y_{\pi(s, g+b)})$ to analyzer. 
\end{ALC@g}
\STATE {\bfseries end procedure}
\STATE{\bfseries procedure: Analyzer} $\mathcal{A}_{\text{1D}}(\mathcal{S}(\mathbf{y}))$
\begin{ALC@g}
\STATE {\bfseries Output:} $\frac{\tau}{g}[(\sum_{i=1}^{s}\sum_{l=1}^{b+g} y_{\pi(i,l)}) - pbs]$.
\end{ALC@g}
\end{algorithmic}
\end{algorithm}

\begin{proposition}[Privacy, Bias, and Variance of \cref{alg: CSDP MeanOracle}]
\label{prop: sdp bias var}
Let $\{x_i\}_{i=1}^s \sim \DD^s$ have mean $\expec x_i = \nu$, $\|\nu\| \leq L$, and $\expec\|x_{i,j} - \nu_j\|^k \leq \gamma_k$ for some $k \geq 2$. Let $\varepsilon \leq 8\ln(2d/\delta), ~\delta \in (0, 1/2)$, and choose $\tau \geq 2L$. Then, there exist choices of parameters ($g, b, p, m$) such that \cref{alg: zCSDP SGD} is $(\varepsilon, \delta)$-SDP, and has bias and variance bounded as follows:  
\begin{equation}
    \|\expec \widetilde{\nu} - \nu \| 
    = \widetilde{\mathcal{O}}\left(\sqrt{d}\left(\frac{\gamma_k^{1/k}}{\sqrt{s}} + \left(\frac{2}{\tau}\right)^{k-1} \gamma_k\right) \right)
\end{equation}
and \begin{equation}
    \expec\|N\|^2 = \expec\|\widetilde{\nu} - \expec \widetilde{\nu}\|^2
    =
    \widetilde{\mathcal{O}}\left(\frac{\tau^2 d \ln^2(d/\delta)}{s^2 \varepsilon^2} + d\left(\frac{1}{s} \gamma_k^{2/k} + \gamma_k^2 \left(\frac{2}{\tau}\right)^{2k-2} \right)\right).
\end{equation}
\end{proposition}

\begin{remark}
\label{rem: sdp has same bias var}
Comparing~Proposition~\ref{prop: sdp bias var} with~Lemma~\ref{lem: priv bias var klz}, we see that the bias and variance of the two \texttt{MeanOracle}s~\cref{alg: CSDP MeanOracle} and~\cref{alg: KLZ MeanOracle} are the same up to logarithmic factors. Therefore, using~\cref{alg: CSDP MeanOracle} as \texttt{MeanOracle} for stochastic optimization results in the same excess risk bounds as one would get by using~\cref{alg: KLZ MeanOracle}, up to logarithmic factors.  
\end{remark}

The proof of~Proposition~\ref{prop: sdp bias var} will require Lemma~\ref{lem:P1D}, which is due to~\cite{cheu2021shuffle}. First, we need the following notation:
\begin{definition}[$\delta$-Approximate Max Divergence]
For random variables $X$ and $Y$, define \[
D^{\delta}_{\infty}(X||Y) = \sup_{S \subseteq \text{supp}(X): \mathbb{P}(X \in S) \geq \delta} \ln\left[\frac{\mathbb{P}(X \in S) - \delta}{\mathbb{P}(Y \in S)}\right].
\]
\end{definition}
\noindent An important fact is that a randomized algorithm $\Al$ is $(\varepsilon, \delta)$-DP if and only if $D^{\delta}_{\infty}(\Al(X)||\Al(X')) \leq \varepsilon$ for all adjacent data sets $X \sim X'$~\cite{dwork2014}. 

\begin{lemma}{\cite[Lemma 3.1]{cheu2021shuffle}}
\label{lem:P1D}
Let $s \in \mathbb{N}, \varepsilon \leq 15$.  
Let $g \geq \tau \sqrt{s}$, $b > \frac{180g^2 \ln(2/\delta)}{\varepsilon^2 s}$ and 
$p = \frac{90g^2 \ln(2/\delta)}{b \varepsilon^2 s}$. Then for any adjacent scalar databases $Z, Z' \in [0, \tau]^r$ differing on user $u$ ($z_u \neq z'_u$), we have: \\
a) $D^{\delta}_{\infty}(\mathcal{S} \circ \mathcal{R}_{1D}^r(Z)
||\mathcal{S} \circ \mathcal{R}_{1D}^r(Z')) \leq \varepsilon\left(\frac{2}{g} + \frac{|z_u - z'_u|}{\tau}\right)$.\\
b) Unbiasedness: $\expec[\mathcal{P}_{1D}(Z)] = \sum_{i=1}^s z_i$. \\
c) Variance bound: $\expec[(\mathcal{P}_{1D}(Z) - \sum_{i=1}^s z_i)^2] = \mathcal{O}\left(\frac{\tau^2}{\varepsilon^2}\ln(1/\delta)\right)$. 
\end{lemma}

To prove the utility guarantees in~Proposition~\ref{prop: sdp bias var}, we begin by providing (Lemma~\ref{lem:4.1}) high probability bounds on the bias and noise induced by the \textit{non-private} version of \cref{alg: CSDP MeanOracle}, in which $\mathcal{P}_{1D}$ in line 8 is replaced by the sum of $z \in Z_j^i$. Lemma~\ref{lem:4.1} is a refinement of~\cite[Theorem 4.1]{klz21}, with correct scaling for arbitrary $\gamma_k > 0$ and exact constants:
\begin{lemma}
\label{lem:4.1}
Let $\zeta \in (0,1)$ and $X \sim \mathcal{D}$ be a random $d$-dimensional vector with mean $\nu$ such that $\|\nu\| \leq L$ and $\expec|\langle X - \nu, e_j \rangle|^k \leq \gamma_k$ for some $k \geq 2$ for all $j \in [d]$. Consider the non-private version of \cref{alg: CSDP MeanOracle} run with $m = \left\lceil 20 \log(4d/\zeta) \right\rceil$ and $\tau \geq 2L$, where $\mathcal{P}_{1D}$ in line 8 is replaced by the sum of $z \in Z_j^i$. Denote the output of this algorithm by $\hat{\nu}$. Then with probability at least $1 - \zeta$, we have \[
\|\hat{\nu} - \nu\| \leq 10\sqrt{d}\left(\sqrt{\frac{m}{s}}\gamma_k^{1/k} + \gamma_k\left(\frac{2}{\tau}\right)^{k-1}\right).
\]
\end{lemma}
\begin{proof}
Denote $X = (x_1, \cdots, x_d)$, $\nu = (\nu_1, \cdots, \nu_d)$, 
and $z_j := \Pi_{[-\tau, \tau]}(x_j)$ for $j \in [d]$. By Lemma~\ref{lem:B.1} (stated and proved below), \begin{equation}
\label{eq:one}
    |\expec z_j - \nu_j | \leq 10 \gamma_k \left(\frac{2}{\tau} \right)^{k-1}.
\end{equation}Now an application of \cite[Lemma 3]{minsker2022u} with $\rho(t):= 
\begin{cases}
\frac{t^2}{2} &\mbox{if $t \in [-\tau, \tau]$} \\
\tau t &\mbox{if $t > \tau$}\\
-\tau t &\mbox{if $t < -\tau$}
\end{cases}$ shows that $\expec|z_j - \expec z_j |^k \leq \gamma_k$. Hence by Lemma~\ref{lem:A.2} (stated and proved below), \begin{align}
    \mathbb{P}\left(\Bigg|\hat{\nu}_j^i - \expec z_j\Bigg| \leq 10\sqrt{\frac{m}{s}}\gamma_k^{1/k}\right) &= \mathbb{P}\left(\Bigg|\frac{m}{s}\sum_{z \in Z_j^i} z - \expec z_j \Bigg| \leq 10\sqrt{\frac{m}{s}}\gamma_k^{1/k}\right) \nonumber \\
    &\geq 0.99, ~\forall i \in [m], ~j \in [d]. 
\end{align}
Next, using the ``median trick'' (via Chernoff/Hoeffding bound, see e.g. \cite{mediantrick}), we get \begin{equation}
\label{eq:two}
\mathbb{P}\left(\Bigg|\hat{\nu}_j - \expec z_j\Bigg| \leq 10\sqrt{\frac{m}{s}}\gamma_k^{1/k}\right) \leq 2e^{-m/20} \leq \frac{\zeta}{2d},
\end{equation}
where the last inequality follows from our choice of $m$. Now, by union bound, we have for any $a > 0$ that \begin{align}
    \mathbb{P}(\|\hat{\nu} - \nu\| \geq \sqrt{d} a) &\leq \sum_{j=1}^d \mathbb{P}(|\hat{\nu}_j - \nu_j| \geq a) \\
    &\leq \sum_{j=1}^d [\mathbb{P}(|\hat{\nu}_j - \expec z_j| \geq a/2) + \mathbb{P}(|\expec z_j - \nu_j| \geq a/2)]. 
\end{align}
Plugging in $a = 2\left[10 \gamma_k^{1/k}\sqrt{\frac{m}{s}} + 10 \gamma_k\left(\frac{2}{\tau}\right)^{k-1}\right]$ implies 
\begin{equation}
    \mathbb{P}\left(\|\hat{\nu} - \nu\| \geq \sqrt{d}\left(10 \gamma_k^{1/k}\sqrt{\frac{m}{s}} + 10 \gamma_k\left(\frac{2}{\tau}\right)^{k-1} \right)\right) \leq d\left(\frac{\zeta}{2d} + 0\right) = \frac{\zeta}{2},
\end{equation}
by \cref{eq:one} and \cref{eq:two}.
\end{proof}

\noindent Below we give the lemmas that we used in the proof of Lemma~\ref{lem:4.1}. The following is a refinement of \cite[Lemma B.1]{klz21} with the proper scaling in $\gamma_k$:
\begin{lemma}
\label{lem:B.1}
Let $x \sim \mathcal{D}$ be a random variable with mean $\nu$ and $\expec|x - \nu|^k \leq \gamma_k$ for some $k \geq 2$. Let $z = \Pi_{[-\tau, \tau]}(x)$ for $\tau \geq 2|\nu|$. Then, \[
| \nu - \expec z | \leq 10 \gamma_k\left(\frac{2}{\tau}\right)^{k-1}.
\]
\end{lemma}
\begin{proof}
We begin by recalling the following form of Chebyshev's inequality: \begin{equation}
\label{eq:cheb}
    \mathbb{P}(|x - \nu| \geq c) \leq \frac{\expec|x - \nu|^k}{c^k} \leq \frac{\gamma_k}{c^k}
\end{equation}
for any $c > 0$, via Markov's inequality. By symmetry, we may assume without loss of generality that $\nu \leq 0$. Note that \begin{align}
\nu - \expec z &= \expec\left[x\mathbbm{1}_{x < -\tau} + x\mathbbm{1}_{x > \tau} + x\mathbbm{1}_{x \in [-\tau, \tau]} - (-\tau\mathbbm{1}_{x < -\tau} + \tau\mathbbm{1}_{x > \tau} + x\mathbbm{1}_{x \in [-\tau, \tau]})\right] \\
&= \expec\left[(x + \tau)\mathbbm{1}_{x < -\tau} + (x - \tau)\mathbbm{1}_{x > \tau}\right].
\end{align}
So, \begin{equation}
    |\nu - \expec z| \leq \underbrace{|\expec(x + \tau) \mathbbm{1}_{x < - \tau}|}_{\textcircled{a}} + \underbrace{|\expec(x - \tau) \mathbbm{1}_{x > \tau}|}_{\textcircled{b}},
\end{equation}
by the triangle inequality. Now,\begin{align}
\label{eq:q}
    \textcircled{a} &= |\expec[x - \nu - (-\tau - \nu)\mathbbm{1}_{x < -\tau}| \nonumber \\
    &\leq \expec[|x - \nu|\mathbbm{1}_{x < -\tau}] + |-\tau - \nu|\expec \mathbbm{1}_{x < -\tau} \nonumber \\
    &\leq \left(\expec|x - \nu|^k \right)^{1/k}(\mathbb{P}(x < -\tau))^{(k-1)/k} + |-\tau - \nu|\mathbb{P}(x < -\tau),
\end{align}
by Holder's inequality. Also, since $-\frac{\tau}{2} \leq \nu \leq 0$, we have \begin{equation}
\label{eq:pb}
    \mathbb{P}(x < - \tau) \leq \mathbb{P}\left(x < \nu - \frac{\tau}{2}\right) \leq \mathbb{P}\left(|x - \nu| > \frac{\tau}{2}\right) \leq \frac{\gamma_k 2^k}{\tau^k},
\end{equation}
via \cref{eq:cheb}. Plugging \cref{eq:pb} into \cref{eq:q} and using the bounded moment assumption, we get \begin{align}
   \textcircled{a} &\leq \gamma_k^{1/k} \left(\frac{\gamma_k 2^k}{\tau^k}\right)^{(k-1)/k} + (\tau + |\nu|)\frac{\gamma_k 2^k}{\tau^k}\\
   &\leq \gamma_k\left(\frac{2}{\tau}\right)^{k-1} + 4\gamma_k \left(\frac{2}{\tau}\right)^{k-1} \\
   &= 5\gamma_k \left(\frac{2}{\tau}\right)^{k-1}.
\end{align}
Likewise, a symmetric argument shows that $\textcircled{b} \leq 5\gamma_k \left(\frac{2}{\tau}\right)^{k-1}$. Hence the lemma follows. 
\end{proof}

Below is a is a more precise and general version of \cite[Lemma A.2]{klz21} (scaling with $\gamma_k$):
\begin{lemma}
\label{lem:A.2}
Let $\mathcal{D}$ be a distribution over $\mathbb{R}$ with mean $\nu$ and $\expec|\mathcal{D} - \nu|^k \leq \gamma_k$ for some $k \geq 2$. Let $x_1, \cdots, x_n$ be i.i.d. samples from $\mathcal{D}$. Then, with probability at least $0.99$, \[
\Bigg|\frac{1}{n} \sum_{i=1}^n x_i - \nu\Bigg| \leq \frac{10 \gamma_k^{1/k}}{\sqrt{n}}.
\]
\end{lemma}
\begin{proof}
First, by Jensen's inequality, we have \begin{equation}
\label{eq:c}
    \expec[(x - \nu)^2] \leq \expec\left[|x - \nu|^k\right]^{2/k} \leq \gamma_k^{2/k}. 
\end{equation}
Hence, \begin{align}
    \expec\left[\left(\frac{1}{n}\sum_{i=1}^n x_i - \nu\right)^2\right] &= \frac{1}{n^2} \expec\left[\sum_{i=1}^n (x_i - \nu)^2\right] \\
    &\leq \frac{\gamma_k^{2/k}}{n},
\end{align}
where we used the assumption that $\{x_i\}_{i=1}^n$ are i.i.d. and  \cref{eq:c}. Thus, by Chebyshev's inequality, 
\[
\mathbb{P}\left(\Bigg|\frac{1}{n}\sum_{i=1}^n x_i - \nu \Bigg| \geq \frac{10 \gamma_k^{1/k}}{\sqrt{n}}\right) \leq \frac{1}{100}.
\]
\end{proof}

Now, we provide the proof of~Proposition~\ref{prop: sdp bias var}:
\begin{proof}[Proof of~Proposition~\ref{prop: sdp bias var}]
\textbf{Privacy:} Let $X$ and $X'$ be adjacent data sets in $\XX^s$. Assume without loss of generality that $x_1 \neq x'_1$ and $x_l = x'_l$ for $l > 1$. By the post-processing property of DP and the fact that each sample is only processed once (due to disjoint batches) during the algorithm, it suffices to fix $i \in [m]$ and show that the composition of all $d$ invocations of $\mathcal{P}_{1D}$ (line 8) for $j \in [d]$ is $(\varepsilon, \delta)$-DP. Assume without loss of generality that $Z_j := Z_j^1$ and $Z_j' := (Z_j^1)'$ contain the truncated $z_{1,j} := \Pi_{[0, 2\tau]}(x_{1, j})$ and $z_{1,j}' := \Pi_{[0, 2\tau]}(x'_{1, j})$, respectively. Then, by the first part of~Lemma~\ref{lem:P1D}, there are choices of $b$ and $p$ such that \begin{align*}
D^{\delta_j}_{\infty}(\mathcal{S} \circ \mathcal{R}_{1D}^r(Z_j)
||\mathcal{S} \circ \mathcal{R}_{1D}^r(Z_j')) \leq \varepsilon_j\left(\frac{2}{g} + \frac{|z_{1,j} - z'_{1,j}|}{2\tau}\right) \leq 2 \varepsilon_j,
\end{align*}
for all $j$ by the triangle inequality, provided $g \geq \max\{2, 2\tau \sqrt{s}\}$. Hence each invocation of $\mathcal{P}_{1D}$ is $(2\varepsilon_j, \delta_j)$-DP. Thus, by the advanced composition theorem~\cite[Theorem 3.20]{dwork2014}, \cref{alg: CSDP MeanOracle} is $(\varepsilon', \delta)$-SDP, where \begin{align}
    \varepsilon' &= \sum_{j=1}^d (2\varepsilon_j) \left(e^{2\varepsilon_j} - 1\right) + 2\sqrt{2\sum_{j=1}^d \varepsilon_j^2 \ln(1/\delta_j)} \\
    &\leq 8d \varepsilon_1^2 + \frac{\varepsilon}{2} \\
    &\leq \frac{\varepsilon^2}{4 \ln(1/\delta_1)} + \frac{\varepsilon}{2} \\
    &\leq \varepsilon,
\end{align}
by our choices of $\varepsilon_j, \delta_j$ and the assumption that $\varepsilon \leq 8 \ln(1/\delta_1) = 8 \ln(2d/\delta)$.  \\

\noindent \textbf{Bias:} Let $\hat{\nu} = (\hat{\nu}_1, \cdots, \hat{\nu}_d)$ denote the output of the non-private version of~\cref{alg: CSDP MeanOracle} where $\mathcal{P}_{1D}$ in line 8 is replaced by the (noiseless) sum of $z \in Z_j^i$. Then, ~Lemma~\ref{lem:4.1} tells us that \[
\|\hat{\nu} - \nu\|^2 \leq 200 d\left(\frac{m}{s} \gamma_k^{2/k} + \gamma_k^2 \left(\frac{2}{\tau}\right)^{2k-2} \right)
\]
with probability at least $1 - \zeta$, if $m = \lceil 20 \log(4d/\zeta) \rceil$. Thus, \begin{align*}
    b^2 = \|\expec \widetilde{\nu} - \nu\|^2 &=  \|\expec \hat{\nu} - \nu\|^2 \leq \expec \|\hat{\nu} - \nu\|^2 \\
    &\leq 200 d\left(\frac{m}{s} \gamma_k^{2/k} + \gamma_k^2 \left(\frac{2}{\tau}\right)^{2k-2} \right)(1 - \zeta) + 2\sup\left(\|\hat{\nu}\|^2 + \|\nu\|^2 \right) \zeta \\
    &\leq 200 d\left(\frac{m}{s} \gamma_k^{2/k} + \gamma_k^2 \left(\frac{2}{\tau}\right)^{2k-2} \right)  + 4\tau^2 d\zeta,
\end{align*}
since $\tau \geq L \geq \|\nu\|$ by assumption. Then choosing \[
\zeta = \frac{1}{\tau^2}\left[\left(\frac{1}{s} \gamma_k^{2/k} + \gamma_k^2 \left(\frac{2}{\tau}\right)^{2k-2} \right)\right]
\]
and noting that $\zeta$ is polynomial in all parameters (so that $m = \widetilde{\mathcal{O}}(1)$)
implies \begin{equation}
\label{eq:bias bound}
    b^2 \leq \expec \|\hat{\nu} - \nu\|^2 = \widetilde{\mathcal{O}}\left(d\left(\frac{1}{s} \gamma_k^{2/k} + \gamma_k^2 \left(\frac{2}{\tau}\right)^{2k-2} \right)\right).
\end{equation}

\noindent \textbf{Variance:} We have  
\begin{align}
   \expec \|N\|^2  &\leq 2 \left[\underbrace{\expec\|\widetilde{\nu} - \hat{\nu}\|^2}_{\textcircled{a}} + \underbrace{\expec\|\hat{\nu} - \expec\hat{\nu} \|^2}_{\textcircled{b}} \right].
\end{align}
We bound \textcircled{a} as follows. For any $j \in [d], i \in [m]$, denote  $\hat{\nu}_j^i:= \frac{m}{s} \sum_{z \in Z_j^i} z$ (c.f. line 8 of \cref{alg: CSDP MeanOracle}), the mean of the (noiseless) $s/m$ clipped $j$-th coordinates of data in group $i$. Denote $\widetilde{\nu}^i_j$ as in \cref{alg: CSDP MeanOracle}, which is the same as $\hat{\nu}_j^i$ except that the summation over $Z_j^i$ is replaced by $\mathcal{P}_{1D}(Z_j^i)$. Also, denote $\hat{\nu}_j = \text{median}(\hat{\nu}_j^1, \cdots, \hat{\nu}_j^m)$ and $\widetilde{\nu}_j =  \text{median}(\widetilde{\nu}_j^1, \cdots, \widetilde{\nu}_j^m)$. Then for any $j \in [d]$, we have \begin{align*}
|\hat{\nu}_j - \widetilde{\nu}_j|^2 &= |\text{median}(\hat{\nu}_j^1, \cdots, \hat{\nu}_j^m) - \text{median}(\widetilde{\nu}_j^1, \cdots, \widetilde{\nu}_j^m)|^2 \\
&\leq \left(\sum_{i=1}^m |\hat{\nu}_j^i- \widetilde{\nu}_j^i| \right)^2 \\
&= \left(\sum_{i=1}^m \left|\frac{m}{s} \left[\sum_{z \in Z_j^i} z - \mathcal{P}_{1D}(Z_j^i) \right] \right| \right)^2 \\
&\leq \left(\frac{m}{s} \sum_{i=1}^m \left|\mathcal{P}_{1D}(Z_j^i) - \sum_{z \in Z_j^i} z \right| \right)^2 \\
&\leq \frac{m^2}{s^2} m \sum_{i=1}^m \left|\mathcal{P}_{1D}(Z_j^i) - \sum_{z \in Z_j^i} z \right|^2.
\end{align*}
Now using Lemma~\ref{lem:P1D} (part c), we get \begin{align*}
    \expec[|\hat{\nu}_j - \widetilde{\nu}_j|^2] &\leq \frac{m^4}{s^2} \expec\left[\left|\mathcal{P}_{1D}(Z_j^i) - \sum_{z \in Z_j^i} z \right|^2 \right] \\
    &\lesssim \frac{m^4}{s^2} \frac{\tau^2}{\varepsilon_j^2} \ln(1/\delta_j)  \\
    &\lesssim \frac{m^4 \tau^2 d \ln^2(d/\delta)}{s^2 \varepsilon^2}.
\end{align*}
Thus, summing over $j \in [d]$, we get \begin{align}
\label{eq: a bound}
    \expec[\|\hat{\nu} - \widetilde{\nu}\|^2] \lesssim \frac{m^4 \tau^2 d \ln^2(d/\delta)}{s^2 \varepsilon^2}.
\end{align}

Next, we bound \textcircled{b}: \begin{align*}
    \textcircled{b} &= \expec \| \hat{\nu} - \expec \hat{\nu}\|^2\\
    &\leq 2 \left[\expec\|\hat{\nu} - \nu\|^2 +  \|\nu - \expec\hat{\nu}\|^2\right] \\
    &\leq 2\left[\expec\|\hat{\nu} - \nu\|^2 + b^2\right] \\
    & = \widetilde{\mathcal{O}}\left(d\left(\frac{1}{s} \gamma_k^{2/k} + \gamma_k^2 \left(\frac{2}{\tau}\right)^{2k-2} \right)\right),
    \end{align*}
by the bias bound~\cref{eq:bias bound}. Therefore, \[
\expec\|N\|^2 =  \widetilde{\mathcal{O}}\left(\frac{\tau^2 d \ln^2(d/\delta)}{s^2 \varepsilon^2} + d\left(\frac{1}{s} \gamma_k^{2/k} + \gamma_k^2 \left(\frac{2}{\tau}\right)^{2k-2} \right)\right).
\]
\end{proof}

\end{document}